\documentclass{amsart}
\usepackage[utf8]{inputenc}
\usepackage[margin=1in]{geometry}
\usepackage{listings}
\usepackage{longtable}
\usepackage[ruled,noline]{algorithm2e}
\usepackage{todonotes}
\usepackage{amsmath}
\usepackage{amsfonts}
\usepackage{comment}
\usepackage[mathlines]{lineno}
\usepackage[sort]{natbib}

\usepackage{bbm}
\usepackage{hyperref}
\hypersetup{
    colorlinks = true,
    allcolors = {blue},
}
\usepackage{enumitem}

\DeclareMathOperator*{\argmax}{arg\,max}

\newcommand{\rr}{\mathbb{R}}
\newcommand{\zz}{\mathbb{Z}}
\newcommand{\nn}{\mathbb{N}}
\newcommand{\Ndat}{T}
\newcommand{\ndat}{D}
\newcommand{\ndatSet}{\mathbb{D}^n}
\newcommand{\mdatSet}{\mathbb{D}^m}
\newcommand{\ndatz}{D^\prime}
\newcommand{\ndatzz}{D^{\prime\prime}}

\newcommand{\med}{\textup{med}}
\newcommand{\nr}{information-set-invariant nonrevealing}
\newcommand{\nrabr}{ISIN}
\newcommand{\diam}{\textup{diam}}
\newcommand{\card}{\textup{\#}}
\newcommand{\rdepth}{\textup{rdepth}}

\newcommand{\hdepth}{\textup{hdepth}}
\newcommand{\convhull}{\textup{conv}}
\newcommand{\hamming}{d_{\mathcal{H}}}
\newcommand{\sign}{\textup{sign}}
\newcommand{\bxi}{\boldsymbol{x_i}}
\newcommand{\bxj}{\boldsymbol{x_j}}
\newcommand{\bu}{\boldsymbol{u}}
\newcommand{\bv}{\boldsymbol{v}}
\newcommand{\bx}{\boldsymbol{x}}
\newcommand{\by}{\boldsymbol{y}}
\newcommand{\bz}{\boldsymbol{z}}
\newcommand{\bzi}{\boldsymbol{z_i}}
\newcommand{\btheta}{\boldsymbol{\theta}}
\newcommand{\hbtheta}{\hat{\boldsymbol{\theta}}}

\newtheorem{theorem}{Theorem}
\newtheorem{lemma}{Lemma}
\newtheorem{corollary}{Corollary}
\newtheorem{definition}{Definition}
\newtheorem{remark}{Remark}
\newtheorem{example}{Example}

\newcommand{\upstairs}[1]{\textsuperscript{#1}}

\newcommand\emails[1]{%
  \begingroup
  \renewcommand\thefootnote{}\footnote{#1}%
  \addtocounter{footnote}{-1}%
  \endgroup
}

\date{July 25, 2022}
\title{Differentially Private Estimation via Statistical Depth}

\usepackage{graphicx}

\hypersetup{colorlinks=true,linkcolor=blue}

\begin{document}

\begin{center}
  \maketitle
  \vspace*{.2in}
  
  \begin{tabular}{cc}
    Ryan Cumings-Menon
  \\[0.25ex]
   {\small U.S. Census Bureau} \\
  \end{tabular}
  
   \emails{\upstairs{*} The views expressed in this paper are those of the authors and not those of the U.S. Census Bureau.}
  \vspace*{0.4in}
\end{center}

\begin{abstract}
    Constructing a differentially private (DP) estimator requires deriving the maximum influence of an observation, which can be difficult in the absence of exogenous bounds on the input data or the estimator, especially in high dimensional settings. This paper shows that standard notions of statistical depth, \textit{i.e.}, halfspace depth and regression depth, are particularly advantageous in this regard, both in the sense that the maximum influence of a single observation is easy to analyze and that this value is typically low. This is used to motivate new approximate DP location and regression estimators using the maximizers of these two notions of statistical depth. A more computationally efficient variant of the approximate DP regression estimator is also provided. Also, to avoid requiring that users specify \textit{a priori} bounds on the estimates and/or the observations, variants of these DP mechanisms are described that satisfy \textit{random differential privacy (RDP)}, which is a relaxation of differential privacy provided by Hall, Wasserman, and Rinaldo (2013). We also provide simulations of the two DP regression methods proposed here. The proposed estimators appear to perform favorably relative to the existing DP regression methods we consider in these simulations when either the sample size is at least 100-200 or the privacy-loss budget is sufficiently high.
\end{abstract}

\section{Introduction} \label{sec:intro}

Since \cite{dwork2006calibrating} first introduced the concept of differential privacy (DP), it has become the gold standard notion of privacy in the statistical disclosure limitation literature, both because of the strong privacy guarantees that DP mechanisms provide and the theoretical properties that make the formulation of many DP mechanisms straightforward, which are described in more detail in the next section. However, DP linear regression methods provide a few interesting theoretical issues related to the lack of exogenous bounds on the regression estimates and/or observations in typical use cases. At the same time, since linear regressions are a ubiquitous tool in applied statistics, there are compelling use cases for linear regressions using sensitive data on respondents in small samples. For example, \cite{chetty2018opportunity} use data from the US Census Bureau and the Internal Revenue Service to estimate socioeconomic mobility of populations within each Census tract based on ordinary least squares regression (OLS) estimates.

This paper explores the use of two notions of statistical depth to formulate approximate DP mechanisms for estimators of location and linear regression coefficients. Specifically, we formulate approximate DP estimators for the Tukey median, which is a multivariate generalization of the median and is defined as the maximizer of halfspace depth, and for the deepest regression, which is defined as the maximizer of regression depth and is an estimator of the median of the dependent variable conditional on the covariates \citep{tukey1975mathematics,rousseeuw1999regression}. One break we make from the norm in the DP literature is that the mechanisms proposed here do not require bounds on the observations in the dataset, but instead require bounds on the space of feasible estimates. In the case of measures of central tendency like the Tukey median, this alternative is a strict relaxation of the requirement of specifying a bounded set containing the observations, since a bounded set containing the observations also contains all reasonable estimators of central tendency, but the converse is not true in the generic case. 

We also provide variants of these methods that satisfy random differential privacy (RDP), which is a relaxation of differential privacy provided by \cite{hall2013random}. This definition of privacy protects against accurate inferences on individual observations of samples that are sufficiently likely to be drawn from the same population distribution, without attempting to limit inferences on the population distribution itself. The RDP variants of these estimators have the advantage of not requiring bounds on the estimates or the observations. Using bounds on the estimates themselves in our proposed DP mechanisms is also required for our derivations of the proposed RDP mechanisms, since the RDP mechanisms simply call our proposed approximate DP methods after using the dataset itself to define the feasible sets by nonparametric confidence regions for the estimators. Also, unlike the motivating example provided by \cite{hall2013random}, for many input datasets, all of the proposed RDP mechanisms provided in this paper are invariant to a privacy attacker's prior information set, in the sense that, for any such information set, these mechanisms do not reveal any respondent's data with certainty for a wide class of input datasets. The final mechanism we introduce is an approximate DP Medsweep mechanism, which is an approximation of the deepest regression provided by \cite{rousseeuw1998computing} that is more computational efficient, particularly in the multivariate setting.

After introducing these estimators, we provide simulations to compare these methods to existing DP regression techniques. One advantageous feature of the DP estimators proposed here is that their statistical performance is typically less dependent on the input bounds for larger sample sizes than existing approaches, which we explore further in these simulations. Specifically, our implementation choices related to bounds on observations and the estimator are intended to err on the side of understating the relative accuracy of the proposed DP regression estimators. The estimators proposed here appear to perform favorably relative to the other DP regression methods considered when either the sample size is larger than approximately 100 and/or $\epsilon$ is sufficiently high. We also use these simulations to estimate the maximum possible rate at which the diameter of the feasible set of estimates can be increased without having any impact on the distribution of the DP estimators. In these simulations, it appears possible to define this feasible set so that its diameter increases at an exponential rate in the sample size. In contrast, the accuracy of many classical DP location and regression estimators depends strongly on the tightness of bounds on the observations. When these bounds are not provided by the use case at hand, one can estimate them using a preliminary DP mechanism; however, one motivation for the RDP estimators proposed here is that this is not always straightforward. For example, \cite{chen2016differentially} propose using a preliminary DP mechanism that outputs a bounding box of the form $[-c,c]^d$ such that the proportion of datapoints in $[-c,c]^d$ is approximately equal to the user choice parameter $\psi\in (0,1),$ but this approach is not ideal when the observations are centered around a point that is far from the origin and/or each dimension of the observations have dissimilar dispersion.\footnote{One method to at least partly ameliorate this issue would be to perform multiple iterations of the approach described by \cite{chen2016differentially}. For example, preliminary bounds could be used within a DP mechanism that outputs the approximate the center of the distribution, and new DP bounds could be computed after using this first estimate to recentering the data. This process could be repeated multiple times, at the cost of requiring additional privacy-loss budget in each iteration. Note that this variant suffers from the same issue as the more basic approach; the accuracy of the final DP estimator becomes worse as the population distribution is shifted away from the origin.} Also, since the datapoints outside of $[-c,c]^d$ are removed from the dataset prior to estimation, setting $\psi$ typically requires balancing a tradeoff between bias and dispersion of the final DP estimator.

In part to bound the scope of the paper, all of the estimators proposed below use perturbation methods that result in DP estimators with log-concave distributions conditional on the data. Many alternatives that do not satisfy this property can be formulated with only minor changes to the proposed approaches, so we will point out some of these possibilities throughout the paper. However, this is also an advantageous property for a DP estimator to satisfy for two reasons. First, this condition ensures the likelihood ratio test for the population mean is monotonic, which we expect will make future work on DP inference of the proposed estimators more straightforward. Second, this condition also ensures the moments of the DP estimators exist conditional on the data. This allows for the DP estimators described here to be used with data from respondents in fairly granular geographic regions and then for summary statistics of these estimates, such as the mean or variance of these estimators across the geographic regions, to have a meaningful interpretation as unbiased estimates of their finite population counterparts. For example, in the use case described by \cite{chetty2018opportunity} and using the proposed estimators, one could estimate a measure of the socioeconomic mobility in the US using a weighted mean of the estimates in each Census tract.

After outlining notation in the next subsection, the remainder of the paper is organized as follows. Section \ref{sec:preliminaries} outlines the required definitions in the DP and statistical depth literature that we will use throughout the paper, and related work on DP linear regressions. Section \ref{sec:betaBounds} provides tighter bounds on a parameter that is used in DP mechanisms that are based on smooth sensitivity \citep{nissim2007smooth} and that use a Laplace noise distribution, which may be of independent interest to the DP community. Afterward, the proposed DP and RDP Tukey median estimators are described in Section \ref{sec:HalfSpaceDepthSens}, and the proposed DP and RDP deepest regression estimators are described in Section \ref{sec:DPRegDepthSens}. Simulations are provided in Section \ref{sec:simulations}, and \ref{sec:discussion} concludes.

\subsection{Notation}

The mechanisms described here take datasets of $n>0$ observations as input; we will denote the set of all such datasets as $\ndatSet.$ For each dataset $\ndat\in\ndatSet,$ we will assume throughout that each respondent contributes to at most one observation in $\ndat,$ which is assumed to be an element of $\rr^d.$\footnote{We also assume that it makes sense to talk about the data of one respondent in isolation, which, for example, fails to hold for data on the adjacency relationships in social networks. \cite{kifer2014pufferfish} provides more information on how interrelated observations impacts formal privacy guarantees.} Unless noted otherwise, we will not assume the observations in the input dataset are drawn from a population distribution. When introducing each of the formally private estimators proposed below, we will denote a candidate set, or feasible set, of estimators as $\Theta,$ and the non-private estimate as $\hbtheta.$ 

We will use bold symbols to denote vectors, and also use $\log(x),$ where $x\in\rr_{++},$ to denote the natural logarithm of $x.$ In addition, let $\mathbf{1}_A:\rr^d\rightarrow \{0,1\}$ denote the indicator function of the set $A\subset\rr^d,$ $\lfloor x \rfloor$ denote the largest integer less than or equal to $x\in\rr,$ $\lceil x \rceil$ denote the smallest integer greater than or equal to $x\in\rr,$ $x_{(k)}$ denote the $k^{\textrm{th}}$ order statistic of $\{x_i\}_{i=1}^n,$ $P_X(\cdot)$ denote the probability with respect to the random variable $X,$ $\convhull(A)$ denote the convex hull of the set $A,$ and let $\sign:\rr\rightarrow\{-1,0,1\}$ denote the sign function. Also, let $\hamming: \ndatSet \times \ndatSet \rightarrow\zz$ be defined so that $\hamming(\ndat,\ndatz)$ is equal to the number of records that must be added and/or removed from $\ndat$ to derive $\ndatz.$ Note that, since substituting one record for another using only these two operations requires adding one record and then removing one record, for any $\ndat,\ndatz\in\ndatSet$ that differ in $k$ records, we have $\hamming(\ndat,\ndatz)=2k.$ We will also use $\lVert \bx \rVert_p$ to denote the $L^p$ norm of $\bx\in\rr^d,$ and $\lVert \bx \rVert$ to denote the Euclidean norm. 

\section{Preliminaries} \label{sec:preliminaries}

\subsection{Differential Privacy}

The definition of an $(\epsilon,\delta)-$DP mechanism was first provided by \cite{dwork2006calibrating,dwork2006our}, as described below. where all randomness is due to the mechanism

\begin{definition}
    \citep{dwork2006calibrating, dwork2006our} Let the \textit{neighbors} of the dataset $\ndat \in\ndatSet$ be defined as $N(\ndat)=\{\ndatz\in\ndatSet : \hamming(\ndat, \ndatz) = 2 \}.$ A randomized algorithm $M:\ndatSet \rightarrow \Theta$ satisfies $(\epsilon, \delta)$-\textit{differential privacy} (DP) if and only if, for all neighboring datasets $\ndat,\ndatz\in\ndatSet$ and any measurable set $B\subset \Theta,$ we have $P_M(M(\ndat)\in B) \leq  \exp(\epsilon) P_M(M(\ndatz)\in B) + \delta.$
\end{definition}

\noindent
We also follow the norm in the literature and refer to $(\epsilon, 0)-$DP mechanisms as \textit{pure $\epsilon-$DP mechanisms}, and we say that a given $(\epsilon, \delta)-$DP mechanism is an \textit{approximate DP} mechanism when $\delta > 0.$ We will also occasionally refer to $\epsilon$ as the privacy-loss budget. Note that, since we define neighboring databases as $N(\ndat)=\{\ndatz\in\ndatSet : \hamming(\ndat, \ndatz) = 2 \},$ the definition above corresponds to \textit{bounded} DP because the sample size of all neighbors of $\ndat\in\ndatSet$ is fixed at $n$ and thus is bounded. The definition of \textit{unbounded} DP follows from instead defining the neighboring datasets of $\ndat\in\ndatSet$ as $N_{u}(\ndat) = \{\ndatz\in \mathbb{D}^{n-1} \cup \mathbb{D}^{n+1} : \hamming(\ndat, \ndatz) = 1\}.$ We use the bounded DP definition here because some of the mechanisms proposed below do not attempt to protect inferences on the sample size of the input dataset, and the norm in the DP literature is to use the bounded DP definition in these cases.

\cite{wasserman2010statistical} provide an intuitive interpretation of a DP guarantee; the Neyman-Pearson lemma implies that the requirement of DP is equivalent to bounding the power of any possible hypothesis test for the null hypothesis that a given respondent's attributes are equal to a given value. Some of the many advantages of this privacy definition include invariance to post-processing, \textit{i.e.}, if $M:\ndatSet \rightarrow \Theta$ satisfies $(\epsilon,\delta)-$DP then so does $\ndat \mapsto f(M(\ndat))$ for any $f:\Theta \rightarrow \Lambda,$ and sequential composition, \textit{i.e.}, if $M_1:\ndatSet \rightarrow \Theta$ and $M_2:\ndatSet \rightarrow \Theta$ both satisfy $(\epsilon,\delta)-$DP  then $ M(\ndat) =  (M_1(\ndat), M_2(\ndat))$ satisfies $(2 \epsilon, 2 \delta)-$DP. For other properties, including other forms of composition, see \citep{dwork2006calibrating, dwork2014algorithmic}.

Designing an $(\epsilon,\delta)-$DP mechanism for a given function often requires a bound on its global sensitivity, which is described in the following definition, along with a related definition, which will also be used in the next subsection.

\begin{definition}
    \citep{dwork2006calibrating, nissim2007smooth} The \textit{local sensitivity} of the function $f:\ndatSet\rightarrow\Theta$ is

    \begin{align*}
        LS_f(\ndat) = \max_{\ndatz \in\ndatSet: \hamming(\ndat, \ndatz) = 2} \lVert f(\ndat) - f(\ndatz) \rVert_1.
    \end{align*}

    The \textit{global sensitivity} of the function $f:\ndat\rightarrow\Theta$ is

    \begin{align*}
        \Delta_f = \max_{\ndat\in\ndatSet} LS_f(\ndat).
    \end{align*}
\end{definition}

The following lemma provides a particularly simple $\epsilon-$DP mechanism known as the Laplace mechanism. Like all methods defined using global sensitivity, this mechanism requires that the global sensitivity of the function is bounded. For example, this is the case for counting queries (\textit{e.g.:} a function that releases the total population within a given geographic region). However, this is not typically the case for functions that have an unbounded domain; for example, without \textit{a priori} bounds on possible incomes, the following mechanism could not be used to release the average income of the residents in a geographic region. One possible way around this issue is to use a DP mechanism to approximate bounds on the observations in the sample, and then use a Laplace mechanism to release the average income of respondents with incomes that are within these bounds \citep{chen2016differentially}. This approach is discussed in more detail in Section \ref{sec:simulations}.

\begin{lemma}
    \citep{dwork2006calibrating} Given the function $f:\ndatSet\rightarrow \Theta\subset\rr^d,$ the \textit{Laplace mechanism}, $M(\ndat) = f(\ndat) + \boldsymbol{Z},$ where  $\boldsymbol{Z}[i] \sim  \textup{Laplace}(0, \Delta_f / \epsilon)$ for each $i\in\{1,\dots, d\},$ is $\epsilon-$DP.
\end{lemma}

One relaxation of $(\epsilon,\delta)-$DP is $(\epsilon, \delta,\gamma)-$random differential privacy, which is defined below. This privacy definition limits inferences between neighboring datasets that are sufficiently likely to consist of draws from the same population distribution, rather than attempting to limit inferences between all pairs of neighboring datasets. Note that this definition does not assume that the data curator knows the population distribution is in a particular class of distributions.

\begin{definition}
    \citep{hall2013random} A randomized algorithm $M:\ndatSet \rightarrow \Theta$ satisfies $(\epsilon, \delta, \gamma)$-\textit{random differential privacy} (RDP) if, for all neighboring datasets $\ndat,\ndatz\in\ndatSet$ composed of observations drawn from the same population distribution $\Ndat,$ and any measurable set $S\subset \Theta,$ we have $P_{\Ndat}(P_M(M(\ndat)\in S) \leq  \exp(\epsilon) P_M(M(\ndatz)\in S) + \delta)\geq 1-\gamma.$
\end{definition}

\subsection{Smooth Sensitivity}

\cite{nissim2007smooth} provide a method to decrease the sensitivity value used within a DP mechanism such as the Laplace mechanism for functions that have high global sensitivity but most often have low local sensitivity. A common example of one such function is the univariate median of values in a sample within a bounded interval, \textit{i.e.}, while it is possible to have a local sensitivity equal to the length of this interval, which is thus also equal to the global sensitivity, changing one record typically changes the median by a much smaller amount.

\begin{definition} \label{neighbor_sens}
    \citep{nissim2007smooth} Let the \textit{local sensitivity at distance $k$} of $f:\ndatSet\rightarrow \Theta$ be defined as the max local sensitivity of $f(\cdot)$ over all datasets in $\{\ndatz \in\ndatSet: \hamming(\ndat, \ndatz) = 2 k\};$ in other words, the \textit{local sensitivity at distance $k$} is defined as

    \begin{align*}
        A_f^{(k)}(\ndat) = \max_{\ndatz \in\ndatSet: \hamming(\ndat, \ndatz) \leq 2 k} LS_f(\ndatz).
    \end{align*}

    The \textit{$\beta-$smooth sensitivity} of $f:\ndatSet \rightarrow \Theta$ is defined as

    \begin{align} \label{def:smoothSense}
        S_f^\star(\ndat) = \max_{k\in \{0,1,\dots\}} \exp(-k \beta) A_f^{(k)}(\ndat),
    \end{align}

    \noindent
    and we will refer to $S_f(\ndat) = \max_{k\in \{0,1,\dots\}} \exp(-k \beta) \widetilde{A}_f^{(k)}(\ndat),$ where $\widetilde{A}_f^{(k)}(\ndat)$ is an upper bound on $A_f^{(k)}(\ndat),$ as a \textit{$\beta-$smooth upper bound on the local sensitivity}.
\end{definition}

The following Lemma provides one way of deriving a $\beta-$smooth upper bound on the local sensitivity of a function, which we make use of below.

\begin{lemma} \label{lem:nissim_claim_3p2}
    \citep{nissim2007smooth} Let $\widetilde{A}_f^{(k)}:\ndat\rightarrow\rr_+$ be defined so that (1) for all $\ndat\in\ndatSet,$ $LS_f(\ndat) \leq \widetilde{A}_f^{(0)}(\ndat),$ and (2) for all neighbors $\ndat,\ndatz\in\ndatSet$ and all $k\in\zz,$ $\widetilde{A}_f^{(k)}(\ndat) \leq \widetilde{A}_f^{(k+1)}(\ndatz).$ Then $S_f(\ndat) = \max_{k\in \{0,1,\dots\}} \exp(-k \beta) \widetilde{A}_f^{(k)}(\ndat)$ is a $\beta-$smooth upper bound on the local sensitivity.
\end{lemma}

As an example, the lemma above implies that one $\beta-$smooth upper bound on the local sensitivity of the function $f(\cdot)$ can be derived by first defining the upper bound on the local sensitivity at distance $k$ as

\begin{align*}
    \widetilde{A}_f^{(k)}(\ndat)= \max_{\btheta_1,\btheta_2\in \Omega(k+1,\ndat)}\lVert \btheta_1 - \btheta_2 \rVert_1,
\end{align*}

\noindent
where $\Omega(k,\ndat)=\{f(\ndatz) : \ndatz\in\ndatSet, \; \hamming(\ndat, \ndatz) \leq 2 k\}$ and then defining the final $\beta-$smooth upper bound as $S_f(\ndat) = \max_{k\in \{0,1,\dots\}} \exp(-k \beta) \widetilde{A}_f^{(k)}(\ndat).$

The following lemma provides an example of a mechanism that uses a $\beta-$smooth upper bound on local sensitivity. 

\begin{lemma} \label{lem:example_ss_laplace_mech}
    \citep{nissim2007smooth} Suppose $\Theta\subset\rr^d$ and $\alpha=\epsilon/2.$ Also, if $d\geq 2,$ let $\beta=\epsilon/(4(d + \log(1/\delta)),$ and, if $d=1,$ let $\beta=\epsilon/(2\log(2/\delta)).$ 

    Then, for $f:\ndatSet\rightarrow \Theta,$ the mechanism that outputs $f(\ndat) + \boldsymbol{Z},$ where $\boldsymbol{Z}[i] \sim \textup{Laplace}(0, S_f(\ndat) / \alpha)$ for each $i\in\{1,\dots, d\},$ and $S_f(\ndat)$ is the $\beta-$smooth sensitivity of $f(\cdot)$ at $\ndat,$ satisfies $(\epsilon,\delta)-$DP.
\end{lemma}

\subsection{Halfspace Depth} \label{sec:HalfSpaceDepthIntro}

\cite{tukey1975mathematics} introduced the concept of halfspace depth, which can be viewed as a multivariate generalization of rank. The halfspace depth of $\btheta\in\Theta$ is the minimum number of datapoints in a halfspace with a boundary passing through $\btheta,$ which is also defined below. 

\begin{definition} \label{ldepthDef}
    \citep{tukey1975mathematics,small1990survey} Given a dataset $\ndat\in\ndatSet$ consisting of observations $\bxi\in \Theta \subset \rr^d,$ the halfspace depth of $\btheta \in \Theta \subset \rr^d$ is defined as

    \begin{align*}
            \hdepth(\btheta, \ndat) = \min_{\bu\neq \boldsymbol{0}} \card \{ \bxi \in \ndat : \bu^\top (\bxi - \btheta) \geq 0 \}.
    \end{align*}
\end{definition}

This measure of depth has many advantageous properties, including that it is affine equivariant, which means that, for any  $A\in\rr^{d\times d}$ with full rank and $\boldsymbol{b}\in\rr^d,$ we have $\hdepth(\btheta, \ndat) = \hdepth(A \btheta + b, \{A \xi + b : \xi \in \ndat\})$ \citep{donoho1992breakdown}. Also, we will make use of the following property on each halfspace-depth contour of $\ndat,$ or the subset of $\Theta$ with half-space depth at least a given value.

\begin{lemma} \label{lem:diamHalfspace}
    \citep{donoho1992breakdown} For any $k\in\zz_{++},$ the halfspace depth contour $\Omega(k, \ndat) = \{  \btheta \in\Theta: \hdepth(\btheta, \ndat)  \geq k \}$ satisfies

    \begin{align*}
        & \max_{ \bz_1, \bz_2 \in \Omega(k, \ndat)} \lVert \bz_1 - \bz_2 \rVert_1 = \max_{ \bxi, \bxj \in A} \lVert \bxi - \bxj \rVert_1,
    \end{align*}

    \noindent
    where $A = \{\bxi \in \ndat : \hdepth(\bxi, \ndat)  \geq k \}.$
\end{lemma}
\begin{proof}
    This follows from the fact that $\Omega(k, \ndat) = \convhull(A);$ see for example, \citep{donoho1992breakdown}.
\end{proof}

Maximizing $\hdepth(\btheta, \ndat) $ provides a measure of central tendency of a dataset that is known as the Tukey median, which is a multivariate generalization of the median. As described in the definition below, the maximizer of the halfspace depth is not generally unique, so, in cases in which this maximizer is not uniquely defined, we will define the Tukey median using a choice rule that outputs a unique element of $\Theta.$ The results in this paper only require that this choice rule outputs a point in the convex hull of data points with maximum halfspace depth, which is satisfied by any reasonable rule; one common approach in practice is to simply define the Tukey median as the arithmetic mean of the data points with maximum halfspace depth.

\begin{definition} \label{TukeyMedianDef}
    \citep{tukey1975mathematics} The \textit{Tukey median} is defined by
    \begin{align*}
        C(\{\bxi \in\ndat : \hdepth(\bxi, \ndat) = \max_{\btheta\in\Theta} \hdepth(\btheta, \ndat) \})
    \end{align*}
    where $C(\cdot)$ is any choice rule that selects a point from the convex hull of its input points.
\end{definition}

To describe a few advantageous properties of the Tukey median, some additional notation will be helpful. First, we will say that a dataset consists of observations in general position if there does not exist a $d-1$ dimensional hyperplane that contains more than $d$ observations. Note that, when this condition does not hold, it can be ensured with probability one by dithering the dataset, \textit{i.e.}, adding continuously distributed mean zero noise with a low scale to each observation; this will be discussed in more detail in Section \ref{sec:HalfSpaceDepthSens}. Second, the breakdown value of an estimator is defined as the smallest possible proportion of points that must be contaminated in order to make the estimator equal to an arbitrary value, as defined below more formally.

\begin{definition} 
    \citep{donoho1992breakdown} The \textit{breakdown value} of the estimator $f:\ndatSet\rightarrow \rr^d$ is defined as
    \begin{align*}
        \zeta^\star = \min \left\{\frac{m}{n + m} : \sup_{\ndatz\in \mdatSet} \lVert f(\ndat \cup \ndatz) - f(\ndat)  \rVert = \infty\right\}.
    \end{align*}
\end{definition}

One advantageous property of the Tukey median is that it is a robust estimator. For example, for datasets with observations in general position, its breakdown value is at least $1/(1+d),$ and one related property that we will use below is that, for any such dataset, the maximum halfspace depth is at least $\lceil n/(1+d) \rceil$ \citep{masse2002asymptotics, donoho1992breakdown}. The Tukey median is also consistent at the usual parametric rate of $O(1/\sqrt{n})$ when $\ndat$ is composed of independent observations from a population distribution. Detail on computational considerations will also be provided in the Section \ref{sec:ComputingDepth}, after outlining regression depth in the next section.

\subsection{Regression Depth}

\begin{figure}[ht]
    \centering
    \includegraphics[scale=.58]{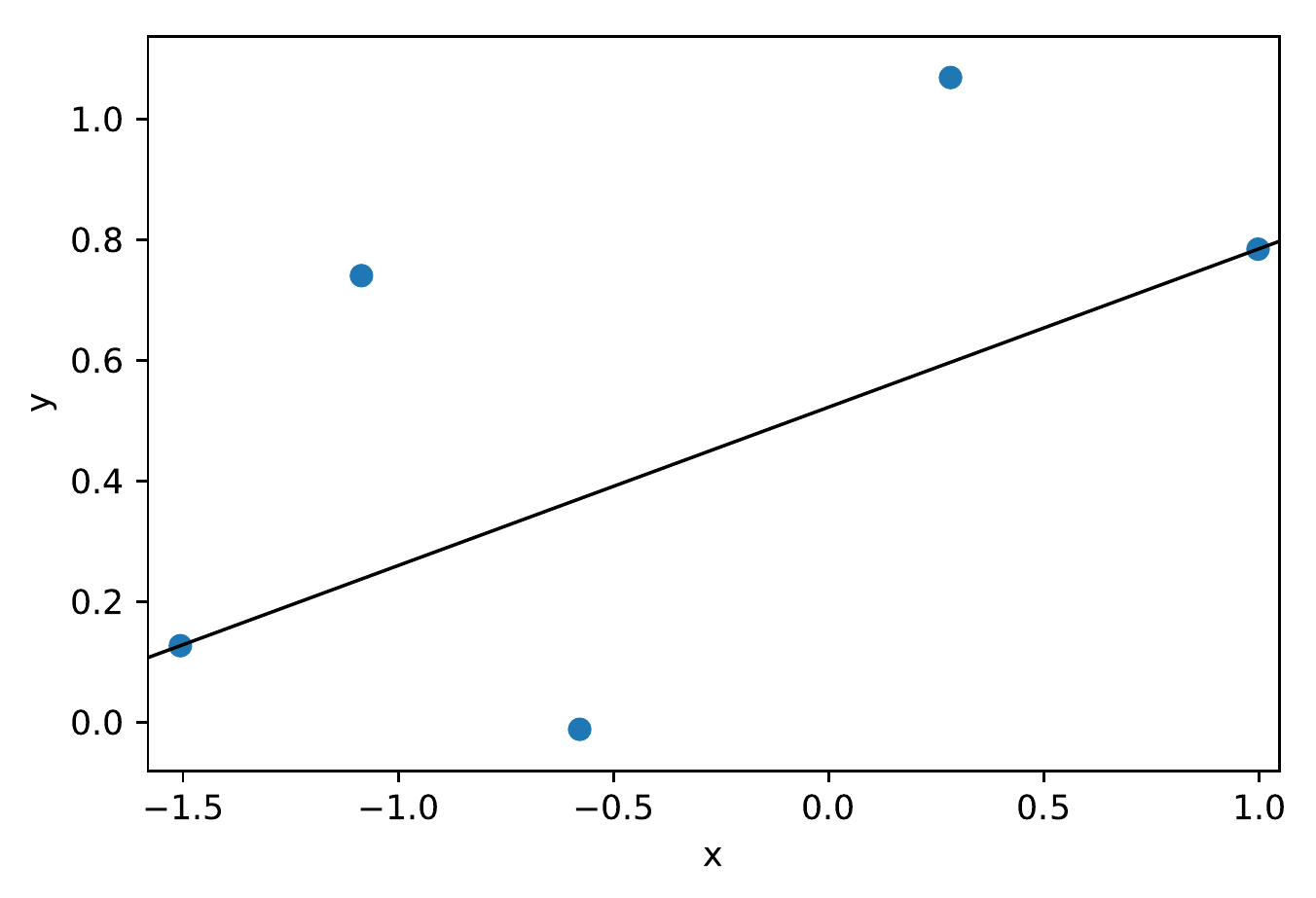}  \includegraphics[scale=.58]{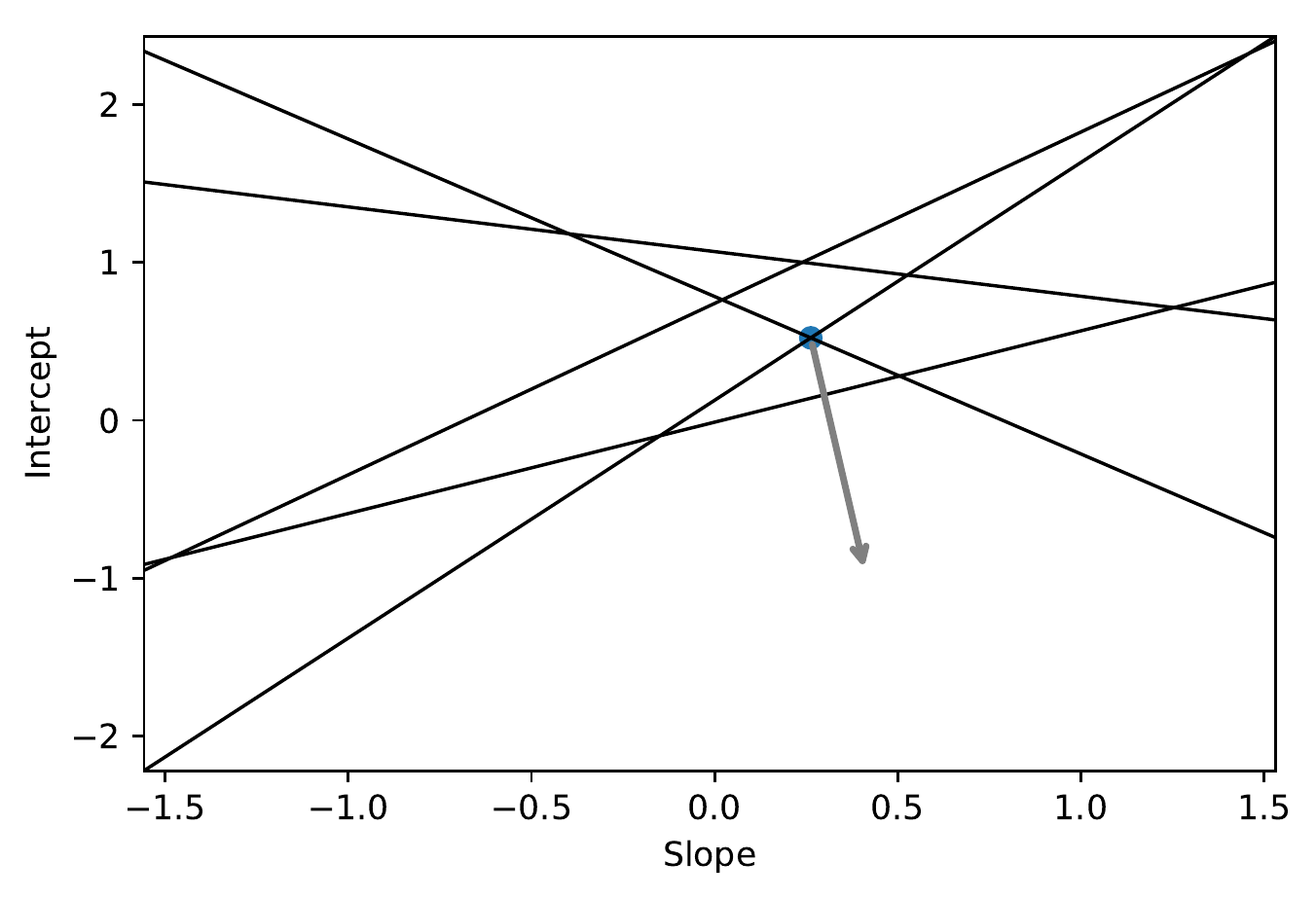}
    \caption{The left plot provides a scatter plot consisting of five datapoints in primal space, as well as a candidate regression line. These five datapoints correspond to the five lines in dual space on the right plot. The candidate regression line in the left plot corresponds to the blue point in the right dual space plot.}
    \label{fig:dualPlot}
\end{figure}

Some notation will be helpful before describing regression depth and its maximizer, the deepest regression hyperplane, which were both first proposed by \cite{rousseeuw1999regression}. In this section, and all sections related to the DP deepest regression estimator below, we will suppose the dataset $\ndat\in \ndatSet$ is composed of observations $(\bxi, y_i)\in \rr^d.$ We also only consider regressions with an intercept throughout the paper; in other words, we suppose that each vector $\btheta\in\Theta\subset\rr^d$ corresponds to the candidate fit $y = (1, \bx^\top)\btheta.$

Before providing the general definition of regression depth, we will start by considering the case of a simple linear regression, \textit{i.e.}, $\Theta \subset \rr^2.$ In this case, for any fixed value $u\notin \{x_i\}_i^n,$ the regression depth can be defined using the integers

\begin{align*}
    L^+_u(\btheta, \ndat) = \card \{(x_i,y_i)\in\ndat: (1, x_i)  \btheta \geq y_i , \; x_i < u\} \\
    R^+_u(\btheta, \ndat) = \card \{(x_i,y_i)\in\ndat: (1, x_i)  \btheta \geq y_i, \; x_i > u\} \\
    L^-_u(\btheta, \ndat) = \card \{(x_i,y_i)\in\ndat: (1, x_i)  \btheta \leq y_i, \; x_i < u\} \\
    R^-_u(\btheta, \ndat) = \card \{(x_i,y_i)\in\ndat: (1, x_i)  \btheta \leq y_i, \; x_i > u\},
\end{align*}

\noindent
which correspond to the number of data points either above or below the graph of the fit and either to the left or the right of $u\in \rr^1.$ Note that rotating the graph of the fit clockwise around the point $(u, (1, u)  \btheta)$ to a vertical line would intersect $L^-_u(\btheta, \ndat) + R^+_u(\btheta, \ndat)$ observations, whereas rotating this graph counterclockwise intersects $L^+_u(\btheta, \ndat) + R^-_u(\btheta, \ndat)$ observations. Thus, the regression depth in this two dimensional case is given by

\begin{align*}
    \rdepth(\btheta, \ndat) = \min_{u\in \rr^1/\{x_i\}_i^n} \min \left( L^+_u(\btheta, \ndat) + R^-_u(\btheta, \ndat), L^-_u(\btheta, \ndat) + R^+_u(\btheta, \ndat) \right).
\end{align*}

Regression depth can be defined in an analogous way in higher dimensional cases by replacing $u$ with a hyperplane in $\bx-$space \citep{rousseeuw1999regression}. We will use the definition below in this paper instead, which appears to have been first described by \citep{mizera2002depth}.

\begin{definition} \label{rdepthDef}
    \citep{rousseeuw1999regression, mizera2002depth} Given a dataset consisting of observations $ (\bxi, y_i)\in\ndat$ with $(\bxi, y_i) \in \rr^d,$ the regression depth of $\btheta \in \Theta \subset \rr^d$ is defined as

    \begin{align*}
        \rdepth(\btheta, \ndat) = \min_{\bu\neq \boldsymbol{0}} \card \{ (\bxi, y_i) \in \ndat : -\bu^\top \bxi \; \sign(y_i - (1, \bxi^\top) \btheta) \geq 0\}.
    \end{align*}
\end{definition}

This definition can be derived by considering the regression depth in the dual space. We will briefly outline this and a few related concepts, in part because the local sensitivity of the deepest regression has a straightforward geometric interpretation in the dual space. More detail on the dual space in the context of regression depth can be found in \citep{van2008efficient, rousseeuw1999regression}; for interesting historical connections to related regressions, see \citep{koenker2000galton}. We will use the example in Figure \ref{fig:dualPlot} to introduce these concepts. The left plot in Figure \ref{fig:dualPlot} depicts a set of five data points $(x_i,y_i)\in\ndat$ in $(x,y)-$space, or \textit{the primal space}. Each of these primal points can alternatively be encoded as lines in dual space using the mapping $(x_i, y_i) \mapsto \{\btheta \in \Theta : (1, x_i) \;\btheta = y_i \},$ or the set of all fits that pass through the data point in the primal space. The right plot in Figure \ref{fig:dualPlot} provides the dual lines corresponding to each of the primal points. Likewise, each point $\btheta \in \Theta$ in the dual space corresponds to a single regression line in primal space using the mapping $\btheta \mapsto \{ (x, y) \in \rr^2 : (1, x)\; \btheta = y \}.$ 

The act of rotating the graph of a fit to a vertical hyperplane in primal space corresponds to simply moving the point $\btheta \in \Theta$ in dual space corresponding to this fit along a ray. For example, in Figure \ref{fig:dualPlot}, we can see that the regression line on the left plot has a regression depth of three in two ways. First, we can count the minimum number of data points this line must intersect when rotated to a vertical line. In this case rotating this regression line counterclockwise to a vertical line around the intersection point of this line with the right hand side of the plot results in the line crossing three points in total. Second, we can alternatively count the minimum number of observation lines that a ray in dual space, originating from this regression coefficient's location, \textit{i.e.}, the blue point in the plot on the right side Figure \ref{fig:dualPlot}, must cross. The gray ray in Figure \ref{fig:dualPlot} is one such example; since this ray crosses three observation lines, including the two lines at the origin of the ray, the regression depth of this fit is three.

To derive definition \ref{rdepthDef}, note that the observation line $\{ \btheta \in \Theta :  y_i = (1, \bxi^\top) \btheta \}$ intersects the ray that originates from $\btheta$ and points in the direction $\bu \neq \boldsymbol{0}$ if and only if there exists $r\in [0, t]$ such that

\begin{align*}
    y_i &= (1, \bxi^\top) (\tilde{\btheta} + r \bu) \iff  \frac{y_i - (1, \bxi^\top) \tilde{\btheta}}{(1, \bxi^\top) \bu} \in [0,t].
\end{align*}

\noindent
After taking the limit as $t\rightarrow \infty,$ we have

\begin{align*}
    \frac{y_i - (1, \bxi^\top) \tilde{\btheta}}{(1, \bxi^\top) \bu} \geq 0 \iff (1, \bxi^\top) \bu \; \sign(y_i - (1, \bxi^\top) \tilde{\btheta}) \geq 0,
\end{align*}

\noindent
which implies Definition \ref{rdepthDef}.

Like halfspace depth, one advantageous property of regression depth is that it is affine equivariant. Also, below we will make use of the following property of regression depth contours, which are sets in dual space given by $\Omega(k, \ndat) = \{  \btheta \in\Theta : \rdepth(\btheta, \ndat)  \geq k \}.$ 

\begin{lemma} \label{lem:diamReg}
    \citep{rousseeuw1999regression} For any $k\in\zz_{++},$ the regression depth contour $\Omega(k, \ndat) = \{  \btheta \in\Theta: \rdepth(\btheta, \ndat)  \geq k \}$ satisfies

    \begin{align*}
        & \max_{ \btheta_1, \btheta_2 \in \Omega(k, \ndat)} \lVert \bz_1 - \bz_2 \rVert_1 = \max_{ \btheta_1, \btheta_2 \in A} \lVert \bxi - \bxj \rVert_1,
    \end{align*}

    \noindent
    where $A = \{\btheta^{(ij)} \in \Theta : \rdepth(\btheta^{(ij)}, \ndat)  \geq k, \; (1, \bxi^\top)\; \btheta^{(ij)} = y_i, \; (1, \bxj^\top)\; \btheta^{(ij)} = y_j, \textup{ and } (\bxi, y_i), (\bxj, y_j) \in \ndat \}.$
\end{lemma}
\begin{proof}
This follows from the fact that $ \convhull(\Omega(k, \ndat)) = \convhull(A)$ \citep{rousseeuw1999regression}.
\end{proof}

One can also consider all other candidate fits to see that the fit in Figure \ref{fig:dualPlot} is the fit with the highest regression depth, or the \textit{deepest regression}, which is defined below. Like the Tukey median, the maximizer of regression depth need not be unique, so the definition below also uses a choice rule to ensure the deepest regression estimator is defined uniquely. Also like the Tukey median, our results only require that this choice rule outputs a final estimate in the convex hull of $\btheta\in\Theta$ with maximum regression depth. One common approach, and the approach we use in the simulations below, is to define this choice function so that it outputs the arithmetic mean of observation line intersections in dual space that have maximum regression depth.

\begin{definition} \label{DeepestRegDef}
    \citep{rousseeuw1999regression} The \textit{deepest regression} is defined as

    \begin{align*}
        C(\{\hbtheta\in\Theta : \rdepth(\hbtheta, \ndat) = \max_{\btheta\in\Theta} \rdepth(\btheta, \ndat) \}),
    \end{align*}

    \noindent
    where $C(\cdot)$ is any choice rule that outputs a point in the convex hull of its input.
\end{definition}

One advantageous property of this estimator is that it is robust, with a breakdown value of at least $1/(1+d)$ \citep{mizera2002depth}. In addition, the deepest regression obtains this high breakdown value without sacrificing the usual parametric rate of consistency of $O(1/\sqrt{n})$ to the conditional median of $y$ given $\bx;$ for more detail, see \citep{bai2008asymptotic,he1998asymptotics}. Note that a similar consistency result also holds for the least absolute deviation (LAD) regression \citep{koenker1978regression}, but the deepest regression has the advantage over the LAD regression of being robust to outliers in the independent variables, $\{\bxi\}_i.$ Also, one property that we will use below is that the regression depth of the deepest regression is at least $\lceil n/(1+d)\rceil,$ which, unlike the similar result described above for the halfspace depth of the Tukey median, has been shown to hold for datasets that are not in general position \citep{mizera2002depth,amenta2000regression}. 

\subsection{Statistical Depth: Computational Considerations} \label{sec:ComputingDepth}

In this subsection we will describe algorithms for computing halfspace and regression depth and their maximizers when $\Theta \subset \rr^d,$ where $d\geq 2.$ Both of the notions of statistical depth can be computed at a candidate estimate $\btheta\in\Theta$ in $O(n^{d-1} \log(n))$ time, and faster approximations for these methods are also available \citep{langerman2000optimal,rousseeuw1998computing}. 

In the case of the Tukey median, simply computing the halfspace depth at each datapoint provides an algorithm for computing the Tukey median in $O(n^d\log(n))$ time. Faster approaches are also available when $d=2$ \citep{miller2001fast,langerman2003optimization}. \cite{struyf2000high} also provide an approximation method for the Tukey median when the sample size and/or the dimension is prohibitively large.

In the case of regression depth, there are $O(n^d)$ candidate fits that pass through $d$ data points to consider, so the na\"ive approach of computing the regression depth for each of these fits has a time complexity of $O(n^{2 d - 1} \log(n)),$ and is most often prohibitively computationally expensive when $d>2.$ \cite{langerman2003complexity} provide a method for approximating the deepest regression that runs in $O(n\log(n))$ time when $d=2.$ \cite{van2002deepest} provide a method for approximating to the deepest regression in use cases in which the sample size and/or the dimension is higher.

We emphasize that the first four differentially private mechanisms provided below actually require more than the maximizer of a notion of statistical depth; in order to also compute a $\beta-$smooth upper bound on the local sensitivity of the estimator, these mechanisms also require upper bounds on diameters of the depth contours defined in Lemma \ref{lem:diamHalfspace} and Lemma \ref{lem:diamReg}.

Computationally efficient methods to compute upper bounds on these values are left as a topic for future research. For the purposes of these first four differentially private methods in this paper we simply use the two na\"ive approaches in order to explicitly find these diameters. Since this implies that the deepest regression mechanisms provided below are primarily applicable when $d=2,$ Section \ref{sec:approxDeepestReg} provides a mechanism for the more computationally efficient approximation of the deepest regression provided by \cite{van2002deepest}.

\subsection{Related Literature} \label{sec:relatedLit}

This section mainly focuses on providing some pointers into the DP literature for linear regression and Tukey median estimators. The only paper that we are of that provides a DP Tukey median estimator is \cite{ramsay2021differentially}, which uses a formulation based on an exponential mechanism rather than on the $\beta-$smooth sensitivity of the Tukey median, as is done below. One such class of DP estimators are those that pass sufficient statistics of a linear regression through a DP primitive mechanism; see for example, \citep{foulds2016theory,mcsherry2009differentially,vu2009differential,wang2018revisiting,dwork2014analyze}. A related class of DP estimators include the methods proposed by \cite{dwork2009differential,alabi2020differentially}, which are based on the Theil-Sen simple linear regression estimator \citep{theil1950rank,sen1968estimates}. \cite{bassily2014private} also propose a DP OLS method that estimates the regression parameters using a DP gradient descent method. One line of work that can be used to formulate a particularly broad class of DP estimators is based on perturbing the objective function of a convex optimization problem \citep{kifer2012private,awan2021structure}. There has also been work on DP inference; see for example, \citep{barrientos2019differentially,sheffet2017differentially,awan2021structure,karwa2018finite}. 

Since some of our results are related to the input requirement for most DP estimators of \textit{a priori} bounds on the observations, it is worth highlighting a few papers related to this issue. First, as described above, \cite{chen2016differentially} describe a DP mechanism that can be used to estimate a bounding box of the form $[-c,c]^d \subset \rr^d$ that contains a fixed proportion of the datapoints. These bounds can be used as input to subsequent DP mechanisms after removing all observations that are outside of this bounding box from the dataset. More detail on this approach is also given in the next subsection and Section \ref{sec:simulations}. Second, \cite{karwa2018finite} provide DP estimators for the univariate mean of the data without the requirement of input bounds on the observations. While the $(\epsilon,\delta,\gamma)-$RDP estimators proposed below also do not require input bounds on the data or the estimator, the $(\epsilon,\delta,\gamma)-$RDP estimators proposed below are different from the (pure and approximate) DP estimators proposed by \cite{karwa2018finite} for a few reasons. First, the accuracy of the approach described by \cite{karwa2018finite} is dependent on whether or not the input dataset satisfies their assumption that it is composed of independent draws from a Gaussian distribution; in contrast, the $(\epsilon,\delta,\gamma)-$RDP estimators proposed below do not require distributional assumptions on the population distribution. Second, since RDP is a relaxation of the privacy guarantee provided by DP, the privacy guarantee provided by the mechanisms proposed by \cite{karwa2018finite} is stronger than the corresponding guarantee provided by the estimators proposed below. Third, in the case of the RDP estimators proposed below, our avoidance of distributional assumptions requires an assumption on the sample size being above a given threshold, which is not required by the approach described by \cite{karwa2018finite}. Later work also generalized this approach to the multidimensional case, while maintaining the normality assumption on the population distribution; see for example, \cite{kamath2022private}.

\section{Improved Bounds for \texorpdfstring{$\beta-$}{beta-}Smooth Sensitivity} \label{sec:betaBounds}

This section provides improved bounds on the maximum value of $\beta$ that can be used in a smooth-sensitivity-based mechanism with noise drawn from a Laplace distribution while still satisfying $(\epsilon, \delta)-$DP. The corresponding bounds on $\beta$ that are provided by \cite{nissim2007smooth} in this case, which are also provided in Lemma \ref{lem:example_ss_laplace_mech}, prioritize tractability over tightness. For example, when $d=2$ and $\delta \leq 1/10^6,$ the bound proposed by \cite{nissim2007smooth} for this case is slightly less than half of the value provided here. The result that provides the improved lower bound on $\beta,$ \textit{i.e.}, Theorem \ref{laplace_smooth_sense}, follows from the next theorem, which provides three upper bounds on a certain quantile of the sum of exponential random variables, given in (\ref{bound1})-(\ref{bound3}). In the simulations provided in Section \ref{sec:simulations}, bound (\ref{bound1}) is used to set $\beta$ to the tightest bound we derive here, as described in Theorem \ref{laplace_smooth_sense}. 

\begin{theorem} \label{thm:prob_bound}
    Suppose $Y = \sum_{i}^d \boldsymbol{X}[i],$ where $ \boldsymbol{X}[i]\sim \textup{Exponential}(1)$ for each $i\in\{1,\dots,d\},$ and let $\rho(\delta, d)$ be defined as the $1-\delta$ quantile of $Y.$ If $d>1,$

    \begin{align} 
        \rho(\delta, d) & = q_{d,1}(1-\delta) \label{bound1}\\
        & \leq - d W_{-1}(-\delta^{1/d}/e) \label{bound2} \\
        & \leq \frac{\sqrt{2 \log (\delta) (2 \log (\delta)-9 d)}+3 d-2 \log (\delta)}{3} \label{bound3} 
    \end{align}

    \noindent
    where $q_{d,1} : [0,1]\rightarrow \rr_{+}$ is the quantile function of the gamma distribution with shape and rate parameters given by $d$ and $1$ respectively, and $W_{-1}:[-1/e,0)\rightarrow \rr_{-}$ is the secondary branch of the Lambert W function. 
\end{theorem}
\begin{proof}
    The equality given in (\ref{bound1}) simply follows from the fact that $Y\sim\textup{Gamma}(d, 1).$ The starting point for both of the two remaining bounds is the following Chernoff bound for $P(Y \geq d (1+r)),$ which is

    \begin{align*}
        P(Y \geq d (1+r)) \leq \frac{E(\exp(t Y))}{\exp(t d (1 + r))} = \frac{1}{(1-t)^d \exp(t d (1 + r))}.
    \end{align*}

    \noindent
    After plugging in $t=r/(r+1)$ into the equation above, which is the minimizer of the right hand side over $t\geq 0,$ and simplifying, we have

    \begin{align*}
        P(Y \geq d (1+r)) \leq  \frac{(1+r)^d}{\exp(r d)}.
    \end{align*}

    \noindent
    This theorem results from various definitions of $r$ that satisfy $(1+r)^d/\exp(r d)\leq\delta,$ and thus imply $\rho(\delta, d)\leq d (1+r)$ for each such value of $r.$

    Note that

    \begin{align}\label{pbound_int1}
        (1+r)^d /\exp(r d) \leq \delta \iff     (1+r)\exp(-r)\leq\delta^{1/d},
    \end{align}

    \noindent
    so we will focus on bounds for this latter inequality. \cite{topsoe2004some} provides the bound $\log(1+x)\leq x (6+x)/(2(3+2x)),$ which implies 

    \begin{align*}
        & (1+r) \exp(-r)\leq\delta^{1/d} \implies \exp\left(\frac{r (6 + r)}{2 (3 + 2 r)} - r\right) \leq \delta^{1/d} \iff \\
        & \frac{r (6 + r)}{2 (3 + 2 r)} - r \leq \log(\delta)/d \iff r \geq \frac{ \sqrt{2 \log (\delta) (2 \log (\delta)-9 d)}-2 \log (\delta)}{3 d},
    \end{align*}

    \noindent
    and thus this implies the second inequality bound in the statement of the theorem. The bound given in (\ref{bound2}) is simply the closed form solution of (\ref{pbound_int1}) when this inequality holds with equality;

    \begin{align*}
        (1+r)\exp(-r)=\delta^{1/d} \iff r = -1 - W_{-1}(-\delta^{1/d}/e).
    \end{align*}

    \noindent
    Note that using the secondary branch of the Lambert W function ensures that $r\geq 0$ for all $\delta\in [0,1]$ and $d\in\nn_{++}.$
\end{proof}

The next result describes how Theorem \ref{thm:prob_bound} can be used to derive possible values for $\beta.$

\begin{theorem} \label{laplace_smooth_sense}
    \citep{nissim2007smooth} For $\epsilon>0 $ and $\delta\in (0,1),$ the mechanism with output given by $M(\ndat) = f(\ndat) + S_f(\ndat)/\alpha \boldsymbol{Z},$ where $\boldsymbol{Z}[i] \sim \textup{Laplace}(0, 1)$ for each $i\in \{1,\dots, d\},$ satisfies $(\epsilon, \delta)-$differential privacy if $\alpha = \epsilon/2$ and $\beta = \epsilon/(2 \hat{\rho}(\delta, d)),$ where $\hat{\rho}(\delta, d)$ is defined so that $\rho(\delta, d) \leq \hat{\rho}(\delta, d)$ and $\rho(\delta, d)$ is defined as in Theorem \ref{thm:prob_bound}. 

    Also, if $d=1,$ $\beta$ can alternatively be defined as,

    \begin{align}
        \beta = W_{-1}\left(\delta \exp(\epsilon/2) \log \left(\delta\right)\right)- \log \left(\delta\right)-\epsilon/2.
    \end{align}
\end{theorem}
\begin{proof}
    In this proof, for any set $B\subset \rr^d$ and $x\in \rr^1,$ let $x B  = \{ x c : c \in B\}$ and $B + x = \{x + c : c \in B\}.$ We will also use two results from the DP literature. First, \cite{nissim2007smooth} show that for any measurable set $B\subset\rr^d,$ and $\lambda$ such that $\lvert \lambda\rvert \leq \beta,$ we have 

    \begin{align} \label{appCRes_1}
        P(\exp(\lambda) \boldsymbol{Z} \in B) \leq \exp(\epsilon/2) P(\boldsymbol{Z} \in B) + \delta.
    \end{align}

    \noindent
    Second, \cite{dwork2006calibrating} show that for any measurable set $B\subset\rr^d,$ and $\Delta$ such that $\lvert \Delta \rvert \leq \alpha,$ 

    \begin{align} \label{appCRes_2}
        P(\boldsymbol{Z} + \Delta \in B) \leq \exp(\epsilon/2) P(\boldsymbol{Z} \in B).
    \end{align}

    Suppose $\ndatz,\ndat\in\ndatSet$ are neighbors. The result requires that $P(M(\ndat) \in B)\leq \exp(\epsilon) P(M(\ndatz) \in B) + \delta.$ Starting from the left hand side of this inequality, we have,

    \begin{align*}
        & P\left(M(\ndat) \in B \right) = P\left(f(\ndat) + S_f(\ndat)/\alpha \boldsymbol{Z} \in B\right) \\
        & = P\left(\boldsymbol{Z} \in \alpha (B - f(\ndat))/S_f(\ndat)\right) \\
        & = P\left(\boldsymbol{Z} S_f(\ndat)/S_f(\ndatz) \in \alpha (B - f(\ndat))/S_f(\ndatz)\right) \\
        & \leq \exp(\epsilon/2) P\left( \boldsymbol{Z} \in \alpha (B - f(\ndat))/S_f(\ndatz) \right) + \delta \\
        & \leq \exp(\epsilon) P\left(\boldsymbol{Z} - \alpha (f(\ndat) - f(\ndatz))/S_f(\ndatz) \in \alpha (B - f(\ndat))/S_f(\ndatz)\right) + \delta \\
        & = \exp(\epsilon) P\left(\boldsymbol{Z} \in \alpha (B - f(\ndatz))/S_f(\ndatz)\right) + \delta \\
        & = \exp(\epsilon) P\left(M(\ndatz) \in B\right) + \delta,
    \end{align*}

    \noindent
    where the first inequality follows from (\ref{appCRes_1}) because the definition of $S_f(\cdot)$ ensures $S_f(\ndatz)/S_f(\ndat) \in [\exp(-\beta),$ $\exp(\beta)],$ and the second inequality follows from (\ref{appCRes_2}) because $ \lvert \alpha (f(\ndat) - f(\ndatz))/S_f(\ndatz) \rvert \leq   \alpha LS_f(\ndatz)/S_f(\ndatz)$ and $S_f(\ndatz)$ is an upper bound on $LS_f(\ndatz).$

    The final possible definition for $\beta$ in the case in which $d=1$ requires showing (\ref{appCRes_1}) holds for all $\lambda$ such that $\lvert \lambda\rvert$ is at most equal to this value of $\beta.$ We will also only consider the nontrivial case in which $\lambda < 0;$ to see that (\ref{appCRes_1}) is satisfied when $\lambda\geq 0$ for the value of $\beta$ we derive here, see \cite{nissim2007smooth}. 

    First, note that $\log(\delta)$ is the $\delta/2 \leq 1/2$ quantile of the standard Laplace distribution. Let $h(z)=\exp(-\lvert z\rvert )/2,$ and let $G$ denote the event, $\{\lvert \boldsymbol{Z}\rvert \leq \log(\delta) \}.$ Since this density is symmetric, $P(\boldsymbol{Z}\in S)\leq  P(\boldsymbol{Z}\in S \cap G) + P(G^C) = P(\boldsymbol{Z}\in S \cap G) + \delta,$ so it is sufficient to show $\log(\exp(-\beta) h(\exp(-\beta) z) / h(z)) \leq \epsilon/2 $ for all $z$ such that $ z \in [\log(\delta), \log(1/\delta)].$ Since this function is maximized over $ z$ in this range at $z\in \{\log(\delta), \log(1/\delta)\},$ we will define $\beta$ using

    \begin{align*}
        & \log(\exp(-\beta) h(\exp(-\beta) \log(\delta)) / h(\log(\delta))) = \epsilon/2  \iff \\ 
        & \exp(\beta) (2 \beta+\epsilon)=2 \left(\exp(\beta)-1\right) \log \left(1/\delta\right) \iff \\
        & \beta = W_{-1}\left(\delta \exp(\epsilon/2) \log \left(\delta\right)\right)- \log \left(\delta\right)-\epsilon/2.
    \end{align*}
\end{proof}

As described previously, in the simulations presented below, we use the tightest bound implied by this result, which is $\beta = \epsilon/(2 q(1-\delta; d, 1)).$ The next two sections will provide the DP mechanisms that make use of the $\beta-$smooth upper bounds on local sensitivity.

\section{DP Location and Regression Estimators}

The next two subsections will introduce the proposed formally private location and regression estimators, respectively. Specifically, both sections will start by introducing an $(\epsilon,\delta)-$DP estimator that, as described above, requires the user to input an \textit{a priori} feasible set for the estimator, and the proposed $(\epsilon,\delta,\gamma)-$RDP mechanisms are based on estimating a suitable feasible set.

Our results on the $(\epsilon,\delta,\gamma)-$RDP estimators will also provide sufficient conditions for these mechanisms to satisfy \textit{\nr} \textit{(\nrabr)}. This property ensures that, even for cases in which a privacy attacker knows all but one observation, the mechanism does not reveal any unknown observations with certainty, as defined in more detail below.

\begin{definition}
    The mechanism $M:\ndatSet\rightarrow\Theta$ is said to satisfy \textit{\nr} \textit{(\nrabr)} if for all datasets $\ndat\in\ndatSet,$  sets $C\subset\Theta$ such that $P(M(\ndat)\in C)>0,$ privacy attacker information sets $\ndatz \subset \ndat,$ unknown records $\bxi \in \ndat/\ndatz,$ and for all $\bz\in\rr^d,$ we have $P( \bxi =\bz : M(\ndat)\in C) < 1.$
\end{definition}

Note that all pure $\epsilon-$DP mechanisms satisfy this property. In practice, most commonly used approximate $(\epsilon,\delta)-$DP mechanisms satisfy this property as well, but it is also possible to construct counterexamples that do not. For example, consider the $(\epsilon,\delta)-$DP mechanism that outputs approximate total populations of each Census block. Suppose the true counts are given by $\btheta\in\nn^d_+.$ Note that the global sensitivity of this query is $2,$ so one $(\epsilon,\delta)-$DP mechanism could be defined using $\bz\in\rr^d,$ where $\bz[i]\sim \textrm{Laplace}(0, 2/\epsilon),$ and $b>0$ defined so that $P(\lVert \bz \rVert_1 \leq b) = 1-\delta.$ The output of the mechanism can be defined as $\btheta + \bz$ when $\lVert \bz \rVert_1 \leq b$ and $\btheta$ directly otherwise. Since the Laplace distribution is absolutely continuous, whenever this mechanism's output is a vector of integers, a privacy attacker would know that this output is the vector of true counts with probability one.

Despite the fact that there exists examples of $(\epsilon,\delta)-$DP mechanisms that do not satisfy \nrabr, these mechanisms are often avoided in practice. However, as described above, the original example of an $(\epsilon,\delta,\gamma)-$RDP estimator provided by \cite{hall2013random}, which was a sparse histogram estimator, did not satisfy \nrabr. This is because the input dataset itself was used to estimate the sparsity pattern of the population distribution, and no noise was added to histogram counts that were very likely to be zero. In cases in which a privacy attacker knows the data of all but one respondent and the unknown respondent changes the sparsity pattern of the true histogram, the data of the unknown respondent is revealed with certainty. In contrast, ensuring the $(\epsilon,\delta,\gamma)-$RDP estimators proposed here satisfy this property is straightforward because we use the input dataset for a more limited purpose.

\subsection{DP Tukey Median Estimators} \label{sec:HalfSpaceDepthSens}

In this section, we describe both an $(\epsilon,\delta)-$DP and an $(\epsilon,\delta,\gamma)-$RDP Tukey median estimators. Afterward, to take advantage of tighter probability bounds in the univariate case, we also provide a separate $(\epsilon,\delta,\gamma)-$RDP median estimator. Note that we provide more detail in the proofs in this section than the next section, which provides the corresponding results for the $(\epsilon,\delta)-$DP and $(\epsilon,\delta,\gamma)-$RDP deepest regression estimators, since the results for both of these estimators follow from identical arguments.

Even though the results on the deepest regression mechanisms follow from the same arguments as the corresponding results for the Tukey median mechanisms, the assumptions are not identical. Specifically, in the case of our proposed $(\epsilon,\delta)-$DP and $(\epsilon,\delta,\gamma)-$RDP Tukey median estimators when $d\geq 2,$ we will require the additional assumption that the dataset consists of datapoints in general position. This can be ensured without assuming the input dataset consists of draws from an absolutely continuous population distribution using dithering, as is also described in Section \ref{sec:HalfSpaceDepthIntro}. Specifically, each $\bxi\in\ndat$ can be redefined as $\tilde{\bxi} = \bxi + \bzi,$ where $\bzi\in\rr^d$ is a mean zero random variable with a distribution that has a low dispersion and is absolutely continuous. Note that the result of using this technique is that we are providing privacy guarantees to the records given by $\{\tilde{\bxi}\}_i,$ for a given set of realizations of the noise random variable $\{\bzi\}_i,$  rather than directly to $\{\bxi\}_i.$ The benefit of this approach is that it allows us to constrain our view of the universe of input datasets to those that are in general position; however, in use cases that require using sequential composition to release the output of several formally private mechanisms, including a formally private Tukey median estimate, ensuring the privacy guarantee provided by each mechanism holds on the same input dataset would require using $\{\tilde{\bxi}\}_i$ as the input dataset in all of these mechanisms. For this reason, after the next Lemma we describe how this assumption can be removed entirely using an alternative approach. In some cases, particularly when the \textit{a priori} bounds on $\Theta$ are relatively tight and/or the sample size is sufficiently large, the negative impact of using this alternative on the accuracy of the final DP estimator is likely small or even entirely nonexistant.

Both of the mechanisms provided in this section will use the following $\beta-$smooth upper bound on the local sensitivity of the Tukey median estimator.

\begin{lemma} \label{lem:HalfSpaceDepthSens}
    If $\Theta \subset \rr^d$ is closed and bounded and either $d=1$ or $\ndat\in\ndatSet$ consists of datapoints in general position, then

    \begin{align} \label{halfSpaceDepthSens1}
        S_{\hbtheta}(\ndat) = \max_{k\in\{0,\dots, n\}}  \exp(-k \beta) \widetilde{A}_{\hbtheta}^{(k)}(\ndat),
    \end{align}

    \noindent
    where

    \begin{align}
        \widetilde{A}_{\hbtheta}^{(k)}(\ndat) = \left\{
          \begin{array}{lr} 
              \max_{\btheta_1, \btheta_2 \in \Omega(m - 2 k - 2, \ndat)} \lVert\btheta_1-\btheta_2 \rVert_1 & k + 1 \leq \overline{k} \\
              \max_{\btheta_1, \btheta_2 \in \Omega(m - k - \overline{k} - 1, \ndat)} \lVert\btheta_1-\btheta_2 \rVert_1 & k + 1 > \overline{k}
              \end{array}
        \right. ,
    \end{align}

    \noindent
    $m = \max_{\btheta \in \Theta} \hdepth(\btheta, \ndat),$  $\Omega(k, \ndat) = \{ \btheta\in \Theta : \hdepth(\btheta, \ndat) \geq k\},$ and $\overline{k}= m - \lceil n/(1+d) \rceil,$ is a $\beta-$smooth upper bound on the local sensitivity of the Tukey median.

    Also, $S_{\hbtheta}(\ndat)$ can be found by only considering datapoints in the supremum in the definition of $\widetilde{A}_{\hbtheta}^{(k)}(\ndat),$ in the sense that

    \begin{align} \label{halfSpaceDepthSens2}
        S_{\hbtheta}(\ndat) =  \max_{k\in\{0,\dots,  n  \}}  \exp(-k \beta) \widehat{A}_{\hbtheta}^{(k)}(\ndat), 
    \end{align}

    \noindent
    where 

    \begin{align}
        \widehat{A}_{\hbtheta}^{(k)}(\ndat) = \left\{
            \begin{array}{lr} 
                \max_{\btheta_1, \btheta_2 \in \Omega(m - 2 k -2, \ndat)  \cap \widetilde{\ndat}} \lVert\btheta_1-\btheta_2 \rVert_1 & k + 1 \leq \overline{k} \\
                \max_{\btheta_1, \btheta_2 \in \Omega(m - k - 1 - \overline{k} , \ndat) \cap \widetilde{\ndat}} \lVert\btheta_1-\btheta_2 \rVert_1 & k + 1 \in (\overline{k}, m - \overline{k}) \\
                \max_{\btheta_1, \btheta_2 \in \Theta} \lVert\btheta_1-\btheta_2 \rVert_1 & k + 1 \geq m - \overline{k}
            \end{array}
        \right.
    \end{align}

    \noindent
    and $\widetilde{\ndat}$ consists of datapoints in $\ndat$ and all points at which a depth contour intersect the boundary of $\Theta.$
\end{lemma}
\begin{proof}
    First we will consider the impact on the halfspace depth of $\btheta\in\Theta$ of removing the observation $\bxi\in\rr^d$ from $\ndat\in\ndatSet.$ There are two possibilities that we will consider, which both follow directly from the definition of halfspace depth, \textit{i.e.}, Definition \ref{ldepthDef}. First, we may have $\hdepth(\btheta, \ndat/\{\bxi\})=\hdepth(\btheta, \ndat),$ which occurs if and only if there is at least one halfspace in $\rr^d$ that has $\btheta$ on its boundary, contains $\hdepth(\btheta, \ndat)$ datapoints, and that does not contain the datapoint $\bxi.$ Second, in all other cases, removing this observation results in halfspace depth decreasing by one, so we have, $\hdepth(\btheta, \ndat/\{\bxi\})=\hdepth(\btheta, \ndat) - 1$ in this case. Note that removing a record cannot increase the halfspace depth, since this operation either does not change or decreases the number of datapoints in each halfspace with a boundary that contains $\btheta.$ Similar logic can be used to show that adding the observation $\bz\in\rr^d$ to $\ndat$ will either leave the halfspace depth of $\btheta$ unchanged or will increase the halfspace depth by one. 
    
    Taking the two observations in the preceding paragraph together implies that the operation defined by substituting $\bz$ for $\bxi$ results in $\argmax_{\btheta\in\Theta} \hdepth(\btheta, \ndat/ \bxi \cup \bz ) \subset \Omega(m-2, \ndat).$ This same logic can be applied $k>1$ times to show that, for any $\ndatz\in\ndatSet$ such that $\hamming(\ndat, \ndatz)=2k,$ we have $\argmax_{\btheta\in\Theta} \hdepth(\btheta, \ndatz) \subset \Omega(m-2k, \ndat).$ Let $\ndatz\in\ndatSet$ be a neighbor of $\ndatzz,$ $\hbtheta(\ndatz)\in\Theta$ be defined as any element of $\argmax_{\btheta\in\Theta} \hdepth(\btheta, \ndatz),$ and $\hbtheta(\ndatzz)\in\Theta$ be defined as any element of $\argmax_{\btheta\in\Theta} \hdepth(\btheta, \ndatzz).$ Since any neighbor $\ndatzz$ of $\ndatz$ satisfies $\hamming(\ndat, \ndatzz) \leq 2k + 2,$ we have $\hbtheta(\ndat), \hbtheta(\ndatzz) \in \Omega(m- 2k - 2, \ndat).$ This implies that one upper bound on the local sensitivity at distance $k$ is given by $\max_{\btheta_1, \btheta_2 \in \Omega(m - 2 k -2, \ndat)} \lVert\btheta_1-\btheta_2 \rVert_1,$ which is the bound defining $\widetilde{A}_{\hbtheta}^{(k)}(\ndat)$ for cases in which $k + 1 \leq\overline{k}.$

    Next we will derive a tighter upper bound on the local sensitivity at distance $k$ for any dataset $\ndat\in\ndatSet$ that satisfies $m = \lceil n/(1+d) \rceil$ by again considering the operation defined by substituting $\bz$ for $\bxi$ in this case. Suppose there exists $\btheta\in\Theta$ such that $\hdepth(\btheta, \ndat)=m-2.$ Since, for any $\btheta\in\Theta,$ removing an observation cannot increase the halfspace depth of $\btheta$ and adding an observation cannot decrease  the halfspace depth of $\btheta,$ the event that $\btheta\in \argmax_{\btheta\in\Theta} \hdepth(\btheta, \ndat / \bxi \cup \bz)$ is only possible if both $\max_{\btheta\in\Theta} \hdepth(\btheta, \ndat / \bxi\cup \bz)=m-1=\lceil n/(1+d) \rceil-1,$ and $\hdepth(\btheta, \ndat / \bxi\cup \bz) = m-1=\lceil n/(1+d) \rceil-1.$ However, since $\ndat/ \bxi\cup \bz\in\ndatSet$ and $\ndat/ \bxi\cup \bz$ is in general position, we have $\max_{\btheta\in\Theta} \hdepth(\btheta, \ndat / \bxi\cup \bz)\geq \lceil n/(1+d) \rceil,$ so the first of these two conditions is not possible \citep{donoho1992breakdown}. This implies that, for any $\ndat\in\ndatSet$ such that $\max_{\btheta\in\Theta} \hdepth(\btheta, \ndat)= \lceil n/(1+d) \rceil,$ we have $\argmax_{\btheta\in\Theta} \hdepth(\btheta, \ndat/ \bxi \cup \bz )$ must be a subset of $\Omega(m-1, \ndat).$ Thus, in the context of arbitrary $\ndat\in\ndatSet,$ not necessarily satisfying  $m = \lceil n/(1+d) \rceil,$ we only need to use the less tight bound on the local sensitivity at distance $k$ described in the preceding paragraph when $k \leq m - \lceil n/(1+d) \rceil-1 \iff k + 1 \leq\overline{k},$ and for any $k> m - \lceil n/(1+d) \rceil-1,$ we instead use the tighter bound on the local sensitivity at distance $k$ given by $\max_{\btheta_1, \btheta_2 \in \Omega(m - k - \overline{k} - 1, \ndat)} \lVert\btheta_1-\btheta_2 \rVert_1,$ as described in the definition of $\widetilde{A}^{(k)}_{\hbtheta}(\ndat)$ above.

    Thus, $\widetilde{A}_{\hbtheta}^{(k)}(\ndat)$ satisfies both requirements of Lemma \ref{lem:nissim_claim_3p2}, which implies (\ref{halfSpaceDepthSens1}) is a $\beta-$smooth upper bound on the local sensitivity of $\hbtheta.$ Also, the equivalence of (\ref{halfSpaceDepthSens1}) and (\ref{halfSpaceDepthSens2}) is a consequence of Lemma \ref{lem:diamHalfspace}.
\end{proof}
\begin{remark} \label{rem:remove_gen_pos_1}
    As described above, the assumption that the input dataset is in general position can also be removed entirely. This can be done by noting that the first bound in the proof did not use the fact that the input dataset was in general position. Thus, we can instead define $\widetilde{A}_{\hbtheta}^{(k)}(\ndat)$ and $\widehat{A}_{\hbtheta}^{(k)}(\ndat)$ as
    
    \begin{align*}
        \widetilde{A}_{\hbtheta}^{(k)}(\ndat) = \max_{\btheta_1, \btheta_2 \in \Omega(m - 2 k - 2, \ndat)} \lVert\btheta_1-\btheta_2 \rVert_1 \\
        \widehat{A}_{\hbtheta}^{(k)}(\ndat) = \left\{
            \begin{array}{lr} 
                \max_{\btheta_1, \btheta_2 \in \Omega(m - 2 k -2, \ndat)  \cap \widetilde{\ndat}} \lVert\btheta_1-\btheta_2 \rVert_1 & 2 k < m - 2 \\
                \max_{\btheta_1, \btheta_2 \in \Theta} \lVert\btheta_1-\btheta_2 \rVert_1 & 2 k \geq m - 2
            \end{array}
        \right. ,
    \end{align*}

    \noindent
    which are less tight bounds, to avoid this assumption entirely. Note that this definition will not alter the resulting $\beta-$smooth upper bound on the local sensitivity in some cases, particularly for input datasets with maximum halfspace depth that is sufficiently higher than $\lceil n/(1+d)\rceil.$ For many data generating processes, this condition occurs as the sample size diverges. \cite{donoho1992breakdown} provide sufficient conditions for the maximum halfspace depth to converge to a higher value than this lower bound. We are not aware of a case in which the lower bound on the maximum halfspace depth of $\lceil n/(1+d)\rceil$ fails to hold for a dataset that is not in general position, but we are also not aware of a similar result that avoids this assumption.

\end{remark}
\begin{remark} \label{rem:half_space_depth_sense}
    Note that the formula for $S_{\hbtheta}(\ndat)$ given in Lemma \ref{lem:HalfSpaceDepthSens} is not generally the lowest possible upper bound. For example, in the case in which $d=1,$ \cite{nissim2007smooth} provide a tight bound. In our notation, and with bounds specified on the feasible set of candidates rather than the observations, this bound is given by

    \begin{align} \label{medianSmoothSense}
        S^\star_{\textup{med}}(\ndat) = \max_{k\in\{0,1,\dots p+1\}} \exp(-\beta k) \max_{t\in \{0,\dots, k+1\}} \min(b, x_{(p+t)}) - \max(a, x_{(p+t-k-1)}),
    \end{align}

    \noindent
    where the median is defined so that it has order statistic $p = \lfloor n / 2 \rfloor,$ $\Theta = [a, b]\subset\rr,$ and, for notational convenience, for any $k\leq 0$ we let $x_{(k)}=-\infty,$ and, for any $k\geq n+1,$ we let $x_{(k)}=\infty$  \citep{nissim2007smooth}. In contrast, in this case our bound can also be written as

    \begin{align} \label{medianSmoothSenseNotTight}
        S_{\textup{med}}( \ndat) = \max_{k\in\{0,\dots, n\}} \exp(-\beta k) \left(\min (b, x_{(p+k+1)}) - \max(a, x_{(p-k-1)})\right).
    \end{align}

    \noindent
    Note that below we also provide an example that makes the equivalence of (\ref{medianSmoothSenseNotTight}) and (\ref{halfSpaceDepthSens1}) more clear. Setting aside the difference that here we specify bounds on the estimator rather than on the observations themselves, \cite{nissim2007smooth} describe an equivalent upper bound in their Claim 3.4, and also show this relaxation inflates the resulting $\beta-$smooth sensitivity bound by at most a multiplicative factor of two. Note that \cite{nissim2007smooth} provide an algorithm that can be used to evaluate equation (\ref{medianSmoothSense}) in $O(n\log(n))$ time.
\end{remark}

\begin{example}
    In this example we will consider the input dataset $\ndat=\{0,0,0,0,3\}$ and $\Theta = [-1, 1]$ and show that the $\beta-$smooth upper bound on the local sensitivity provided in Lemma \ref{lem:HalfSpaceDepthSens} can be found using (\ref{medianSmoothSenseNotTight}) in this case. We will use notation from Remark \ref{rem:half_space_depth_sense} and Lemma \ref{lem:HalfSpaceDepthSens} in this example.

    First, note that the depth of $\btheta=0$ is $4$ because the interval $(-\infty, 0]$ contains four observations, and the interval $[0,\infty)$ contains five observations. This logic can be repeated at each $\btheta\in\Theta$ to show that $\btheta=0$ is the unique maximizer of the halfspace depth in $\Theta,$ so we have $\hdepth(0,\ndat)=4=m.$ Also, for any $\btheta\in[-1,0)$ we have $\hdepth(\btheta,\ndat)=0,$ and for any $\btheta\in (0,1]$ we have $\hdepth(\btheta,\ndat)=1.$ Thus, for each $k\in\{2,3,4\},$ we have $\Omega(k, \ndat) \cap \widetilde{\ndat} = \{0\},$ and we also have $\Omega(1, \ndat) \cap \widetilde{\ndat} = \{0, 1\}.$ Also, note that $\overline{k}= m - \lceil n/(1+d) \rceil=1.$ After substituting these terms into the definition of $\widehat{A}^{(k)}_{\hbtheta}(\ndat),$ we have,

    \begin{align*}
        \widehat{A}^{(0)}_{\hbtheta}(\ndat) &= \max_{\btheta_1, \btheta_2 \in \Omega(2, \ndat)  \cap \widetilde{\ndat}} \lVert\btheta_1-\btheta_2 \rVert_1 = \min (b, x_{(4)}) - \max(a, x_{(2)}) = 0 \\
        \widehat{A}^{(1)}_{\hbtheta}(\ndat) &= \max_{\btheta_1, \btheta_2 \in \Omega(1, \ndat) \cap \widetilde{\ndat}} \lVert\btheta_1-\btheta_2 \rVert_1 = \min (b, x_{(5)}) - \max(a, x_{(1)}) = 1 \\
        \widehat{A}^{(2)}_{\hbtheta}(\ndat) &= \max_{\btheta_1, \btheta_2 \in \Theta} \lVert\btheta_1-\btheta_2 \rVert_1 = \min (b, \infty) - \max(a, -\infty) = 2,
    \end{align*}
    \noindent
    where, in the final line, we used the substitution $x_{(6)}=\infty$ and $x_{(0)}=-\infty.$
\end{example}

Algorithm \ref{alg:DPHalfspaceDepth} provides a mechanism to approximate the Tukey median that satisfies $(\epsilon,\delta)-$DP, which is an immediate consequence of the previous lemma and the main result of \cite{nissim2007smooth}, as stated in the following corollary. Note that it is also possible to use the $\beta-$smooth upper bound on the local sensitivity of the Tukey median provided in Lemma \ref{lem:HalfSpaceDepthSens} to formulate pure $\epsilon-$DP mechanisms, at the cost of the distribution of the resulting estimator not having exponential tails, as is described by \cite{nissim2007smooth} in more detail.

\begin{corollary} \label{cor:DPTukeyMedianDP}
    If $\Theta \subset \rr^d$ is closed and bounded and $\ndat\in\ndatSet$ consists of datapoints in general position then Algorithm \ref{alg:DPHalfspaceDepth} satisfies $(\epsilon,\delta)-$DP.
\end{corollary}
\begin{proof}
    This follows from Lemma \ref{lem:diamHalfspace} and Theorem \ref{laplace_smooth_sense}.
\end{proof}

\begin{algorithm} \label{alg:DPHalfspaceDepth}
    \DontPrintSemicolon
    \SetKwInOut{Output}{return}
    \SetKwInOut{Input}{input}
    \tcp{Define $\alpha,\beta$ using Theorem \ref{laplace_smooth_sense}:}
    $\beta \gets \epsilon/(2 \hat{\rho}(\delta, d))$ \;
    $\alpha \gets \epsilon/2$ \;
    $\hbtheta \gets $ Calculate\_Tukey\_Median$(\ndat)$ \;
    \tcp{Define $S_{\hbtheta}(\ndat)$ using equation (\ref{halfSpaceDepthSens2}) from Lemma \ref{lem:HalfSpaceDepthSens} when $d\geq 2$ or equation (\ref{medianSmoothSense}) when $d=1,$ which follows from \citep{nissim2007smooth}:}
    $S_{\hbtheta}(\ndat) \gets \textup{Find\_}\beta\textup{\_Smooth\_Local\_Sensitivity\_Bound}(\beta, \ndat, \Theta) $\;
    \Output{$\hbtheta + $Laplace$(\boldsymbol{\mu} = \boldsymbol{0}, b=S_{\hbtheta}(\ndat)/\alpha)$}
    \Input{$(\ndat, \Theta, \epsilon, \delta)$}
    \caption{DPTukeyMedian}
\end{algorithm}

The next theorem also provides a method to define the set $\Theta$ that is the input into  DPTukeyMedian$(\cdot)$ in a data dependent manner while satisfying $(\epsilon, \delta, \gamma)-$RDP when $d\geq 2.$ Note that we treat the case in which $d=1$ separately in Theorem \ref{thm:median_rdp}.

\begin{theorem} \label{thm:tukeyMedian_rdp}
    Suppose $d\geq 2,$ the dataset $\ndat\in\ndatSet$ consists of independent observations from the population distribution $\Ndat,$ and $\ndat\in\ndatSet$ consists of datapoints in general position. Let

    \begin{align*}
        \kappa^\star = \frac{\sqrt{4-2 (n-2) \log \left( \gamma /\left(16 \left(\left(n^2-1\right)^d+1\right)\right) \right)}+2}{2 (n-2)}.
    \end{align*}

    If $\lceil 2 n \kappa^\star \rceil < \lceil  n/(1+d)\rceil,$ then $(\ndat, \delta, \epsilon) \mapsto$ DPTukeyMedian$(\ndat, \widetilde{\Theta}, \delta, \epsilon),$ where $\widetilde{\Theta} = \{ \btheta\in \Theta  : \hdepth(\btheta, \ndat) \geq \hdepth(\hbtheta, \ndat) - 2 n \kappa^\star\},$ satisfies $(\epsilon, \delta, \gamma)-$RDP.

    In addition, this mechanism satisfies \nrabr.
\end{theorem}
\begin{proof}
    Since this proof considers the deepest regression for different datasets, we will denote the deepest regression for $\ndat\in\ndatSet$ by $\hbtheta(\ndat).$ We will also define the halfspace depth of $\btheta\in\Theta$ for the population distribution as $\hdepth(\btheta, \Ndat) = \sup_{\bu\neq \boldsymbol{0}} L(\btheta, \bu),$ where $L(\btheta, \bu) = \int \mathbf{1}_{[0,\infty)} \left( \bu^\top (\bx - \btheta) \right) \textup{d} \Ndat (\bx).$ Lemma \ref{LcontourBound} in Appendix \ref{appendixA} and the definition of $\kappa^\star$ imply

    \begin{align*}
        & P\left(\lvert \hdepth(\hbtheta(\Ndat), \Ndat) - \hdepth(\hbtheta( \ndat), \ndat)/n \rvert \geq \kappa^\star \right) \leq 8 \left((n^2-1)^d + 1\right) \exp\left(2 \kappa^\star (2+(2-n) \kappa^\star)\right) \\
        & \iff  P\left( \lvert \hdepth(\hbtheta(\Ndat), \Ndat) - \hdepth(\hbtheta( \ndat), \ndat)/n \rvert \geq \kappa^\star\right) \leq \gamma/2.
    \end{align*}

    \noindent
    Since a similar bound can be constructed for any alternative dataset $\ndatz \in\ndatSet,$ also defined by independent observations from $\Ndat,$ we have

    \begin{align*}
        & P\left(\lvert \hdepth(\hbtheta(\ndatz), \ndat)/n - \hdepth(\hbtheta(\ndat), \ndat)/n \rvert \geq 2 \kappa^\star \right) \\
        &  \leq P\left( \lvert  \hdepth^\star(\hbtheta(\Ndat), \Ndat) - \hdepth(\hbtheta(\ndatz), \ndatz)/n \rvert \geq \kappa^\star \right)  + P\left(\lvert \hdepth^\star(\hbtheta(\Ndat), \Ndat) - \hdepth(\hbtheta(\ndat), \ndat)/n \rvert  \geq \kappa^\star \right)  \\
        & \leq  \gamma/2 + \gamma/2 = \gamma.
    \end{align*}

    \noindent
    Thus, for $\widetilde{\Theta},$ defined using $\ndat,$ and any such dataset $\ndatz,$ we have $P(\hbtheta(\ndatz) \in \widetilde{\Theta})\geq 1-\gamma.$ Since the mechanism DPTukeyMedian$(\cdot),$ is $(\epsilon, \delta)-$DP when $\Theta$ is deterministic, this implies $(\ndat, \delta, \epsilon) \mapsto$ DPTukeyMedian$(\ndat, \widetilde{\Theta}, \delta, \epsilon),$ is $(\epsilon, \delta, \gamma)-$RDP.

    Note that $\lceil 2 n \kappa^\star \rceil < \lceil n/(1+d) \rceil$ is a sufficient condition to ensure $\widetilde{\Theta}$ is bounded with probability one because, for any sample of observations in general position, we have $\hdepth(\hbtheta, \ndat) \geq \lceil n/(1+d)\rceil$ \citep{donoho1992breakdown}.
    
    Note that the additional use of the data in this mechanism, \textit{i.e.}, defining $\widetilde{\Theta},$ only impacts the output distribution through the scale of the Laplace noise, and the mechanism satisfies \nrabr whenever $\widetilde{\Theta}$ is not equal to a singleton set. Thus, our assumption that the data are in general position is sufficient to ensure the mechanism satisfies \nrabr.
\end{proof}
\begin{remark}
    Recall that Remark \ref{rem:remove_gen_pos_1} described how the assumption that the datapoints are in general position can be removed. This assumption is used in the result above in two places. First, this assumption was used to derive the sample size bound that is required for the mechanism. In the absence of a generalization of the lower bound $\hdepth(\hbtheta, \ndat) \geq \lceil n/(1+d)\rceil$ provided by \cite{donoho1992breakdown} that avoids the assumption that the datapoints are in general position, one could resort to stating the privacy guarantee as being conditional on the assumption that $\hdepth(\hbtheta, \ndat) \geq \lceil n/(1+d)\rceil$ holds for datasets that are not in general position. Second, we also use this assumption to show the mechanism satisfies \nrabr. This use of the assumption can be avoided more easily than the first; for example, this can be done by simply redefining $\widetilde{\Theta}$ as an enlargement, as described in more detail in Theorem \ref{thm:deepestReg_rdp}.
\end{remark}

The final result of this section improves on Theorem \ref{thm:tukeyMedian_rdp} for the case in which $d=1.$ Note that the assumption that the datapoints are in general position is not required for this Theorem.

\begin{theorem} \label{thm:median_rdp}
    Suppose $d=1,$ the dataset $\ndat\in\ndatSet$ consists of independent observations from the population distribution $\Ndat,$ and 

    \begin{align*}
        n \geq 2 \left(-\log (\gamma)+\sqrt{(\log (\gamma)-4) \log (\gamma)+3}+2\right).
    \end{align*}

    \noindent
    Also, let $c>0$ and 

    \begin{align*}
        \kappa^\star = \sqrt{\frac{1-\log(\gamma)}{n}}.
    \end{align*}

    Then, the mechanism defined as $(\ndat, \delta, \epsilon) \mapsto$ DPTukeyMedian$(\ndat, \widetilde{\Theta}, \delta, \epsilon),$ where $\widetilde{\Theta} = [x_{(\lfloor n/2 - n \kappa^\star \rfloor )} - c, x_{(\lceil n/2 + n \kappa^\star \rceil)} + c],$ satisfies $(\epsilon, \delta, \gamma)-$RDP for any $c\geq 0.$
    
    Also, the mechanism satisfies \nrabr  for any $c > 0.$
\end{theorem}
\begin{proof}
    Let the empirical distribution function for the dataset $\ndat\in\ndatSet$ be denoted by $T_n(x, \ndat) = \frac{1}{n} \sum_{x_i\in\ndat} \mathbf{1}_{(-\infty, x]}(x_i)$ and the median of the dataset $\ndat$ by $\hbtheta(\ndat).$ This result follows from inverting the two sample Kolmogorov–Smirnov test. Specifically, for $\ndat,\ndatz\in\ndatSet$ composed of independent draws from the same population distribution, the two sample Kolmogorov–Smirnov test implies the two sample Dvoretzky-Kiefer-Wolfowitz type inequality \citep{dvoretzky1956asymptotic}
    
    \begin{align*}
        P(\sup_{x} \lvert T_n(x,\ndat) -  T_n(x,\ndatz) \rvert \geq \kappa ) \leq C \exp(-n \kappa^2),
    \end{align*}

    \noindent
    where $C>0$ is a constant. \cite{wei2012two} show that using $C=e,$ \textit{i.e.}, setting $C$ to Euler's number, is sufficient for all sample sizes.

    The definition of $\kappa^\star$ implies

    \begin{align*}
        P(\sup_{x} \lvert T_n(x,\ndat) -  T_n(x,\ndatz) \rvert \geq \kappa^\star ) \leq \gamma
    \end{align*}

    \noindent
    This implies,

    \begin{align*}
        P(T_n(\hbtheta(\ndat),\ndatz) \in [1/2 - \kappa^\star, 1/2 + \kappa^\star]) \geq 1-\gamma  \iff P(\hbtheta(\ndatz)  \in [x_{(\lfloor n/2 - n \kappa^\star \rfloor )}, x_{(\lceil n/2 + n \kappa^\star \rceil)}]) \geq 1-\gamma 
    \end{align*}

    Note that the lower bound on the sample size holds if and only if $ n/2 - n \kappa^\star \geq 1,$ so this condition ensures that $\widetilde{\Theta}$ is bounded.
    
    The final result follows from the same logic that was used in Theorem \ref{thm:tukeyMedian_rdp}, as $c>0$ is a sufficient condition to ensure $\widetilde{\Theta}$ is not a singleton set.
\end{proof}

\subsection{DP Regression Estimators} \label{sec:DPRegDepthSens}

In this section, we will provide deepest regression mechanisms that satisfy formal privacy guarantees. We will start by providing the regression depth counterpart to Lemma \ref{lem:HalfSpaceDepthSens}. 

\begin{lemma} \label{lem:regDepthSens}
    If $\Theta \subset \rr^d$ is closed and bounded, then

    \begin{align} \label{regDepthSens1}
        S_{\hbtheta}(\ndat) = \max_{k\in\{0,1,\dots\}}  \exp(-k \beta) \widetilde{A}_{\hbtheta}^{(k)}(\ndat),
    \end{align}

    \noindent
    where

    \begin{align}
        \widetilde{A}_{\hbtheta}^{(k)}(\ndat) = \left\{
          \begin{array}{lr} 
              \max_{\btheta_1, \btheta_2 \in \Omega(m - 2 k - 2, \ndat)} \lVert\btheta_1-\btheta_2 \rVert_1 & k + 1 \leq \overline{k} \\
              \max_{\btheta_1, \btheta_2 \in \Omega(m - k - 1 - \overline{k}, \ndat)} \lVert\btheta_1-\btheta_2 \rVert_1 & k + 1 > \overline{k}
              \end{array}
        \right. ,
    \end{align}

    \noindent
    $m = \max_{\btheta \in \Theta} \rdepth(\btheta, \ndat),$  $\Omega(k, \ndat) = \{ \btheta\in \Theta : \rdepth(\btheta, \ndat) \geq k\},$ and $\overline{k}= m - \lceil n/(1+d) \rceil,$ is a $\beta-$smooth upper bound on the local sensitivity of the deepest regression.

    Also, $S_{\hbtheta}(\ndat)$ can be found by only considering fits that pass through each combination of $d$ datapoints in the maximum in the definition of $\widetilde{A}_{\hbtheta}^{(k)}(\ndat),$ in the sense that
    \begin{align} \label{regDepthSens2}
        S_{\hbtheta}(\ndat) =  \max_{k\in\{0, 1, \dots\}}  \exp(-k \beta) \widehat{A}_{\hbtheta}^{(k)}(\ndat), 
    \end{align}

    \noindent
    where

    \begin{align}
        \widehat{A}_{\hbtheta}^{(k)}(\ndat) = \left\{
            \begin{array}{lr} 
                \max_{\btheta_1, \btheta_2 \in \Omega(m - 2 k - 2, \ndat)  \cap \mathcal{A}(\ndat)} \lVert\btheta_1-\btheta_2 \rVert_1 & k + 1 \leq \overline{k} \\
                \max_{\btheta_1, \btheta_2 \in \Omega(m - k - 1 - \overline{k} , \ndat) \cap \mathcal{A}(\ndat)} \lVert\btheta_1-\btheta_2 \rVert_1 & k+1 \in (\overline{k}, m - \overline{k}) \\
                \max_{\btheta_1, \btheta_2 \in \Theta} \lVert\btheta_1-\btheta_2 \rVert_1 & k + 1 \geq m - \overline{k}
            \end{array}
        \right.
    \end{align}

    \noindent
    and $\mathcal{A}(\ndat)\subset \Theta$ is the set of $\btheta \in \Theta$ such that either $\card \{(\bxi,y_i)\in\ndat : (1, \bxi^\top) \btheta = y_i\} \geq d$ or $\btheta\in\Theta$ is in an intersection of the boundary of $\Theta$ and a boundary of a depth contour.
\end{lemma}
\begin{proof}
    The proof of the first inequality follows from the same logic as Lemma \ref{lem:HalfSpaceDepthSens}. The equivalency between equation (\ref{regDepthSens1}) and (\ref{regDepthSens2}) follows from Lemma \ref{lem:diamReg}.
\end{proof}

Algorithm \ref{alg:DPDeepestReg} provides an $(\epsilon,\delta)-$DP mechanism for approximating the deepest regression, which is an immediate consequence of the Lemma \ref{lem:regDepthSens} and the main result of \cite{nissim2007smooth}, as stated in the following corollary. Note that the same comment preceding Corollary \ref{cor:DPTukeyMedianDP} also holds in this case; it is also straightforward to use Lemma \ref{lem:regDepthSens} to formulate pure $\epsilon-$DP deepest regression estimators \citep{nissim2007smooth}.

\begin{corollary}
    If $\Theta \subset \rr^d$ is closed and bounded, DPDeepestReg$(\ndat, \Theta, \epsilon, \delta),$ as described in Algorithm \ref{alg:DPDeepestReg}, satisfies $(\epsilon,\delta)-$DP.
\end{corollary}
\begin{proof}
    This follows from Lemma \ref{lem:regDepthSens} and Theorem \ref{laplace_smooth_sense}.
\end{proof}

\begin{algorithm} \label{alg:DPDeepestReg}
\DontPrintSemicolon
\SetKwInOut{Output}{return}
\SetKwInOut{Input}{input}
\tcp{Define $\alpha,\beta$ using Theorem \ref{laplace_smooth_sense}}
$\beta \gets \epsilon/(2 \hat{\rho}(\delta, d))$ \;
$\alpha \gets \epsilon/2$ \;
$\hbtheta \gets $ Calculate\_Depth\_Regression$(\ndat)$ \;
\tcp{Define $S_{\hbtheta}(\ndat)$ using equation (\ref{regDepthSens2}) from Lemma \ref{lem:regDepthSens}:}
$S_{\hbtheta}(\ndat) \gets \textup{Find\_}\beta\textup{\_Smooth\_Local\_Sensitivity\_Bound}(\beta, \ndat) $\;
\Output{$\hbtheta + $Laplace$(\boldsymbol{\mu} = \boldsymbol{0}, b=S_{\hbtheta}( \ndat)/\alpha)$}
\Input{$(\ndat, \Theta, \epsilon, \delta)$}
\caption{DPDeepestReg}
\end{algorithm}

The next theorem also provides a method to set the input $ \Theta$ of DPDeepestReg$(\ndat, \Theta, \epsilon, \delta)$ using the data directly while satisfying $(\epsilon, \delta, \gamma)-$RDP. This mechanism can be made to satisfy \nrabr by simply expanding this feasible set input, as described in more detail in the theorem. The simulations in Section \ref{sec:simulations} provide evidence that even a large expansion of this feasible set input may not have any impact on the distribution of the resulting estimator when the sample size is sufficiently large.

\begin{theorem} \label{thm:deepestReg_rdp}
    Suppose the dataset $\ndat\in\ndatSet$ consists of independent observations from the population distribution $\Ndat.$ Let

    \begin{align*}
        \kappa^\star = \frac{\sqrt{4-2 (n-2) \log \left(\gamma/\left(128 \left(\left(n^2-1\right)^{d-1}+1\right)^4\right) \right)}+2}{2 (n-2)}.
    \end{align*}

    If $\lceil 2 n \kappa^\star \rceil < \lceil n/(1+d) \rceil,$ then $(\ndat, \delta, \epsilon) \mapsto$ DPDeepestReg$(\ndat, \widetilde{\Theta}, \delta, \epsilon),$ where

    \begin{align}
        \widetilde{\Theta} = \{ \btheta + \bz : \btheta\in \Theta, \; \rdepth(\btheta, \ndat) \geq \rdepth(\hbtheta, \ndat) - 2 n \kappa^\star, \; \bz \in\rr^d, \; \lVert \bz\rVert_p \leq c \},
    \end{align}

    \noindent
    $p\in [1, \infty],$ and $c\geq 0,$ satisfies $(\epsilon, \delta, \gamma)-$RDP. 

    Also, if $c>0,$ the mechanism satisfies \nrabr. 
\end{theorem}
\begin{proof}
    This can be proved in a similar manner as Theorem \ref{thm:tukeyMedian_rdp}. In this case Lemma \ref{RcontourBound} in Appendix \ref{appendixB} provides the required uniform error bound. Also, as in Theorem \ref{thm:tukeyMedian_rdp}, the bound on $\kappa^\star,$ \textit{i.e.}, $\lceil 2 n \kappa^\star \rceil < \lceil n/(1+d) \rceil,$ ensures $\widetilde{\Theta}$ is bounded because $\rdepth(\hbtheta, \ndat) \geq \lceil n/(1+d)\rceil$ \citep{mizera2002depth,amenta2000regression}.

    The final part of the result follows from $c>0$ being sufficient to ensure $\widetilde{\Theta}$ is not a singleton set.
\end{proof}

\subsubsection{Approximate Deepest Regression} \label{sec:approxDeepestReg}

As described previously, the deepest regression methods described in the previous section are most often limited to the case in which $d=2,$ so this section provides a computationally efficient method for computing an $(\epsilon,\delta)-$DP approximate deepest regression estimate called Medsweep that was proposed by \cite{van2002deepest}. In this section it will be helpful to define $X\in\rr^{n\times (d-1)}$ so that $X[i,\cdot]=\bxi$ and $\by\in \rr^n$ so that $\by[i]=y_i.$

Informally, the method can be viewed as similar to applying the Frish-Waugh-Lovell theorem when finding an ordinary least squares estimator, in the sense that the method proceeds by parsing (or ``sweeping") $X[\cdot,i]$ out of $X[\cdot, k]$ by approximating the univariate regression of $X[\cdot,i]$ on $X[\cdot, k]$ and then redefining $X[\cdot,k]$ as the residual of this regression. Afterward, a similar approach is used to iteratively parse each column $X[\cdot,k]$ out of $\by.$

There are a few  possible ways to formulate DP Medsweep mechanisms. The simplest approach is to replace each median evaluation within Medsweep with a DP median mechanism. Here we will describe a slightly more involved formulation that has the advantage of providing a DP estimate with exponential tails conditional on the data. Specifically, this requires formulating DP variants of two methods that are used within Medsweep. The first is simply the median; below we use the $\beta-$smooth sensitivity of the median provided by \cite{nissim2007smooth} to define the scale of a Laplace mechanism, as described in equation (\ref{medianSmoothSense}).

The second method is $r:\rr^n\times \rr^n \rightarrow \rr,$ and is defined as

\begin{align*}
    r(\boldsymbol{u}, \boldsymbol{v}) = \textup{med} \left(\frac{\boldsymbol{u} - \textup{med}(\boldsymbol{u})}{\boldsymbol{v} - \textup{med}(\boldsymbol{v})}\right).
\end{align*}

\noindent
The following lemma provides the $\beta-$smooth upper bound on the local sensitivity of $r(\cdot)$ when its output is truncated to the range $[L, U]\subset\rr.$ This method is used in two places in Medsweep algorithm, with inputs $\bu,\bv\in \rr^n$ are given by either columns of $X$ or one column of $X$ and the vector $\by.$ Note that the final DP mechanism actually imposes bounds on the output of these intermediate mechanisms, rather than the final estimator.

\begin{lemma} \label{lem:ssMedsweepPreliminary}
    Suppose $\boldsymbol{u},\boldsymbol{v}\in\rr^n$ and $L,U\in\rr$ satisfy $L<U.$ Let $m=\lfloor  n/2 \rfloor$ and
    
    \begin{align*}
        \textup{clip}(z, (L,U)) = \left\{
        \begin{array}{lr} z & z \in [L, U] \\
        L & z < L \\
        U & z > U 
        \end{array}
        \right..
    \end{align*}
    
    \noindent
    Given $\bu, \bv \in\rr^n,$ a $\beta-$smooth upper bound on the local sensitivity of $c(\boldsymbol{u},\boldsymbol{v},L,U) = \textup{clip}(r(\boldsymbol{u},\boldsymbol{v}), (L,U)),$ is given by

    \begin{align} \label{med_demed_ratio_ss}
        S_c(\boldsymbol{u},\boldsymbol{v}) = \max_{k\in\{0,1,\dots\}} \;\exp\left(-k\beta\right)\;\min\left\{U-L, \widetilde{A}^{(k)}_r(\boldsymbol{u}, \boldsymbol{v}) \right\},
    \end{align}

    \noindent
    where
    \begin{align} \label{med_demed_ratio_ls_at_k}
        \widetilde{A}^{(k)}_r(\boldsymbol{u}, \boldsymbol{v}) = \sup_{(\tilde{u}, \tilde{v}), (\tilde{u}^\prime, \tilde{v}^\prime) \in B(k, \bu,\bv), \; t\in\{0,\dots,k+1\}} \left\{
        \begin{array}{lr} \left(\frac{\bu-\tilde{u} }{\bv-\tilde{v}}\right)_{(m+t)}- \left(\frac{\bu - \tilde{u}^\prime}{\bv-\tilde{v}^\prime}\right)_{(m+t-k-1)} & m - k \geq 2 \\
            \infty & m - k < 2
            \end{array}
        \right.,
    \end{align}

    \noindent
    and $B(k, \bu, \bv) = [\bu_{(m-k-1)}, \bu_{(m+k+1)} ) \times [\bv_{(m-k-1)}, \bv_{(m+k+1)}).$
\end{lemma}
\begin{proof}
    We will show that $\widetilde{A}_r^{(k)}(\bu, \bv)$ is an upper bound on the local sensitivity of $r(\bu,\bv)$ at distance $k,$ \textit{i.e.,} $A_r^{(k)}(\bu, \bv).$ Let $C(k, \bu, \bv) = \{\bu^\prime, \bv^\prime \in\rr^n :  \hamming(\{(\bu^\prime[i], \bv^\prime[i])\}_{i=1}^n, \{(\bu[i], \bv[i])\}_{i=1}^n) \leq 2 k\}$ and $E(k, \bu) = \{\bu^\prime \in\rr^n :  \hamming( \bu^\prime, \bu) \leq 2 k\}.$ By the definition of $A_r^{(k)}(\bu, \bv)$
    
    \begin{align*}
        & A_r^{(k)}(\bu, \bv) = \sup_{(\bu^\prime, \bv^\prime) \in C(k, \bu, \bv), \; (\tilde{\bu}^\prime, \tilde{\bv}^\prime) \in C(1, \bu^\prime, \bv^\prime)} \lvert r(\bu^\prime, \bv^\prime) - r(\tilde{\bu}^\prime, \tilde{\bv}^\prime) \rvert \\
        & = \sup_{(\bu^\prime, \bv^\prime) \in C(k, \bu, \bv), \; (\tilde{\bu}^\prime, \tilde{\bv}^\prime) \in C(1, \bu^\prime, \bv^\prime)} \left\lvert \textup{med}\left(\frac{\bu^\prime- \med(\bu^\prime) }{\bv^\prime- \med(\bv^\prime) }\right)-\textup{med}\left(\frac{\tilde{\bu}^\prime - \med(\tilde{\bu}^\prime)}{\tilde{\bv}^\prime-\med(\tilde{\bv}^\prime)}\right) \right\rvert \\
        & \leq \widehat{A}_r^{(k)}(\bu, \bv) = \sup_{\bu^\prime, \bv^\prime, \tilde{\bu}^\prime, \tilde{\bv}^\prime, \check{\bu}^\prime, \check{\bv}^\prime, \hat{\bu}^\prime, \hat{\bv}^\prime \in \rr^n} \left\lvert \textup{med}\left(\frac{\bu^\prime - \med(\check{\bu}^\prime) }{\bv^\prime - \med(\check{\bv}^\prime) }\right) - \textup{med}\left(\frac{\tilde{\bu}^\prime - \med(\hat{\bu}^\prime)}{\tilde{\bv}^\prime-\med(\hat{\bv}^\prime)}\right) \right\rvert \textup{  such that:} \\
        & i) \; (\bu^\prime, \bv^\prime) \in C(k, \bu, \bv) \\
        & ii) \;(\tilde{\bu}^\prime, \tilde{\bv}^\prime) \in C(1, \bu^\prime, \bv^\prime) \\
        & iii) \; \check{\bu}^\prime, \hat{\bu}^\prime \in E(k+1, \bu) \\ 
        & iv) \; \check{\bv}^\prime, \hat{\bv}^\prime \in E(k+1, \bv), \\
    \end{align*}

    \noindent
    where the inequality above follows from the fact that $ \widehat{A}_r^{(k)}(\bu, \bv)$ is defined by a relaxation of the optimization problem that defines $A_r^{(k)}(\bu, \bv).$ Specifically, the optimization problem defining $A_r^{(k)}(\bu, \bv)$ can be derived by starting from $\widehat{A}_r^{(k)}(\bu, \bv)$ and adding the constraints $\hat{\bu}^\prime = \tilde{\bu}^\prime,$ $\hat{\bv}^\prime = \tilde{\bv}^\prime,$ $\check{\bu}^\prime = \bu^\prime,$ and $\check{\bv}^\prime = \bv^\prime.$ 

    Recall the closed form $\beta-$smooth sensitivity of the median is given in (\ref{medianSmoothSense}). This definition, and the fact that, for any $\bz^\prime \in E(k+1, \bz),$ we have $\textup{med}(\bz^\prime) \in [\bz_{(m-k-1)}, \bz_{(m+k+1)}],$ implies $\widehat{A}_r^{(k)}(\bu, \bv) = \widetilde{A}^{(k)}_r(\boldsymbol{u}, \boldsymbol{v}).$
\end{proof}

The function $\widetilde{A}^{(k)}_r(\boldsymbol{u}, \boldsymbol{v})$ can alternatively be defined using a concept from computational geometry literature; some additional notation will be helpful to describe this connection. We will let $\Gamma$ denote a set of curves in $\rr^3,$ with curve $i\in\{1,\dots,n\}$ defined as $ \gamma(i, \tilde{u}, \tilde{v}) = (\bu[i]-\tilde{u} )/(\bv[i]-\tilde{v}).$ Let $\mathcal{A}(\Gamma)$ denote the arrangement of $\Gamma,$ which is defined as the collection of graphs of functions in $\Gamma.$ More detail on arrangements and the other concepts from the computational geometry that we briefly describe here can be found in \cite{edelsbrunner1987algorithms}. We will define the level of a point $p\in\rr^3$ as the number of curves in $\Gamma$ that either pass below or through $p,$ and define the $k-$level of $\mathcal{A}(\Gamma),$ or $L(k, \Gamma),$ as the points on one or more curves in $\Gamma$ with level $k.$ Also, a $\zeta-$approximate $k-$level of an arrangement of curves is defined as a curve that is within the $k-\zeta$ and $k+\zeta$ levels of an arrangement, which we will denote by $L_\zeta(k, \Gamma).$ Using this notation, we have, 

\begin{align} \label{tildeAReform}
    &\widetilde{A}^{(k)}_r(\bu, \bv) = \sup_{j,t, \tilde{u}, z,  \tilde{u}^\prime, z^\prime} z - z^\prime \textup{ such that:} \\
    & i)\; \; t\in\{0,\dots,k+1\} \nonumber\\
    & ii) \; (\tilde{u}, \tilde{v}, z) \in L(m + t, \Gamma)\nonumber\\
    & iii) \;(\tilde{u}^\prime, \tilde{v}^\prime, z^\prime) \in L(m + t-k-1, \Gamma)\nonumber\\
    & iv)\; \tilde{u},\tilde{u}^\prime \in  [\bu_{(m-k-1)}, \bu_{(m+k+1)}) \nonumber\\
    & v)\; \tilde{v},\tilde{v}^\prime \in  [\bv_{(m-k-1)}, \bv_{(m+k+1)}), \nonumber
\end{align}

Unfortunately, we are not aware of a computationally efficient method of computing approximations of all levels of $\{L(k, \Gamma)\}_k$ for this use case, in which the curves in $\Gamma$ are nonlinear and defined in $\rr^3.$\footnote{\cite{agarwal1990partitioning} provides a related method, which provides approximations to all levels when $\Gamma$ is a collection of lines in $\rr^2.$ This can used within a method to compute a $\beta-$smooth upper bound on the local sensitivity using some of the same techniques in the approach described below. For example, upper bounds on the order statistics of $\{\gamma(i, \tilde{u}, \tilde{v})\}_{i=1}^n$ can be computed for $\tilde{v}$ in each partition element $q_j = (j c, (j+1) c],$ where $c>0,$ by constructing lines that are upper bounds of $\tilde{u} \mapsto \gamma(i,\tilde{u},\tilde{v})$ over all $\tilde{v}\in q_j.$ However, there are some issues with using this approach in practice, such as the large constants in the asymptotic time complexity of the approach described by \cite{agarwal1990partitioning}.} Algorithm \ref{alg:ssMedsweep} provides an alternative approach, which is also used in the simulations provided in the next section. This approach uses the choice parameters $c_u,c_v>0$ to define the partitions of $\rr^1$ given by $q_u=\{\cdots, [c_u (j-1), c_u j), [c_u j, c_u (j + 1)), \cdots\}_{j=-\infty}^\infty$ and $ q_v=\{\cdots, [c_v (j-1), c_v j), [c_v j, c_v (j + 1)), \cdots\}_{j=-\infty}^\infty.$ After defining a partition of $[\bu_{(m-k)}, \bu_{(m+k)}] \times [\bv_{(m-k)}, \bv_{(m+k)}] \subset \rr^2$ as 

\begin{align}
    B(k, \bu, \bv) = \{q \in q_u : q \cap [\bu_{(m-k)}, \bu_{(m+k)}] \neq \emptyset \} \times \{q \in q_v : q \cap [\bv_{(m-k)}, \bv_{(m+k)}] \neq \emptyset \},
\end{align}

\noindent
we bound each of the order statistics of $\{\gamma(i,\tilde{u}, \tilde{v})\}_i,$ for each $(\tilde{u}, \tilde{v})$ in the partition element $ c \in B(k,\bu,\bv),$ between the corresponding order statistic of $ \{ \inf_{(\tilde{u}, \tilde{v}) \in c} \gamma(i,\tilde{u}, \tilde{v}) \}_i$ and $\{\sup_{(\tilde{u}, \tilde{v}) \in c} \gamma(i,\tilde{u}, \tilde{v})\}_i.$ However, if $c_u,c_v$ are defined to be $O(1/n^{\alpha })$ for a choice parameter $\alpha > 0$ and $\{(\boldsymbol{u}[i], \boldsymbol{v}[i])\}_i$ consists of independent and identically distributed observations, this algorithm's output still converges to zero and has a time complexity of $O_p(n^{1 + 2\alpha } \log(n)),$ as described in the next Lemma. In the simulations in the next section we defined $c_u$ and $c_v$ as $8/n^{3/4},$ which results in a time complexity of $O_p(n^{2.5}\log(n)).$

\begin{lemma} \label{lemma:ssMedsweep}
    The output of Algorithm \ref{alg:ssMedsweep} is a $\beta-$smooth upper bound on the local sensitivity of $\textup{clip}(r(\boldsymbol{u},\boldsymbol{v}), (L,U)).$ When $\{(\boldsymbol{u}[i], \boldsymbol{v}[i])\}_i$ consists of independent and identically distributed observations, the algorithm has a time complexity of $O_p(n^{1 + 2 \alpha } \log(n)).$

    Also, if $\{(\bu[i], \bv[i])\}_{i=1}^n$ is composed of independent observations from an absolutely continuous distribution with a convex support, then the output of Algorithm \ref{alg:ssMedsweep} converges in probability to zero.
\end{lemma}
\begin{proof}
    The only difference between the bound provided by the output of Algorithm \ref{alg:ssMedsweep} and the one described in Lemma \ref{lem:ssMedsweepPreliminary} is that the former does not use the exact values for the order statistics of $\{\gamma(i,\tilde{u}, \tilde{v})\}_i,$ as described in the paragraph preceding this Lemma. Since using inexact bounds on these order statistics can only increase the upper bound on the local sensitivity at distance $k$ relative to that of (\ref{med_demed_ratio_ls_at_k}), Algorithm \ref{alg:ssMedsweep} provides a $\beta-$smooth upper bound on the local sensitivity.

    Since the area of each cell $c \in B(k,\bu,\bv)$ converges to zero, the upper and lower bounds on the order statistics of $\{\gamma(i,\tilde{u}, \tilde{v})\}_i$ converge to the corresponding true order statistic of $\{\gamma(i,\tilde{u}, \tilde{v})\}_i.$ The conditions of the lemma imply, for any fixed $k\in\zz,$ $\bu_{(m+k)} - \bu_{(m-k)},$ $\bv_{(m+k)} - \bv_{(m-k)},$ and $\bz_{(m+k)} - \bz_{(m-k)},$ where $\bz[i]=(\bu[i]-\tilde{u})/(\bv[i]-\tilde{v}),$ are each $o_p(1),$ so the output of Algorithm \ref{alg:ssMedsweep} is also $o_p(1).$

    The time complexity follows from the algorithm evaluating $O_p(n^{2 \alpha})$ sort operations on vectors of length $n.$ 
\end{proof}

\begin{algorithm} \label{alg:ssMedsweep}
\DontPrintSemicolon
\SetKwInOut{Output}{return}
\SetKwInOut{Input}{input}
$n\gets \textup{length}(\bu)$ \;
$m\gets \lfloor  n/2 \rfloor $ \;
lower\_bounds\_dict, upper\_bounds\_dict $\gets \{\}, \{\}$ \;
smooth\_sensitivity $ \gets -\infty$ \;
\For{$k \in \{0,\dots, m-1\}$}{
    $B(k, \bu,\bv) \gets  \{q \in q_u : q \cap [\bu_{(m-k-1)}, \bu_{(m+k+1)}] \neq \emptyset \} \times \{q \in q_v : q \cap [\bv_{(m-k-1)}, \bv_{(m+k+1)}] \neq \emptyset \} $\;
    \For{$c \in B(k,\bu, \bv)$}{
        \If{$c \notin  \textup{keys(lower\_bounds\_dict)}$}{
            lower\_bounds\_dict$[c],$ upper\_bounds\_dict$[c] \gets \emptyset, \; \emptyset$\;
            \For{$i\in\{1,\dots,n\}$}{
                lower\_bounds\_dict$[c] \gets $ lower\_bounds\_dict$[c] \cup \inf_{(\tilde{u}, \tilde{v}) \in c} (\bu[i] - \tilde{u}) / (\bv[i] - \tilde{v})$  \;
                upper\_bounds\_dict$[c] \gets $ upper\_bounds\_dict$[c] \cup \sup_{(\tilde{u}, \tilde{v}) \in c} (\bu[i] - \tilde{u}) / (\bv[i] - \tilde{v})$  \;
            }
        }
    }
    local\_sens $ \gets -\infty$\;
    \For{$t \in \{0,\dots, k+1\}$}{
        \For{$c \in B(k,\bu, \bv)$}{
            local\_sens$\gets\max(\textup{local\_sens}, \;\textup{upper\_bounds\_dict}[c]_{(m + t)} - \textup{lower\_bounds\_dict}[c]_{(m+t-k-1)})$\;
        }
    }
    smooth\_sensitivity $\gets \max( \textup{smooth\_sensitivity}, \exp(-\beta k) \min(\textup{local\_sens}, U-L))$ \;
    \If{$\textup{local\_sens} \geq U -L $}{
        \textbf{Break} \;
    }
}
\Output{$\textup{smooth\_sensitivity}$}
\Input{$(\bu, \bv, \beta,  L, U, q_u, q_v, \Theta)$}
\caption{Medsweep\_Smooth\_Sensitivity}
\end{algorithm}

\begin{algorithm} \label{alg:dpmedsweep}
\DontPrintSemicolon
\SetKwInOut{Output}{return}
\SetKwInOut{Input}{input}
\tcp{Find the total number of required primitive mechanism calls:}
\eIf{$d = 2$}{$ \textup{n\_mechs} \gets (d - 1) \textup{ max\_iter } + 1$}{$\textup{n\_mechs} \gets (d - 1)  (d - 2) / 2 + (d - 1) \textup{ max\_iter } + 1$}
\tcp{Each primitive mechanism will be $(\epsilon/\textup{n\_mechs}, \delta/\textup{n\_mechs})-$DP; Theorem \ref{laplace_smooth_sense} provides $\alpha,\beta:$}
$\tilde{\epsilon}, \; \tilde{\delta} \gets \epsilon / \textup{n\_mechs}, \;  \delta / \textup{n\_mechs}$\;
$\alpha, \; \beta \gets \tilde{\epsilon} / 2, \;  W_{-1}(\tilde{\delta} \log(\tilde{\delta}) \exp(\tilde{\epsilon} / 2)) - \log(\tilde{\delta}) - \tilde{\epsilon} / 2$ \;
\If{ $d \geq 3$}{
    $\btheta_x \gets \boldsymbol{0}_{d\times d}$ \;
    \For{$k\in\{2, \dots d -1\}$}{
        \tcp{Parse (or ``sweep") $X[:,i]$ out of $X[:, k]$ using a univariate approximation to the deepest regression:}
        \For{ $i\in \{1, \dots, k-1\}$}{
            $S_{\textup{MDR}}\gets \textup{smooth\_upper\_bound\_on\_MDR\_local\_sensitivity}(X[:, k], X[:, i], \beta, U, L)$ \;
            $\btheta_x[i, k] \gets \textup{med}\left(\frac{X[:, k] - \textup{med}(X[:, k])}{X[:, i] - \textup{med}(X[:, i])}\right) + $ Laplace$(\mu = 0, b=S_{\textup{MDR}}/\alpha)$ \;
        }
        \For{$i\in \{1, \dots, k\}$}{
            $X[:, k] \gets X[:, k] - X[:, i] \btheta_x[i, k]$\;
        }
    }
}
$\btheta_y, \; \hbtheta  \gets \boldsymbol{0}_{d, \textup{max\_iter}}, \; \boldsymbol{0}_d$ \;
\For{$j\in\{1, \dots, \textup{max\_iter}\}$}{
    \For{$k\in \{1, \dots, d - 1\}$}{
        \tcp{Parse (or ``sweep") $X[:,k]$ out of $\by$ using a univariate approximation to the deepest regression:}
        $S_{\textup{MDR}}\gets \textup{smooth\_upper\_bound\_on\_MDR\_local\_sensitivity}(X[:, k], \by, \beta, U, L)$ \;
        $\btheta_y[k, j] \gets \textup{med}\left(\frac{\by - \textup{med}(\by)}{X[:, k] - \textup{med}(X[:, k])}\right) + $Laplace$(\mu = 0, b=S_{\textup{MDR}}/\alpha)$ \;
        $\hbtheta[k + 1] \gets \hbtheta[k + 1]  +  \btheta_y[k, j]$ \;
        $\by \gets \by - \btheta_y[k, j] X[:, k]$ \;
    } 
    \If{$ \lvert \btheta_y[k, j] \rvert \leq \textup{tol}$ \textup{ for all } $k\in \{1, \dots, d - 1\}$}{
        Break \;
    }
}
\tcp{An algorithm for computing the $\beta-$smooth sensitivity of the median is provided by \cite{nissim2007smooth}; see also equation (\ref{medianSmoothSense}):}
$S_{\textup{med}} \gets \textup{smooth\_upper\_bound\_on\_med\_local\_sensitivity}(\by, \beta, U, L)$ \;
$\hbtheta[0] \gets \textup{med}(\by) + $Laplace$(\boldsymbol{\mu} = 0, b=S_{\textup{med}}/\alpha)$ \;
\For{$k \in \{d-1, d-2, \dots, 2\}, \; i \in \{k-1, k-2, \dots, 1\}$}{
    $\hbtheta[k + 1] \gets \hbtheta[k + 1] - \hbtheta[i + 1] \btheta_x[i, k]$ \;
}
\Output{$\hbtheta$} 
\Input{$(X \in \rr^{n \times d-1}, \by \in\rr^n, U,L, \epsilon, \delta,\textup{tol}, \textup{max\_iter})$} 
\caption{DPMedsweep}
\end{algorithm}

Algorithm \ref{alg:dpmedsweep} provides pseudocode for the approximate DP variant of Medsweep; the following theorem provides the privacy guarantee of this mechanism.

\begin{lemma}
    The mechanism DPMedsweep, as described in Algorithm \ref{alg:dpmedsweep}, satisfies $(\epsilon, \delta)-$DP.
\end{lemma}
\begin{proof}
    This follows from the fact that the primitive mechanisms used within DPMedsweep are $(\epsilon/\textup{n\_mechs},$ $\delta/\textup{n\_mechs})-$DP and that the primitive mechanisms are called $\textup{n\_mechs}$ times.
\end{proof}

The primary advantage of the approach used in Algorithm \ref{alg:dpmedsweep} is that the noise added to the output has exponential tails. However, there are several alternatives to this approach that are worth mentioning as well, particularly in cases in which noise with exponential tails is not required. First, as described above, one can alternatively formulate a DP variant of Medsweep by replacing each median evaluation with a DP median mechanism. This can be done using a variety of DP median mechanisms, including those that use smooth sensitivity, as described by \cite{nissim2007smooth}, or an approach based on the exponential mechanism provided by \cite{mcsherry2007mechanism}. Second, since the mechanism below uses several calls to primitive mechanisms, especially in the case of multivariate regressions, $\rho-$concentrated differential privacy ($\rho-$CDP) variants are also worth mentioning because this privacy guarantee provides tight sequential composition. This privacy guarantee was first described by \cite{bun2016concentrated}; also, \cite{bun2019average} describe how $\rho-$CDP mechanisms can be formulated using $\beta-$smooth upper bounds on local sensitivity. 

\section{Simulations} \label{sec:simulations}

This section provides simulations to compare the accuracy of DPDeepestReg$(\cdot)$ and DPMedsweep$(\cdot)$ with three other DP regression methods for three different data generating processes (DGPs). Specifically, for each $F\in \{N(\mu=0, \sigma^2=1), \textup{Laplace}(\mu=0, b=1), t(\mu=0, \nu=3)\},$ each $i\in\{1,\dots,n\},$ and each $n\in \{50,100,200,300,400,500\},$ we define the dataset $\ndat$ as $\{(x_i, y_i)\}_{i=1}^n,$ where 

\begin{align} \label{eqn:DGPsDescriptions}
    y_i = 1 + 2 x_i + v_i \textup{ and } x_i, v_i \sim F.
\end{align}

\noindent
Note that, since each $F\in \{N(\mu=0, \sigma^2=1), \textup{Laplace}(\mu=0, b=1), t(\mu=0, \nu=3)\}$ is symmetric, the mean of $y$ conditional on $x$ is equal to the median of $y$ conditional on $x,$ so the DP conditional mean and median estimators described in the next subsection estimate the same population value. These simulations were performed for each $\epsilon\in \{4,12\},$ and, for each $(\epsilon,\delta)-$DP mechanism, we set $\delta$ to $10^{-6}.$

For each simulation iteration, we estimate $\btheta$ using DPDeepestReg$(\cdot),$ DPMedsweep$(\cdot),$ and the three classic DP estimators described in the next subsection. As described previously, our aim with the implementation choices described in the next section was to err on the side of understating the relative accuracy of the two proposed estimators, particularly for choices related to both datapoint bounds (in the case of two of the three classic DP methods) and estimator bounds (in the case of the proposed approaches), which are described in more detail in the final paragraph of the next subsection.

\subsection{Tested DP Estimators}

The implementation for DPMedsweep$(\cdot)$ used the pseudocode given in Algorithm \ref{alg:dpmedsweep}. The implementation of DPDeepestReg$(\cdot)$ used the pseudocode given in Algorithm \ref{alg:DPDeepestReg}; to compute the deepest regression estimator, we simply computed the depth of all candidate regressions that pass through each pair of data points, using the function rdepth2$(\cdot)$ from \cite{segaert2017mrfdepth} that was rewritten in python.

All three of the existing DP regression estimators are based on python code from \cite{alabi2020differentially}, which is available at \url{https://github.com/anonymous-conf/dplr}. The first method is DPMedTS\_exp$(\cdot),$ which is a differentially private Theil-Sen estimator \citep{theil1950rank,sen1968estimates}. This $\epsilon-$DP mechanism is based on the mechanism provided by \cite{dwork2009differential}, but also incorporates subsequent results in the DP literature that improve accuracy of $\epsilon-$DP median mechanisms.

Second, NoisyStats$(\cdot)$ is an $\epsilon-$DP OLS estimator. This mechanism works by passing sufficient statistics for the OLS estimator through a Laplace mechanism. One disadvantage of this approach is that, whenever the $\epsilon-$DP approximation of var$(x) = \sum_i (x_i- \overline{x})^2/n$ is negative, this mechanism uses $2/3$ of the total privacy-loss budget without actually providing an estimate. For the purposes of calculating the error metrics that are reported below, these simulation iterations were simply removed before computing these error metrics; the count of the simulation iterations for which no output was returned from this mechanism is provided in Table \ref{tab:NoisyStatsFailures}. Note that this occurred most often for smaller sample sizes, the $\epsilon=4$ simulations, and the DGPs with thicker tails that we consider. 

\begin{table} 
\centering
\begin{tabular}{ c c | c c c c c c}
$\epsilon$ & DGP & $n=$ 50 & $n=$ 100 & $n=$ 200 & $n=$ 300 & $n=$ 400 & $n=$ 500\\ \hline
4 & $x_i, v_i \sim N(\mu=0, \sigma^2=1)$ & 41 & 36 & 19 & 14 & 6 & 3 \\
4 & $x_i, v_i \sim \textup{Laplace}(\mu=0, b=1)$ & 39 & 31 & 21 & 17 & 11 & 9 \\
4 & $x_i, v_i \sim t(\mu=0, \nu=3)$  & 47 & 30 & 29 & 25 & 18 & 12 \\
12 & $x_i, v_i \sim N(\mu=0, \sigma^2=1)$ & 27 & 6 & 2 & 1 & 0 & 0 \\
12 & $x_i, v_i \sim \textup{Laplace}(\mu=0, b=1)$ & 30 & 11 & 4 & 0 & 1 & 0 \\
12 & $x_i, v_i \sim t(\mu=0, \nu=3)$ & 30 & 21 & 13 & 2 & 1 & 1 \\
\end{tabular}
\caption{The number of simulation iterations, out of a total of 100 iterations, for which NoisyStats$(\cdot)$ failed to produce an estimate, for each combination of $\epsilon,$ DGP, and sample size.} \label{tab:NoisyStatsFailures}
\end{table}

Third, DPGradDesc$(\cdot)$ is an $(\epsilon,\delta)-$DP mechanism that provides a DP OLS estimate using a DP gradient descent method. Like the methods proposed here, we used $\delta=1/10^6$ for this mechanism. We initialized these gradient descent steps at the origin and chose an iteration limit of 30. and we also set the iteration limit to 30. Another choice parameter was set so that the gradient was clipped to the unit cube $[-10, 10]^2.$ These choices appeared to increase accuracy in our simulations relative to their default values, but it is possible that other choices would provide a further improvement. 

These three existing mechanisms were designed to return estimates of the dependent variable at two values of the independent variable, so we also modified DPDeepestReg$(\cdot)$ and DPMedsweep$(\cdot)$ to output estimates of the dependent variable at two values of the independent variable. Specifically, in each simulation iteration, we saved the errors of the dependent variable estimate at the 0.25 and 0.75 quantiles of the random variable $x_i$ for each of the five DP estimators considered, and, after 100 simulation iterations, we computed the mean squared error and median absolute error for each of these estimators, which are provided in Figures \ref{fig:sims_MSE} and \ref{fig:sims_MAE} respectively. These results are also discussed in more detail in the next subsection.

The first two existing DP mechanisms described above require bounds on $\{(x_i,y_i)\}_i.$ To define these bounds, we use an approach that was inspired by the DP mechanism provided by \cite{chen2016differentially}, which, as described above, outputs an approximation of the bounding box given by $[-c, c]^2,$ where $c\in \rr_+$ is as small as possible such that at least a proportion $\psi\in (0,1)$ of the data points are in $[-c, c]^2.$ Setting this choice parameter in practice requires balancing two tradeoffs. First, higher values typically result in larger datapoint bounds, which leads to a higher global sensitivity and thus a higher variance of the noise in these mechanisms. Second, since all datapoints outside of these bounds are ignored in these mechanisms, lower values tend to decrease the statistical efficiency and increase the bias of the DP estimator. Rather than use this DP mechanism, we simply set the bounds on $\{(x_i,y_i)\}_i$ using this definition of $[-c, c]^2$ directly, with $\psi=0.98$ in the simulation results presented in this section. Since this choice parameter did have a significant impact on the performance of both of these classical DP estimators, Appendix \ref{appendixC} provides the simulation results for both $\psi=0.95$ and $ \psi=1.$ The approach taken here is intended to simulate an ideal realization of the output of the preliminary DP mechanism provided by \cite{chen2016differentially}. Note that we do not account for the privacy-loss budget that would be required to compute these data-dependent bounds in an actual use case in order to err on the side of overstating the accuracy of these existing methods for a given total privacy-loss budget. In contrast, for the two proposed DP estimators, we use the non-data-dependent bounds on $\btheta$ of $\Theta=[-50, 50]^2.$ The impact of this choice is described in the final paragraph of the next subsection. As described in the paragraph preceding Lemma \ref{lemma:ssMedsweep} in more detail, we defined $c_u,c_v$ as $8/n^{0.75},$ which results in DPMedsweep$(\cdot)$ having a time complexity of $O(n^{2.5}\log(n)).$ 

\begin{figure}[ht]
    \centering
    \includegraphics[scale=.58]{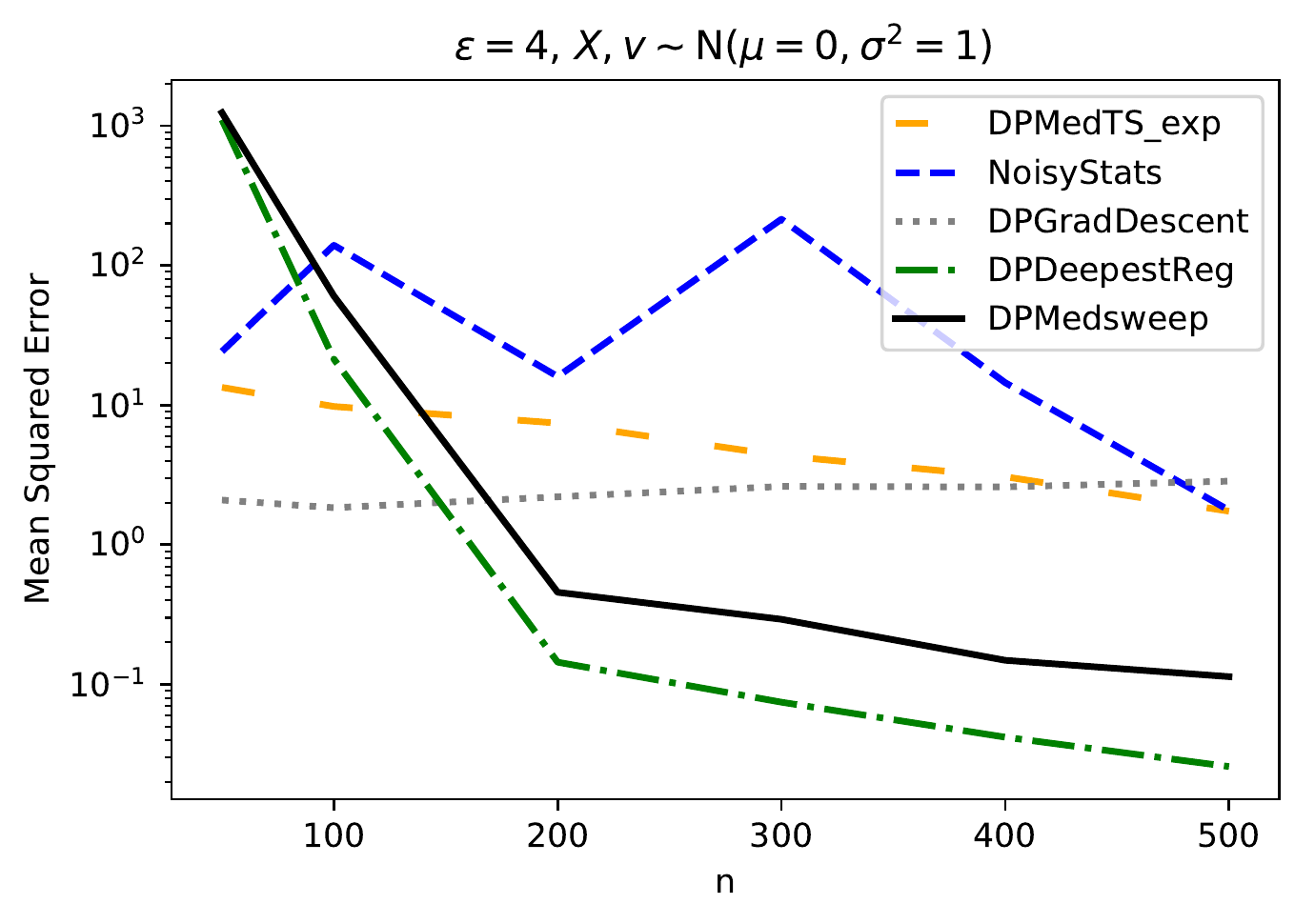} \includegraphics[scale=.58]{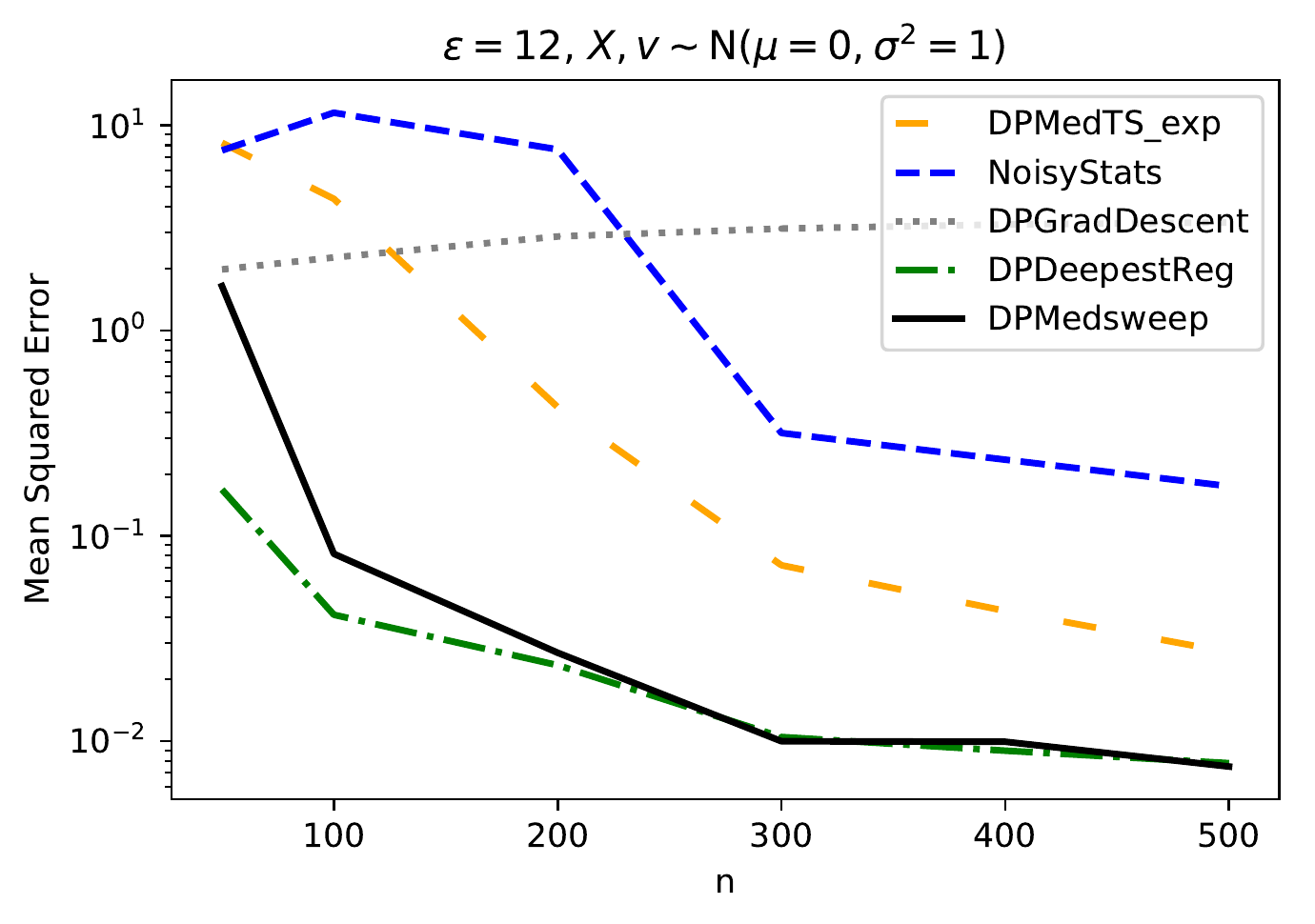} \\
    \includegraphics[scale=.58]{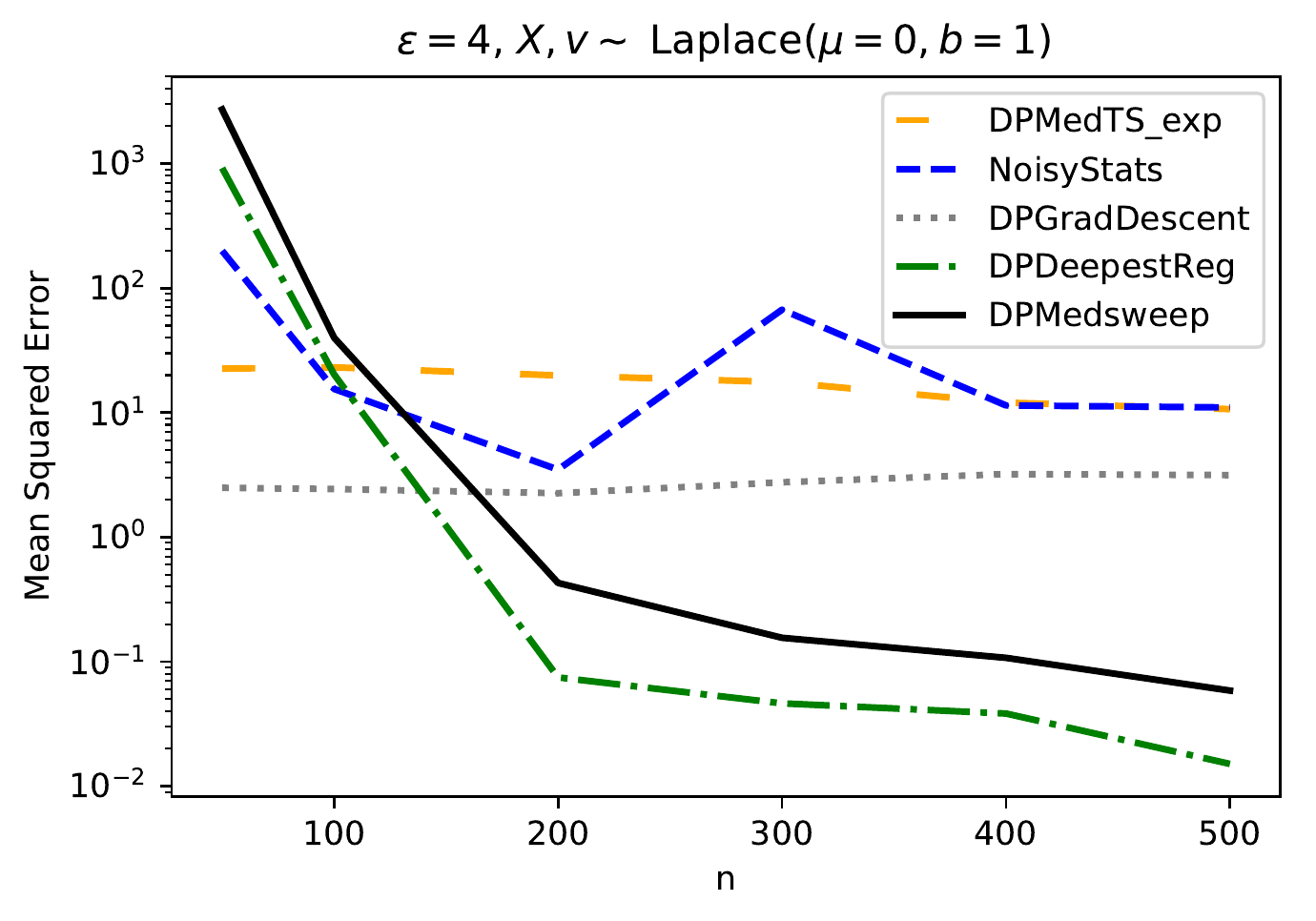} \includegraphics[scale=.58]{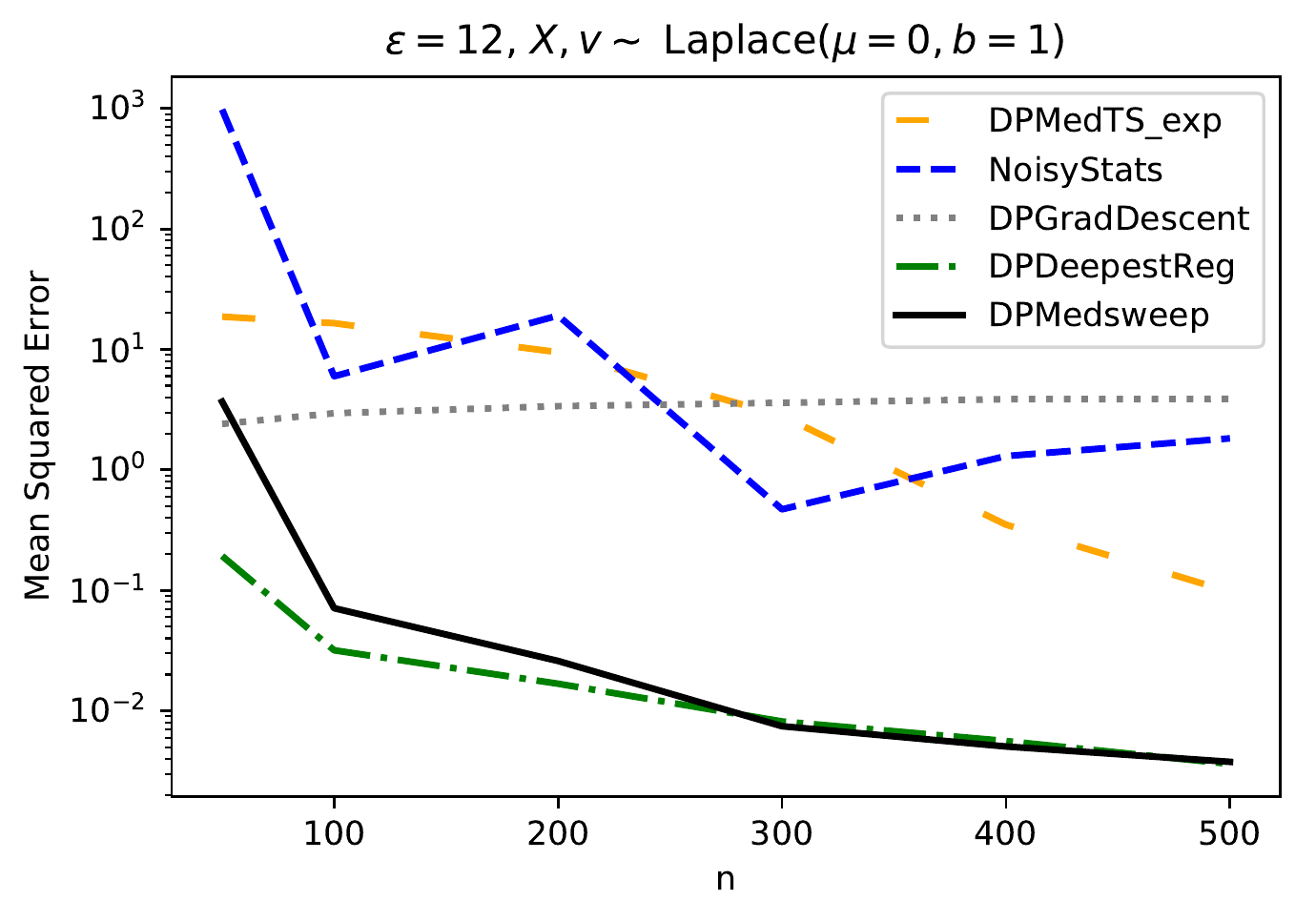} \\
    \includegraphics[scale=.58]{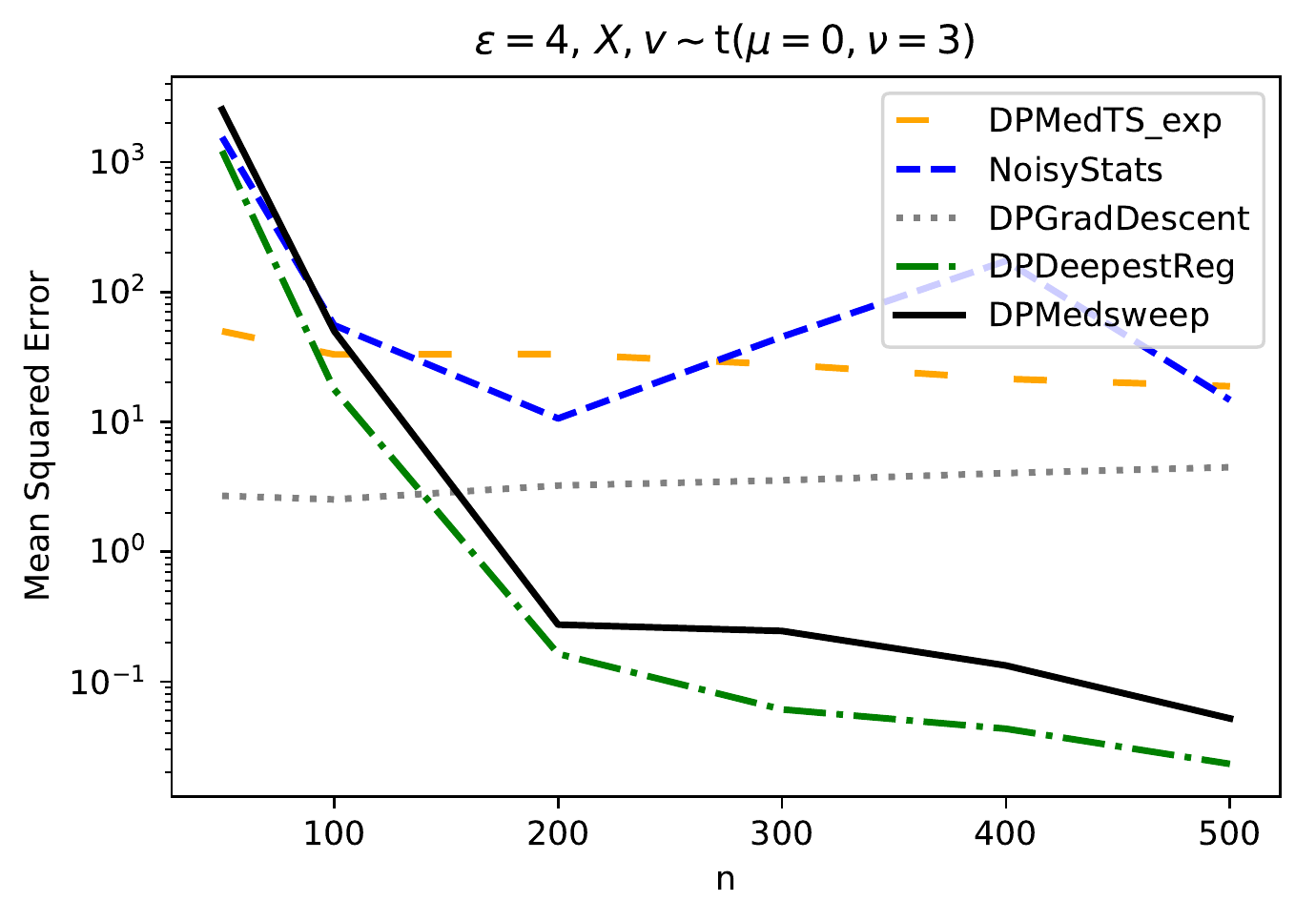} \includegraphics[scale=.58]{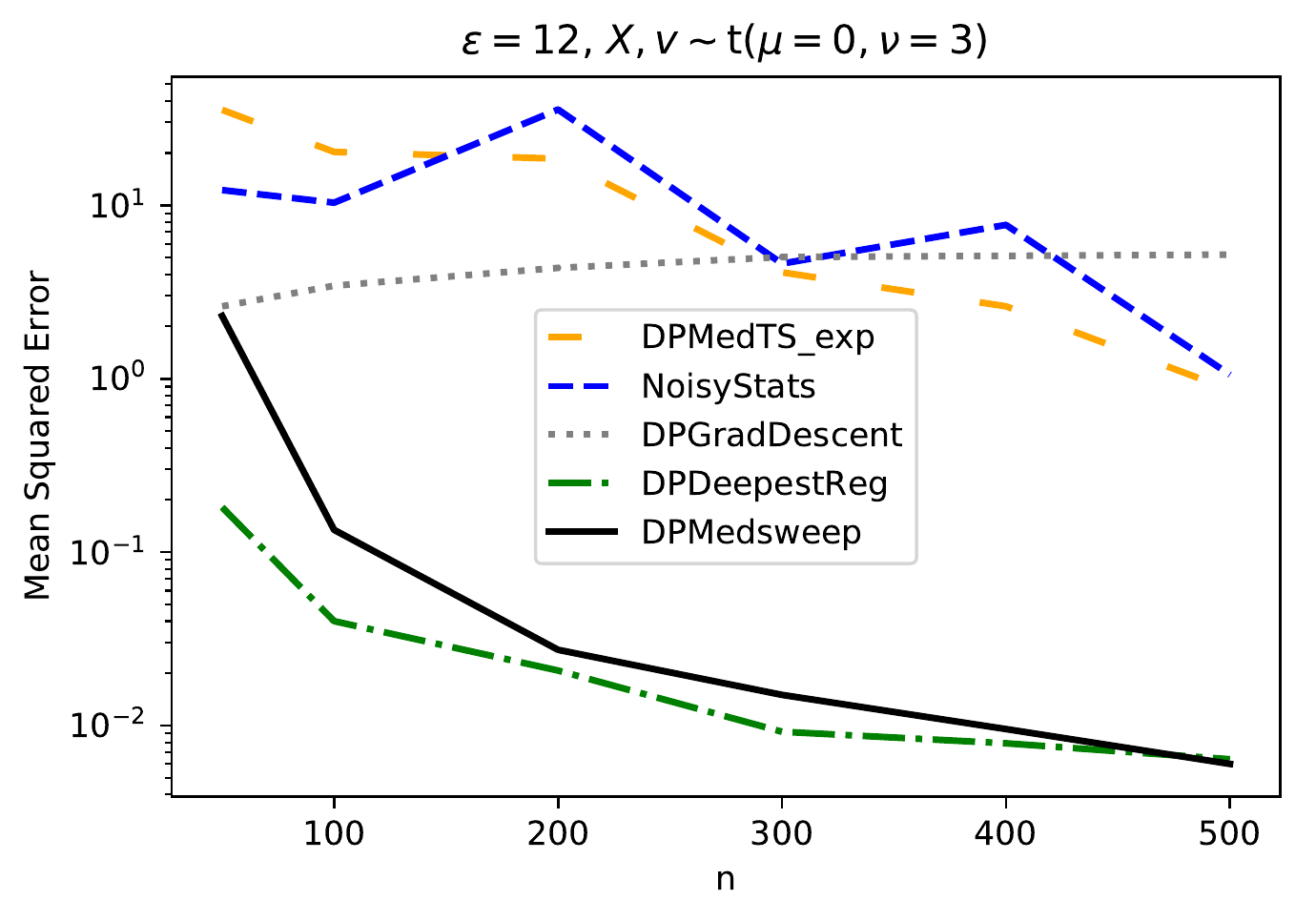}
    \caption{The mean squared error at the 25\% and 75\% percentile of the random variable $x_i$ for each DP estimator, for each DGP $x_i, v_i \sim N(\mu=0, \sigma^2=1)$ (top), $x_i, v_i \sim \textup{Laplace}(\mu=0, b=1)$ (middle), $x_i, v_i \sim t(\mu=0, \nu=3)$ (bottom), and for $\epsilon=4$ (left) and $\epsilon=12$ (right). Both NoisyStats$(\cdot)$ and DPMedTS\_exp$(\cdot)$ satisfy $\epsilon-$DP; DPGradDesc$(\cdot),$ DPMedsweep$(\cdot),$ and DPDeepestReg$(\cdot)$ satisfy $(\epsilon, 10^{-6})-$DP. }
    \label{fig:sims_MSE}
\end{figure}

\begin{figure}[ht]
    \centering
    \includegraphics[scale=.58]{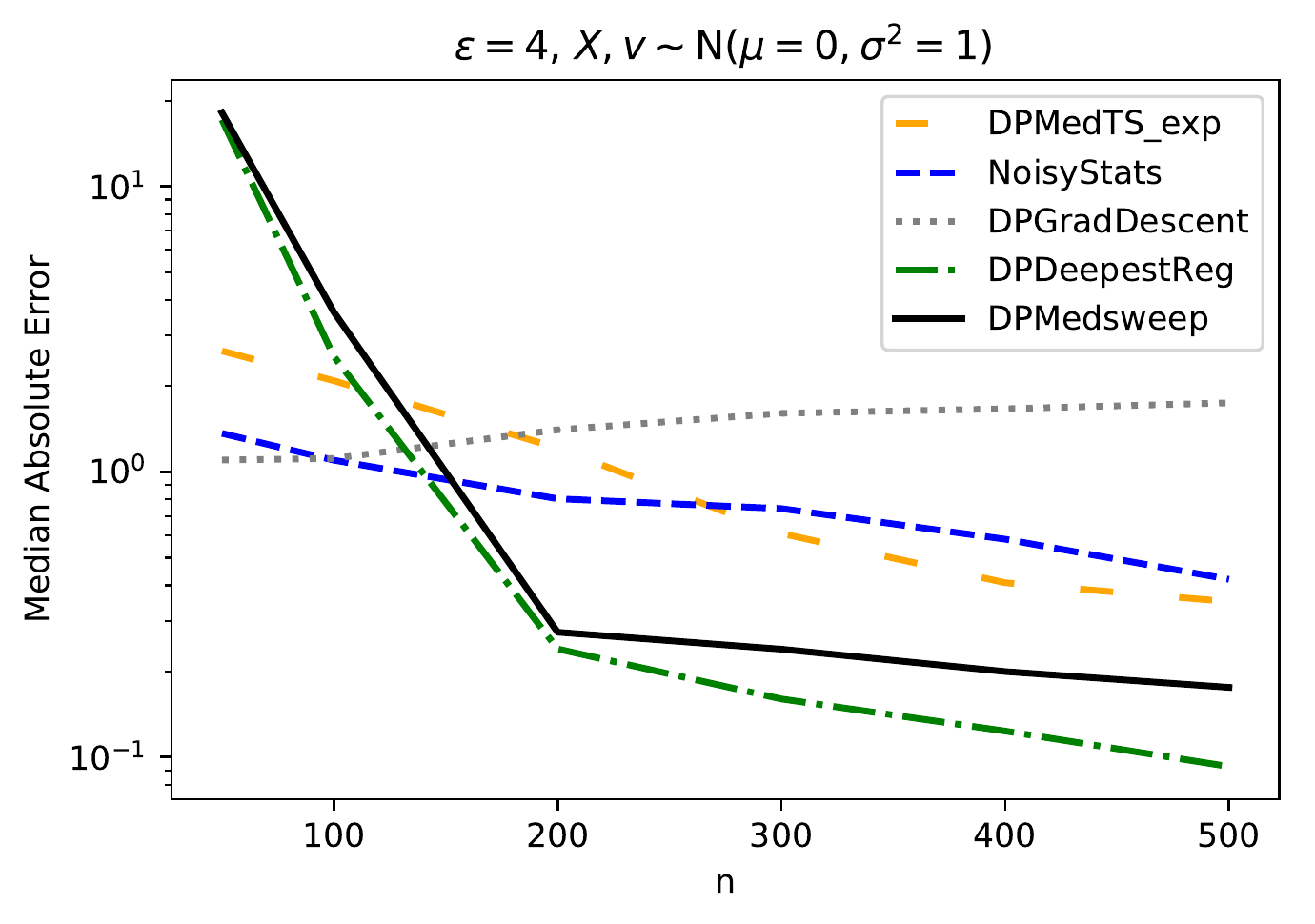} \includegraphics[scale=.58]{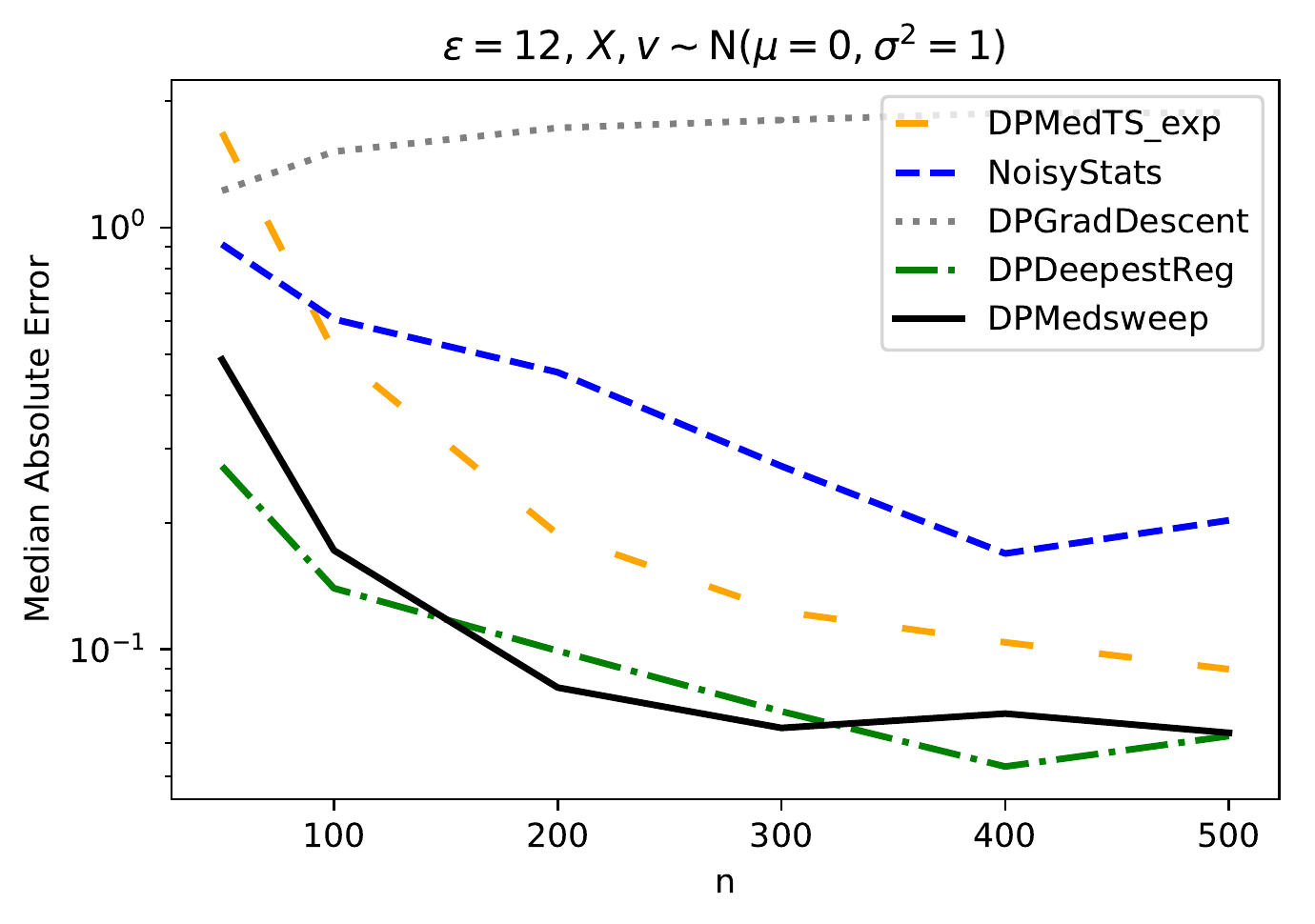} \\
    \includegraphics[scale=.58]{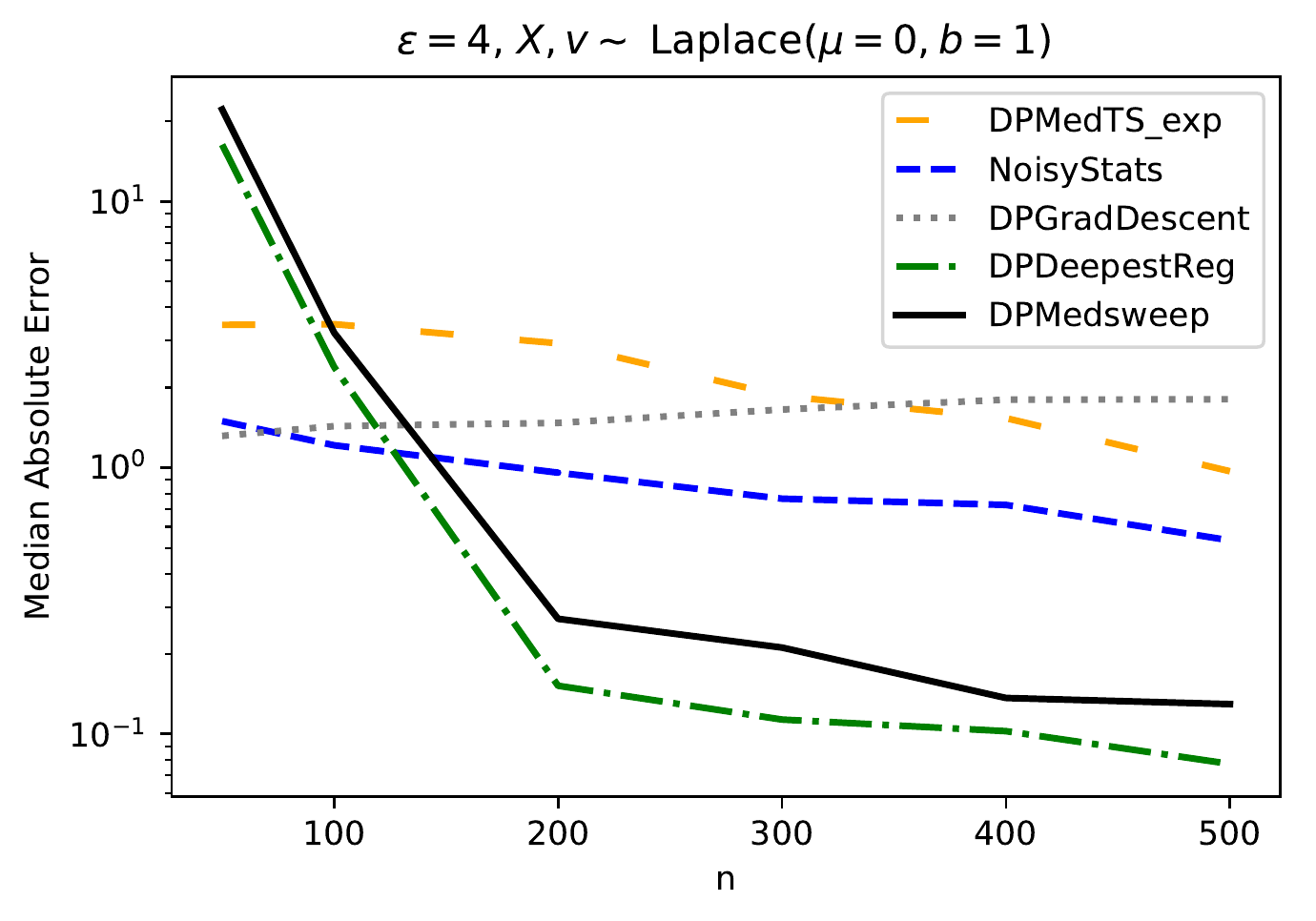} \includegraphics[scale=.58]{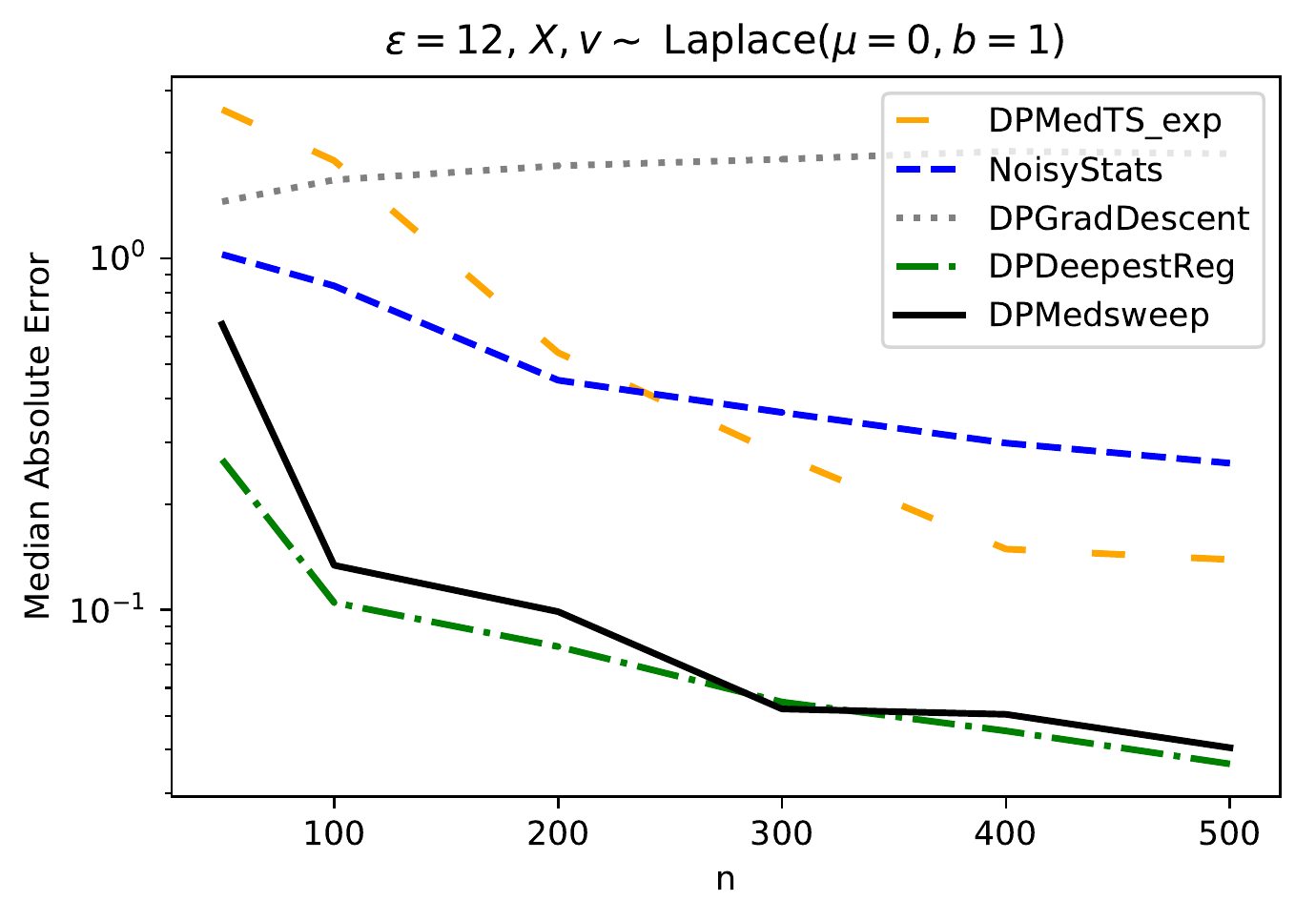} \\
    \includegraphics[scale=.58]{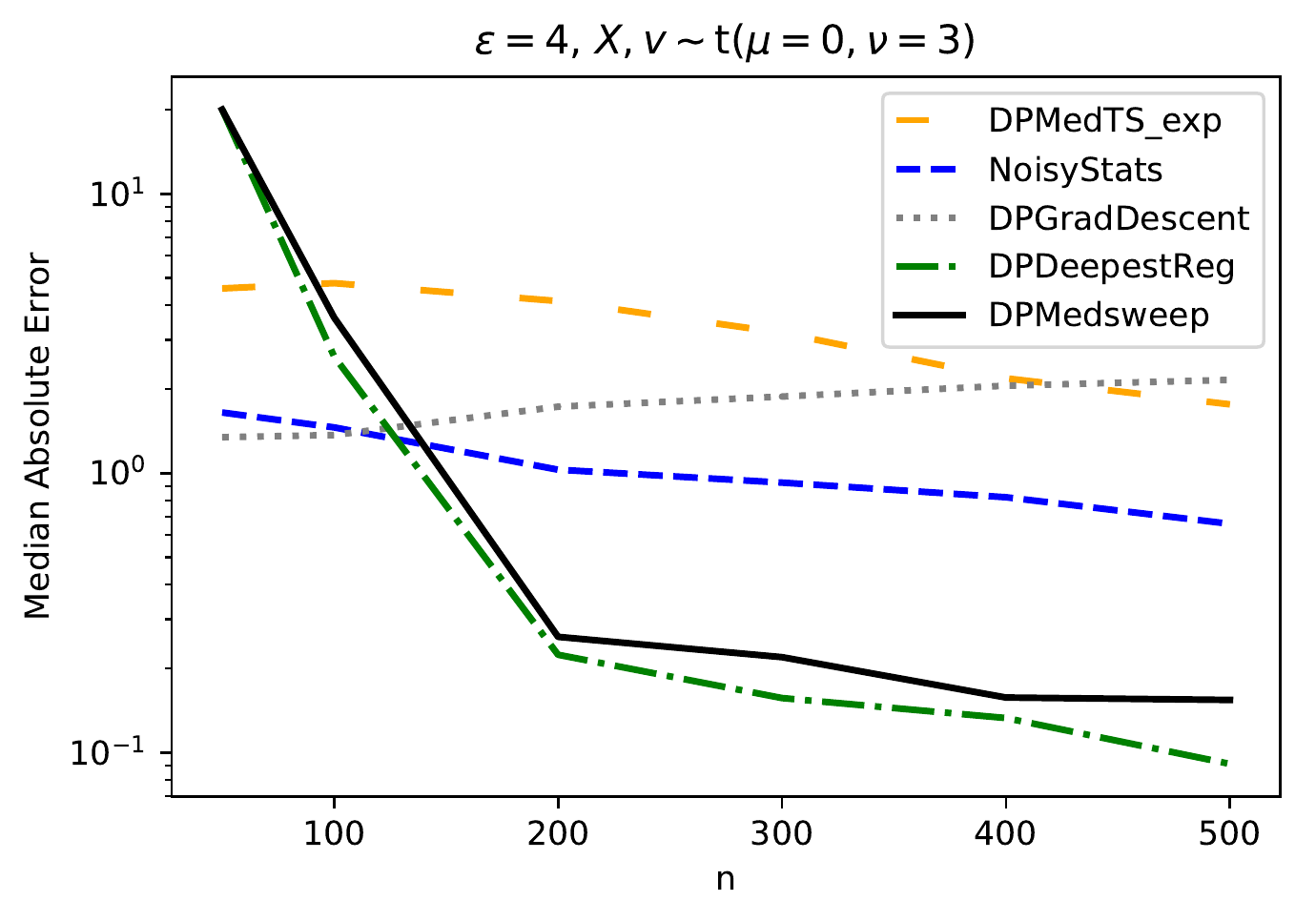} \includegraphics[scale=.58]{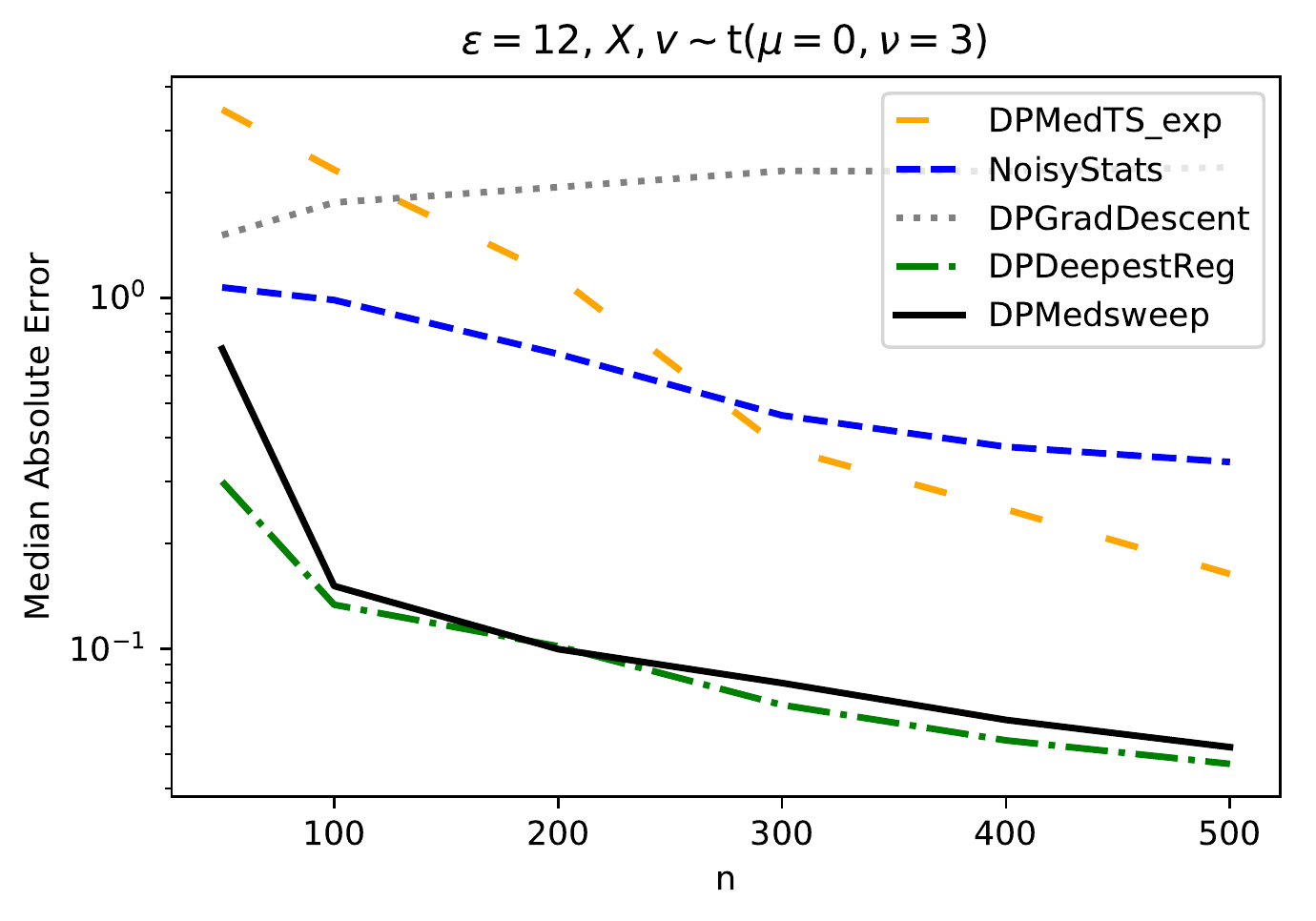}
    \caption{The median absolute error at the 25\% and 75\% percentile of $x_i$ of each DP estimator, for each DGP $x_i, v_i \sim N(\mu=0, \sigma^2=1)$ (top), $x_i, v_i \sim \textup{Laplace}(\mu=0, b=1)$ (middle), $x_i, v_i \sim t(\mu=0, \nu=3)$ (bottom), and for $\epsilon=4$ (left) and $\epsilon=12$ (right). Both NoisyStats$(\cdot)$ and DPMedTS\_exp$(\cdot)$ satisfy $\epsilon-$DP; DPGradDesc$(\cdot),$ DPMedsweep$(\cdot),$ and DPDeepestReg$(\cdot)$ satisfy $(\epsilon, 10^{-6})-$DP.}
    \label{fig:sims_MAE}
\end{figure}

\subsection{Simulation Results}

Before discussing the simulation results in detail, it may be worth pointing out our reasons for including both the mean squared error and the median absolute error of these methods. Both of these error metrics are examples of loss functions that one may seek to minimize with the choice of an estimation strategies. If this is done, the mean squared error may be regarded as the more prudent choice, in the sense that the upper tail of the absolute errors are given more weight. However, since NoisyStats$(\cdot)$ has an infinite expected absolute error, the mean squared error estimates for this mechanism do not converge in the number of simulation iterations. Thus, we choose to include median absolute error metrics solely because the population counterparts of these metrics exist for each mechanism, rather than because the invariance of median absolute error on the magnitude of the largest errors is an advantage in the context of selecting an estimator.

In the case of both mean squared error and median absolute error, one interesting feature of the simulation results is the consistency of DPGradDescent$(\cdot)$ across DGPs, values of $\epsilon,$ and sample sizes. This mechanism does not appear to exhibit a gain in accuracy as the sample sized increases; although, as described previously, this may be due to the fact that we only performed limited tuning of the choice parameters of this method. This mechanism also appears to outperform other estimators in terms of median absolute error and mean squared error in all DGPs when $\epsilon=4$ and $n\in\{50,100\}.$ In contrast, in the case of DPMedTS\_exp$(\cdot),$ both the error metrics we consider are more obviously decreasing in the sample sizes for most of the DGPs and $\epsilon$ combinations we consider. It is interesting to note that even the exponential tails of the Laplace distribution appear to be sufficiently thick to make the impact of sample size less apparent than for the normal distribution for DPMedTS\_exp$(\cdot).$ Since the NoisyStats$(\cdot)$ estimator does not have finite expected squared error, we will instead focus on the median absolute errors for this estimator. The impact of the tail behavior of the DGP on the statistical efficiency, \textit{i.e.}, the rate at which errors decrease as the sample size increases, appears to be less obvious in this case. For both DPMedTS\_exp$(\cdot)$ and NoisyStats$(\cdot),$ higher values of $\epsilon$ also appear to increase the statistical efficiency.

The two estimators proposed in this paper appear to provide similar performance to one another, but DPDeepestReg$(\cdot)$ appears to provide a slight gain in accuracy over DPMedsweep$(\cdot).$ These two estimators also appear to compare favorably to the classic DP regression estimators in these simulations when either the sample size or epsilon is sufficiently large. Also, note that this accuracy improvement is most notable in cases in which the tails of the DGP are thicker. 

To assess computational costs of the proposed approaches we also measured the median runtime for each sample size. On an ordinary laptop, in each simulation iteration with a sample size of 500, the median runtime was 5.8 seconds for DPDeepestReg$(\cdot)$ and 1.9 seconds for DPMedsweep$(\cdot).$ Since the vast majority of the runtime of DPMedsweep$(\cdot)$ was spent computing the smooth sensitivity, the runtime of this mechanism can be improved by reducing $\alpha,$ as defined in in Lemma \ref{lemma:ssMedsweep}, below $3/4,$ at the cost of increasing the noise scale used within this mechanism.

One advantageous feature of the proposed approaches in these simulations is that $\Theta$ can be defined as progressively larger sets as the sample size diverges without changing the distribution of the DP estimators. In other words, when the sample size is sufficiently large, the bound on the local sensitivity at distance $k$ multiplied by $\exp(-k \beta),$ \textit{i.e.}, $\exp(-k \beta) \widetilde{A}^{(k)}_{\hbtheta}(\ndat),$ typically obtains its maximum value at a relatively low value of $k,$ while the choice of $\Theta$ only impacts the value of $\exp(-k \beta) \widetilde{A}^{(k)}_{\hbtheta}(\ndat)$ when $k$ is large. In contrast, for the three classical approaches, the choice of bounds has a significant impact on the distribution of the output in all sample sizes. To explore this feature in the context of these simulations, we computed the minimum required diameter of $\Theta,$ defined as $\diam(\Theta) = \max_{\btheta_1,\btheta_2\in\Theta} \lVert \btheta_1-\btheta_2 \rVert_1,$ that would be required to change the scale of the Laplace noise used within DPDeepestReg$(\cdot)$ for at least one of the simulation iterations of any of the three DGPs we consider in these simulations, for each sample size and each $\epsilon.$ As shown in Table \ref{table:minDiam}, the resulting minimum diameter values appear to be increasing at an exponential rate in $n$ in these simulations, with a higher rate when $\epsilon=12.$

\begin{table} 
\centering
\begin{tabular}{ c | c c c c c c}
$\epsilon$ & $n = 50$ & $n = 100$ & $n = 200$ & $n = 300$ & $n = 400$ & $n = 500$\\ \hline 
4 & $2.0\times 10^2$ & $2.0\times 10^2$ & $4.4\times 10^2$ & $1.8\times 10^4$ & $9.7\times 10^5$ & $4.1\times 10^7$ \\ 
12 & $2.0\times 10^2$ & $1.1\times 10^4$ & $1.5\times 10^9$ & $2.0\times 10^{14}$ & $2.9\times 10^{19}$ & $3.1 \times 10^{24}$ \\  
\end{tabular}
\caption{The minimum value of $\diam(\Theta)$ that would be required to change the scale of the Laplace noise used within DPDeepestReg$(\cdot)$ in at least one of the simulation iterations described above. Note that $\diam([-50,50]^2)=200,$ so, for sufficiently small sample sizes, increasing $\diam(\Theta)$ from the one used in the experiments would increase this scale value in at least one simulation iteration.} \label{table:minDiam}
\end{table}

\section{Discussion} \label{sec:discussion}

This paper proposes several approximate DP and RDP estimators based on maximizing notions of statistical depth, including the Tukey median and the deepest regression estimator. An approximate DP mechanism for the Medsweep regression estimator is also proposed. Afterward, simulations are provided for both of the proposed approximate DP regression estimators, and these estimators appear to compare favorably to classic DP regression estimators when either the sample size or epsilon is large. 

One reason the input requirement of bounds on the data or the estimator itself is an area of emphasis in this paper is that we are not aware of approximate DP location or regression estimators that provide a reasonable balance between accuracy and privacy in the absence of strong assumptions on the population distribution, as described in Sections \ref{sec:intro} and \ref{sec:relatedLit}. These assumptions are unlikely to hold in practice for many use cases of interest, especially for use cases in which one or more variables exhibit thick tail behavior, and we are also not aware of a way to check these assumptions using a DP mechanism without making other strong distributional assumptions. 

There are also a number of areas for further research in the intersection of DP estimators and statistical depth. First, \cite{rousseeuw1999regression} provide a quantile regression estimator based on regression depth, along the lines of \cite{koenker1978regression}, and it appears that similar techniques to the ones used here can be used to formulate $\beta-$smooth upper bounds on the local sensitivity of these estimators. Second, improving the bound described in Theorem \ref{lemma:ssMedsweep} would also be advantageous for multivariate regression estimation. Third, since inference is often one of the main goals when using regressions, one important area for further work is the development of methods related to DP inference on the estimators described here.

\section*{Acknowledgements}
The author thanks John M. Abowd, Brian Finley, Daniel Kifer, Roger Koenker, and Tucker McElroy for their insightful comments.

\appendix

\section{Probability Bounds on Halfspace Depth} \label{appendixA}

\noindent
In this section we will provide results on convergence of the function

\begin{align*}
    L_n(\btheta, \bu) = \frac{1}{n} \sum_i^n \mathbf{1}_{[0,\infty)}\left( \bu^\top (\bxi - \btheta)  \right),
\end{align*}

\noindent
to its population counterpart,

\begin{align*}
    L(\btheta, \bu) = \int \mathbf{1}_{[0,\infty)}\left( \bu^\top (\bx - \btheta)   \right) \textup{d} F(\bx).
\end{align*}

\noindent
Let $\mathcal{C}_L$ denote the class with elements given by $\{ \bx : \bu^\top (\bx - \btheta) \geq 0\}.$ To find uniform error bounds for $L_n( \btheta, \bu)$ we will start by bounding the  $n^{\textup{th}}$ shatter coefficient of the class $\mathcal{C}_L,$ which is the maximum number of nonempty intersections of $\ndat$ with elements of the class $\mathcal{C}_L$ \citep{vapnik2015uniform}. More formally, the $n^{\textup{th}}$ shatter coefficient is defined as

\begin{align}
    s(\mathcal{C}_L, n) = \max_{\ndat} \card \{\ndat \cap C : C \in \mathcal{C}_L\}.
\end{align}

\begin{lemma} \label{Lvcdim}
    The $n^{\textup{th}}$ shatter coefficient of the class of sets $\mathcal{C}_L$ with elements given by $\{\bx\in\rr^{d}:  \bu^\top (\bx - \btheta) \geq 0  \},$ where  $\bu \neq \boldsymbol{0}$ and $\btheta\in \Theta,$ satisfies $s(\mathcal{C}_L, n) \leq 2 (n-1)^d + 2.$
\end{lemma}
\begin{proof}
    See for example, Corollary 13.1 of \citep{devroye2013probabilistic}.
\end{proof}

Lemma \ref{Lvcdim} implies that $\sup_{\btheta} \lvert \hdepth(\btheta, \ndat)/n - \inf_{\bu\neq\boldsymbol{0}} L(\btheta, \bu) \rvert$ converges almost surely to zero at a rate of $O(1/\sqrt{n}).$ This is also already a well known result in the literature, since deriving the limiting value of the objective function is typically a first step for showing consistency of an estimate (\textit{i.e.}: $\arg \inf_{\btheta} \hdepth(\btheta, \ndat)$ in this case); see for example, \citep{masse2002asymptotics,bai2008asymptotic}. Next we will use this result to derive an explicit finite sample probability bound on $\sup_{\btheta\in\Theta, \; \bu\neq \boldsymbol{0}} \lvert L(\btheta, \bu) - L(n, \btheta, \bu) \rvert.$

\begin{lemma}\label{LcontourBound}
    Suppose $\ndat$ is composed of observations that were sampled independently from the population distribution $\Ndat.$ Then, for $\kappa > 0,$ the following bound holds

    \begin{align*}
        P\left(\sup_{\btheta\in\Theta, \; \bu\neq \boldsymbol{0}} \lvert L(\btheta, \bu) - L(n, \btheta, \bu) \rvert \geq \kappa\right) \leq 8 \left((n^2-1)^d + 1\right) \exp\left(2 \kappa (2+(2-n) \kappa)\right).
    \end{align*}
\end{lemma}
\begin{proof}
    This a direct consequence of Lemma \ref{Lvcdim} and the main result of \citep{devroye1982bounds}, which states that there exists a constant $c$ such that for all $\kappa>0,$

    \begin{align*}
        P\left(\sup_{\btheta\in\Theta, \; \bu\neq \boldsymbol{0}} \lvert G(\btheta, \bu) - G(n, \btheta, \bu) \rvert \geq \kappa\right) \leq c s(\mathcal{C}_L, n^2) \exp\left(-2 \kappa^2 n \right),
    \end{align*}

    \noindent
    and $c \leq 4 \exp(4 \kappa - \kappa^2).$
\end{proof}

\begin{remark}
    This result prioritizes the probability bounds being explicitly computable in practice and being as tight as possible for typical sample sizes. It is worth pointing out that this is not the best theoretical rate. For example, both of the main results of \cite{talagrand1994sharper} provide uniform error bounds with a better rate of convergence, but we are not aware of explicit numerical bounds on the constants in these results for them to be used in practice.
\end{remark}

\section{Probability Bounds on Regression Depth} \label{appendixB}

\noindent
In this section we will provide results on convergence of the function

\begin{align*}
    G_n(\btheta, \bu) = \frac{1}{n} \sum_i^n \mathbf{1}_{[0,\infty)}\left( (1, \bxi^\top) \bu \; \sign(y_i - (1, \bxi^\top) \btheta) \right),
\end{align*}

\noindent
to its population counterpart,

\begin{align*}
    G(\btheta, \bu) = \int \mathbf{1}_{[0,\infty)}\left( (1, \bx^\top) \bu \; \sign(y - (1,\bx^\top) \btheta)  \right) \textup{d} F.
\end{align*}

Let $\mathcal{C}_R$ denote the class with elements given by $\{ (\bx,y) : (1, \bx^\top) \bu \; \sign(y - (1, \bx^\top) \btheta) \geq 0\}.$ To find uniform error bounds for $G_n( \btheta, \bu)$ we will start by bounding the  $n^{\textup{th}}$ shatter coefficient of the class $\mathcal{C}_R,$ which is the maximum number of nonempty intersections of $\ndat$ with elements of the class $\mathcal{C}_R$ \citep{vapnik2015uniform}. More formally, the $n^{\textup{th}}$ shatter coefficient is defined as

\begin{align}
    s(\mathcal{C}_R, n) = \max_{\ndat} \card \{\ndat \cap C : C \in \mathcal{C}_R\}.
\end{align}

\begin{lemma} \label{Rvcdim}
    The $n^{\textup{th}}$ shatter coefficient of the class of sets $\mathcal{C}_R$ with elements given by $\{(\bx, y)\in\rr^{d+1}:  \frac{y - (1, \bx^\top) \btheta}{(1, \bx^\top) \bu} \geq 0 \},$ where  $\bu \neq \boldsymbol{0}$ and $\btheta\in \Theta,$ satisfies $s(\mathcal{C}_R, n) \leq 16 ((n-1)^{d-1} + 1)^4.$
\end{lemma}
\begin{proof}
    We will start by finding the $n^{\textup{th}}$ shatter coefficients of several intermediate classes. Specifically, let $\mathcal{C}_1$ and $\mathcal{C}_2$ be defined as

    \begin{align*}
        \mathcal{C}_1 &= \{ \{(\bx, y)\in\rr^{d+1} : (1, \bx^\top) \bu > 0 \} : \bu \neq \boldsymbol{0} \}\\
        \mathcal{C}_2 &= \{ \{(\bx, y) \in \rr^{d+1} : \sign(y -(1, \bx^\top) \btheta) \geq 0 \} : \btheta\in \Theta \} \\
        \widetilde{\mathcal{C}}_1 &= \{ \{(\bx, y)\in C_1 \cap C_2 \} : C_1 \in \mathcal{C}_1, \; C_2 \in \mathcal{C}_2 \}.
    \end{align*}

    \noindent
    Since both classes are defined by linear functions of dimension $d,$ for each $i\in \{1,2\},$ we have $s(\mathcal{C}_i, n) \leq 2 (n-1)^{d-1} + 2;$ see for example, Theorem 13.9 of \citep{devroye2013probabilistic} and the subsequent remark. Thus, after considering the case in which both classes pick out the most possible subsets of $\ndat$ simultaneously, we have $s(\widetilde{\mathcal{C}_1}, n) \leq 4 ((n-1)^{d-1} + 1)^2.$

    This same argument can be repeated to show $s(\widetilde{\mathcal{C}_2}, n) \leq 4 ((n-1)^{d-1} + 1)^2,$ where 

    \begin{align*}
        \widetilde{\mathcal{C}}_2 &= \{ \{(\bx, y) \in \rr^{d+1} : (1, \bx^\top) \bu < 0,  \; \sign(y -(1, \bx^\top) \btheta) \leq 0  \} : \bu \neq \boldsymbol{0}, \; \btheta\in \Theta \}.
    \end{align*}

    Since each element of $\mathcal{C}_R$ is given by $\widetilde{C}_1 \cup \widetilde{C}_2,$ where $ \widetilde{C}_1 \in \widetilde{\mathcal{C}}_1 $ and $\widetilde{\mathcal{C}}_1,$ a similar argument that was made when defining classes using intersections results in the bound, $s(\mathcal{C}_R, n) \leq 16 ((n-1)^{d-1} + 1)^4.$
\end{proof}

Note that this bound on the shattering coefficient also implies $\sup_{\btheta} \lvert \rdepth(\btheta, \ndat)/n - \inf_{\bu\neq \boldsymbol{0}} G(\btheta, \bu) \rvert$ converges almost surely to zero at a rate of $O(1/\sqrt{n}),$ which is already a known result in the literature; see for example, \cite{gao2020robust}. Lemma \ref{Rvcdim} also implies our required probability bound, which is provided below.

\begin{figure}[ht]
    \centering
    \includegraphics[scale=.58]{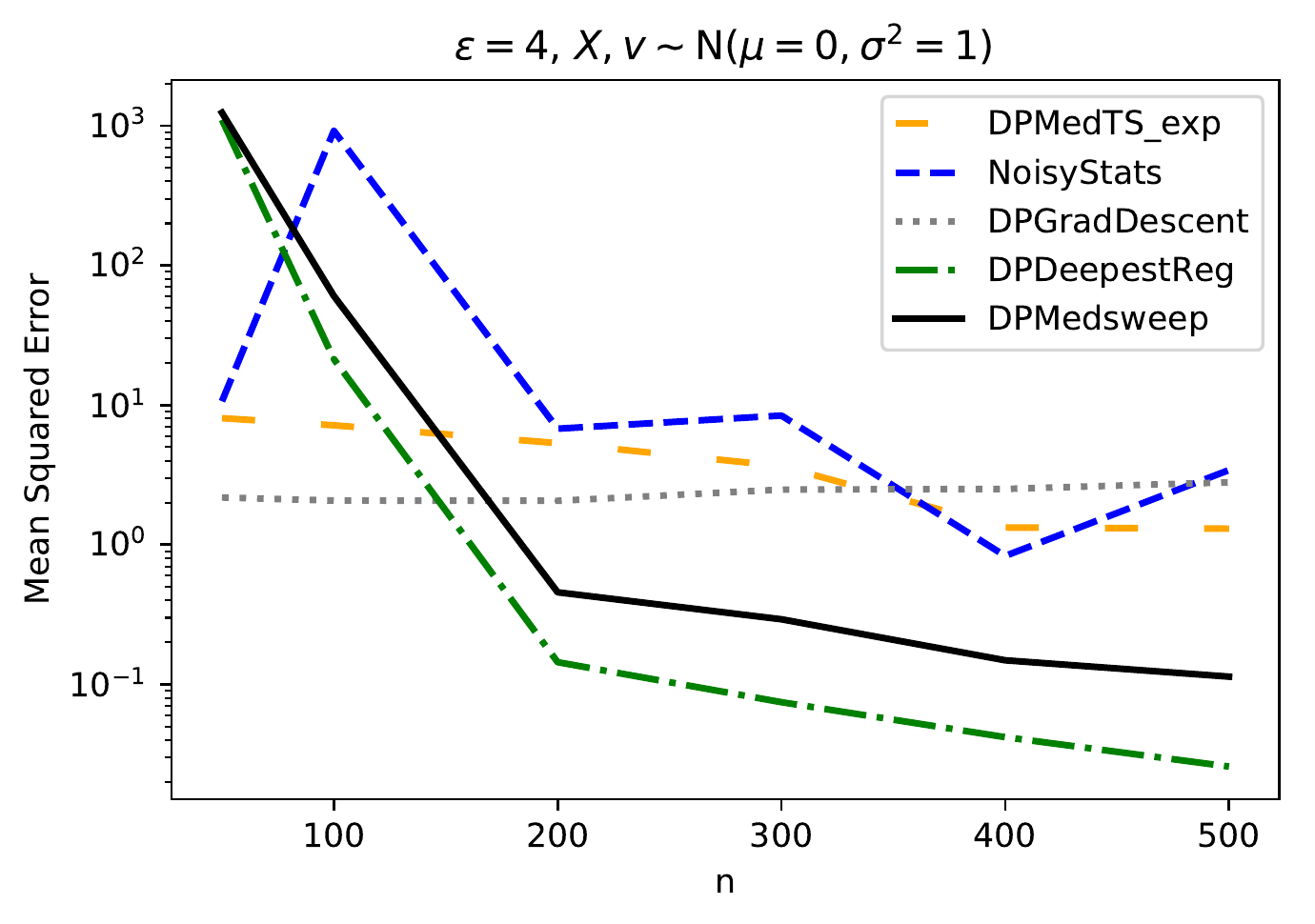} \includegraphics[scale=.58]{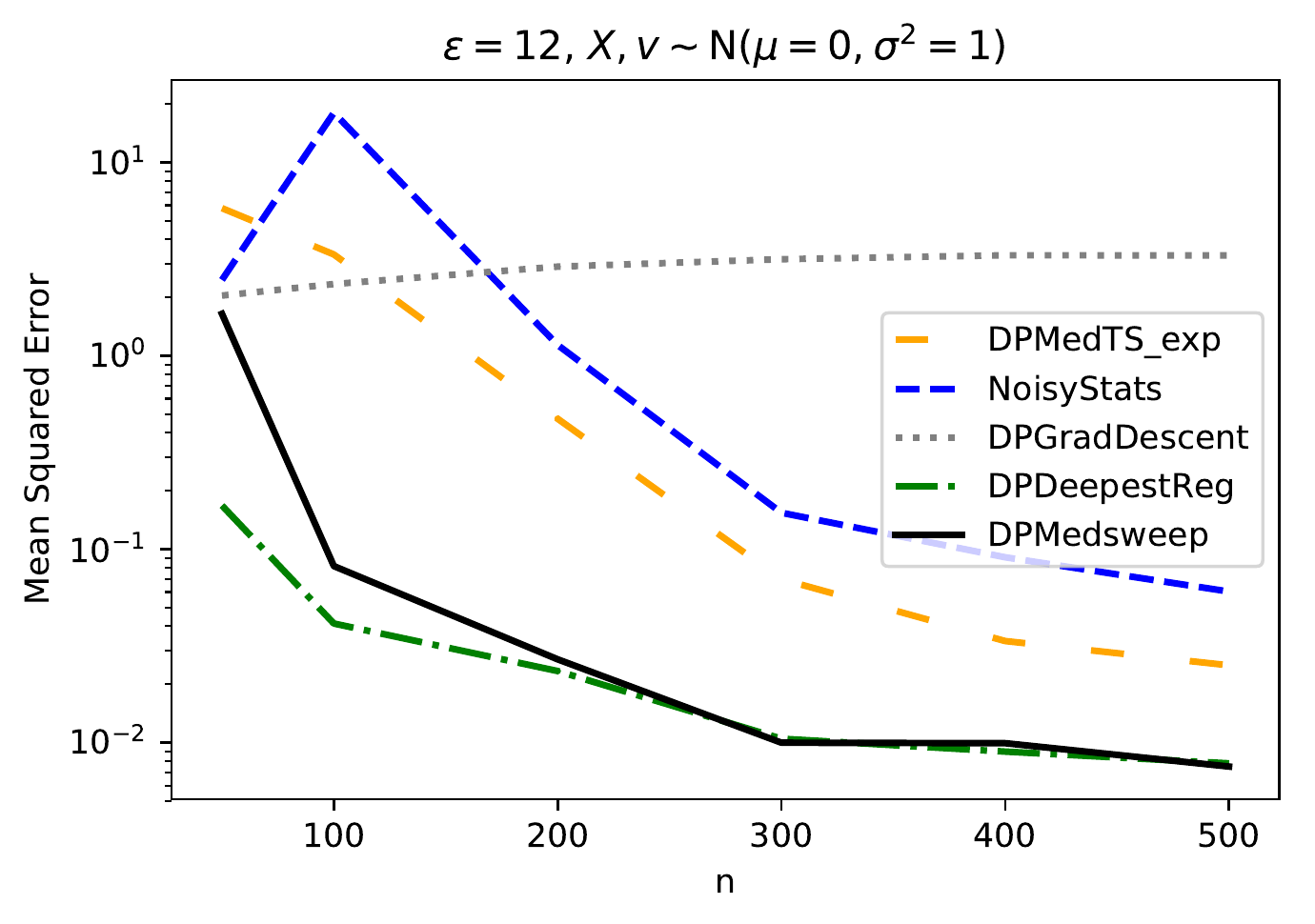} \\
    \includegraphics[scale=.58]{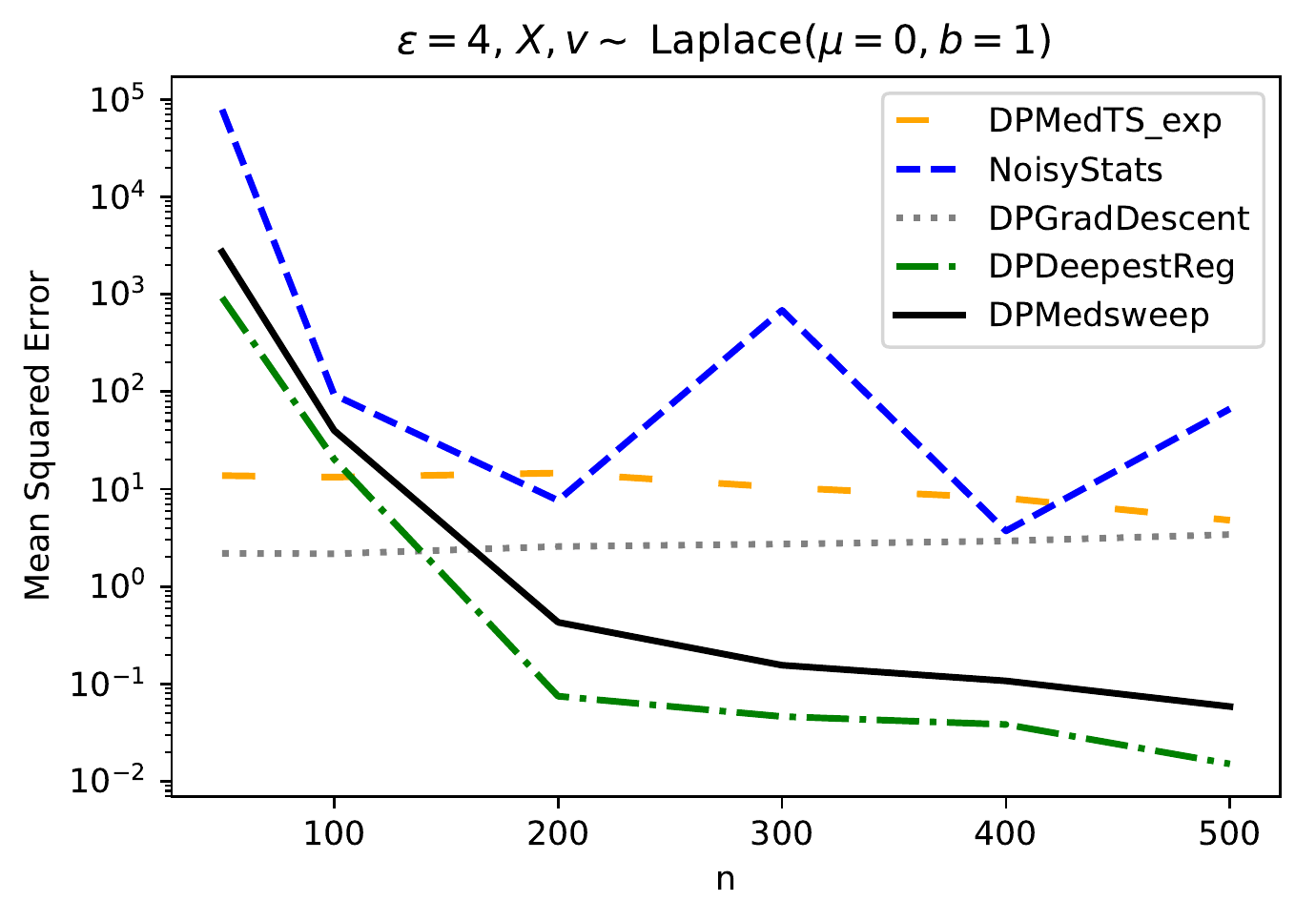} \includegraphics[scale=.58]{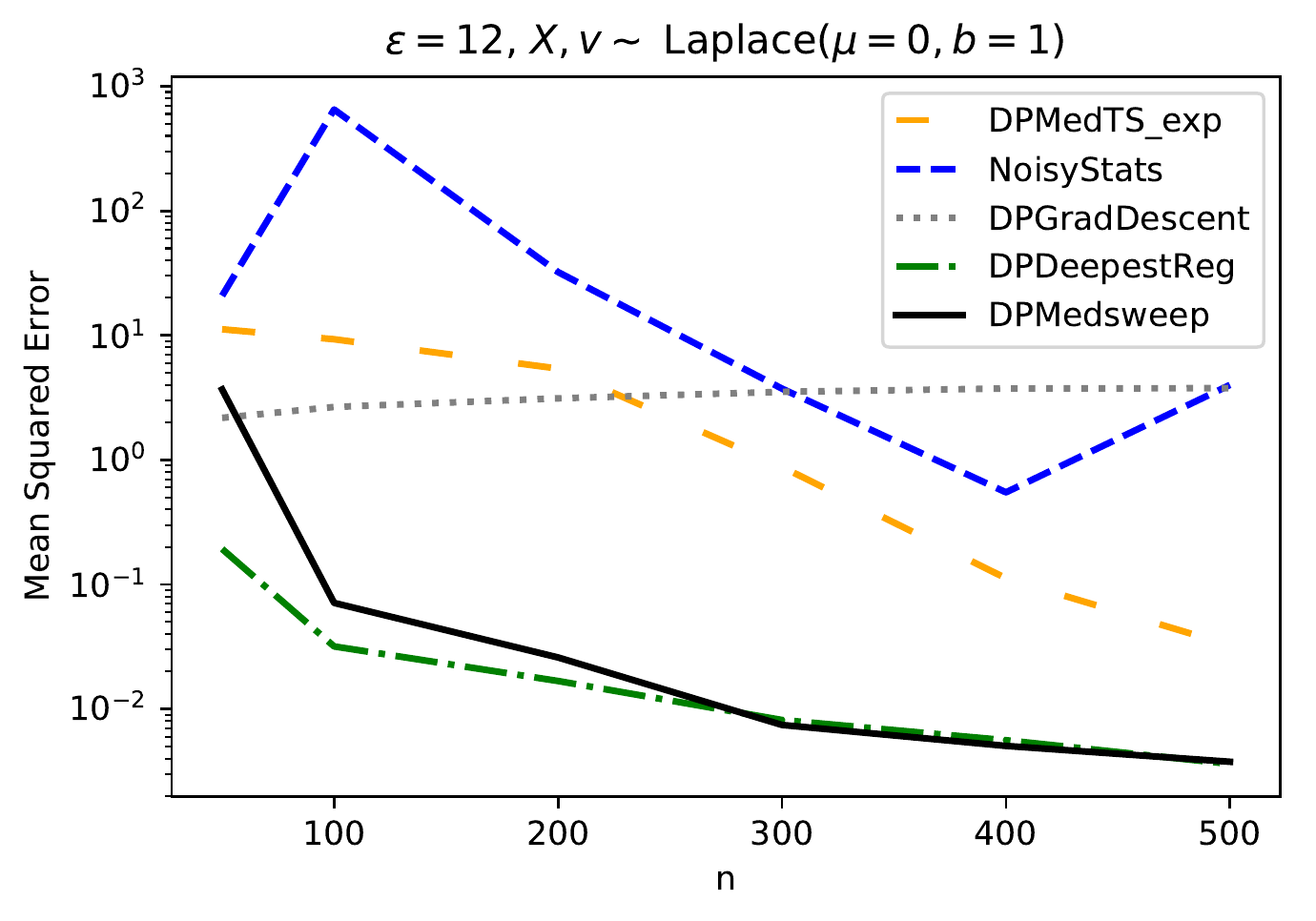} \\
    \includegraphics[scale=.58]{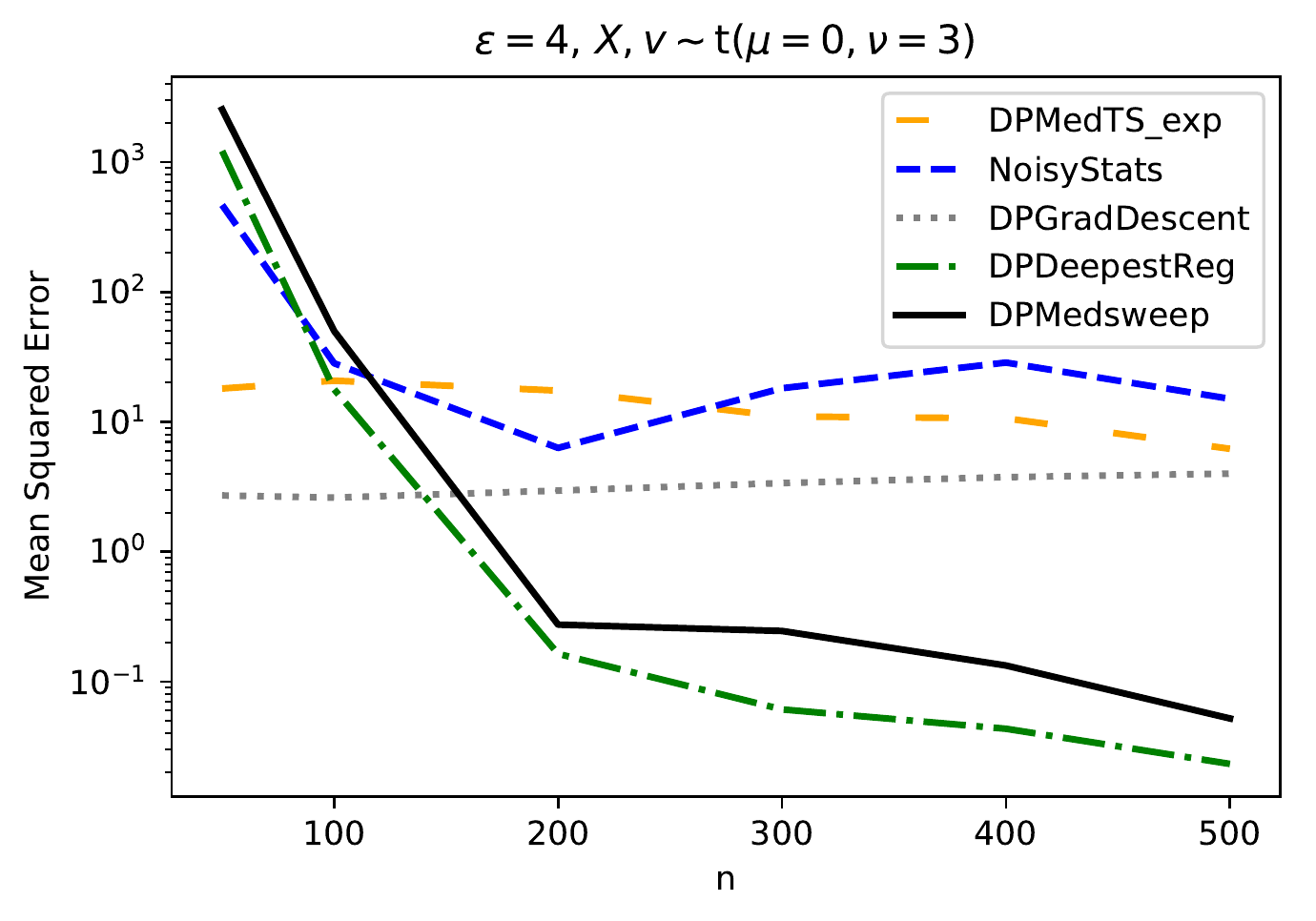} \includegraphics[scale=.58]{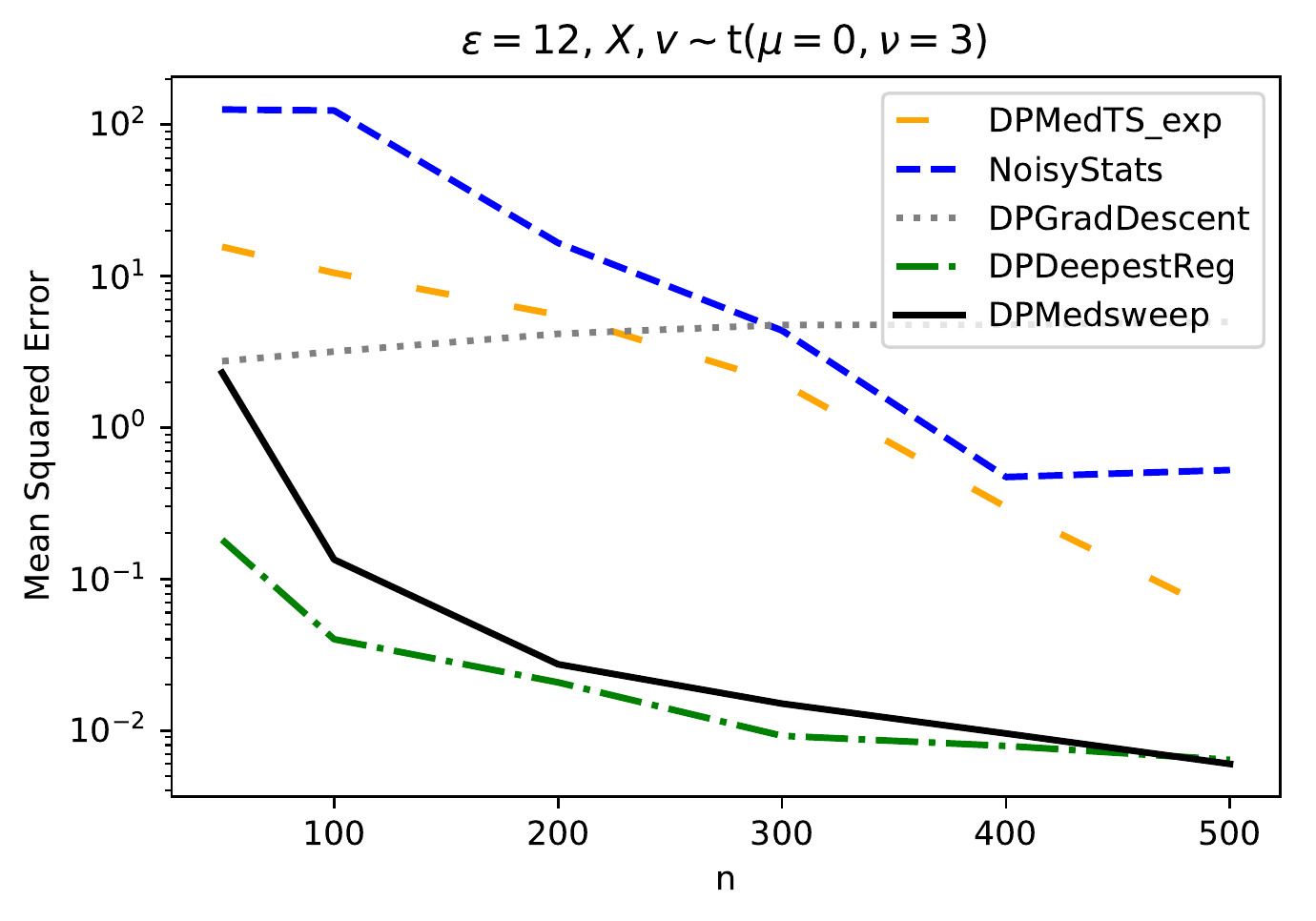}
    \caption{The mean squared error at the 25\% and 75\% percentile of $x_i$ of each DP estimator, for each DGP $x_i, v_i \sim N(\mu=0, \sigma^2=1)$ (top), $x_i, v_i \sim \textup{Laplace}(\mu=0, b=1)$ (middle), $x_i, v_i \sim t(\mu=0, \nu=3)$ (bottom), and for $\epsilon=4$ (left) and $\epsilon=12$ (right). For NoisyStats$(\cdot)$ and the DPMedTS\_exp$(\cdot)$ mechanisms, $\psi$ was set to 0.95. Both NoisyStats$(\cdot)$ and DPMedTS\_exp$(\cdot)$ satisfy $\epsilon-$DP; DPGradDesc$(\cdot),$ DPMedsweep$(\cdot),$ and DPDeepestReg$(\cdot)$ satisfy $(\epsilon, 10^{-6})-$DP.}
    \label{fig:sims_MSE_05}
\end{figure}

\begin{figure}[ht]
    \centering
    \includegraphics[scale=.58]{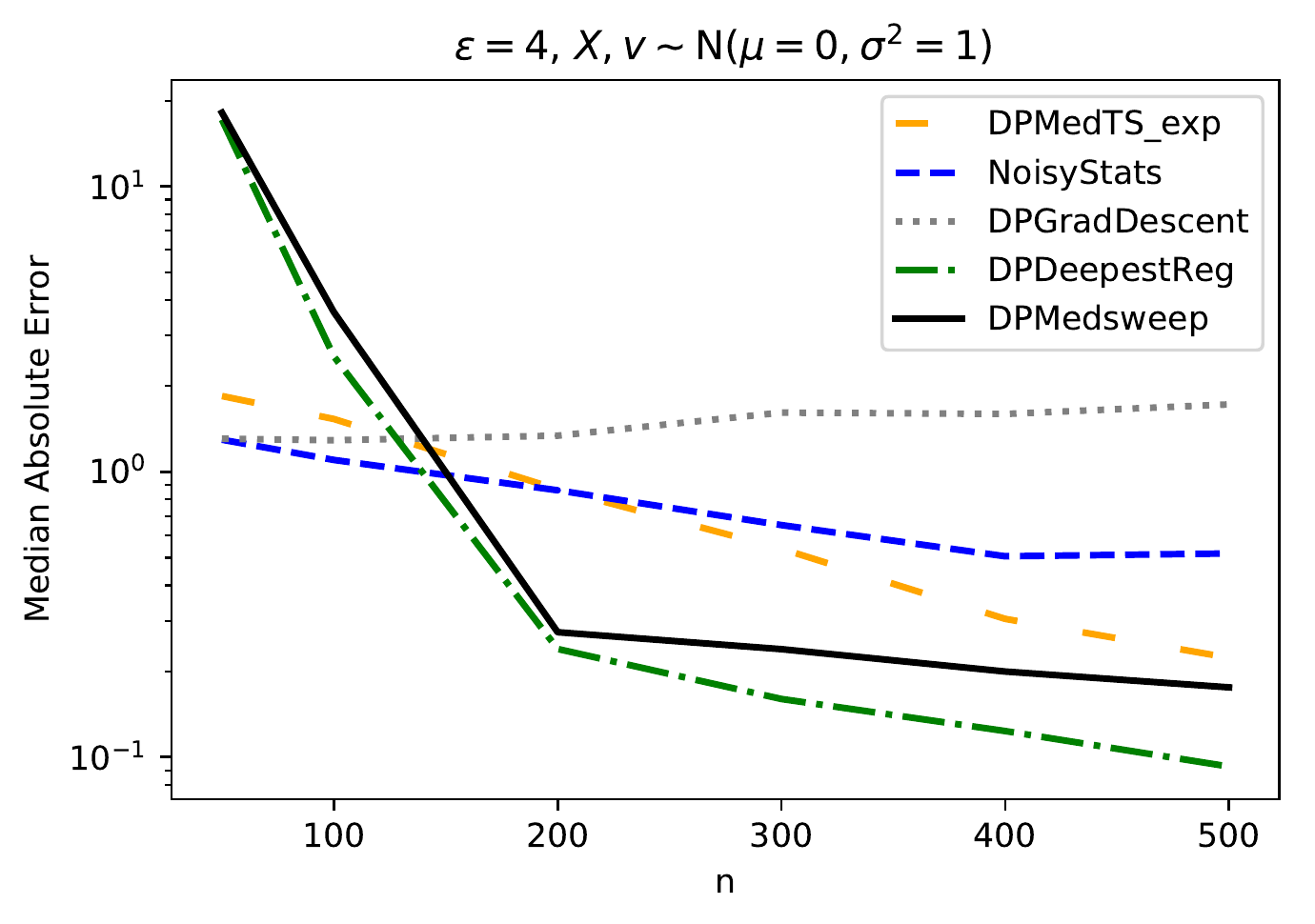} \includegraphics[scale=.58]{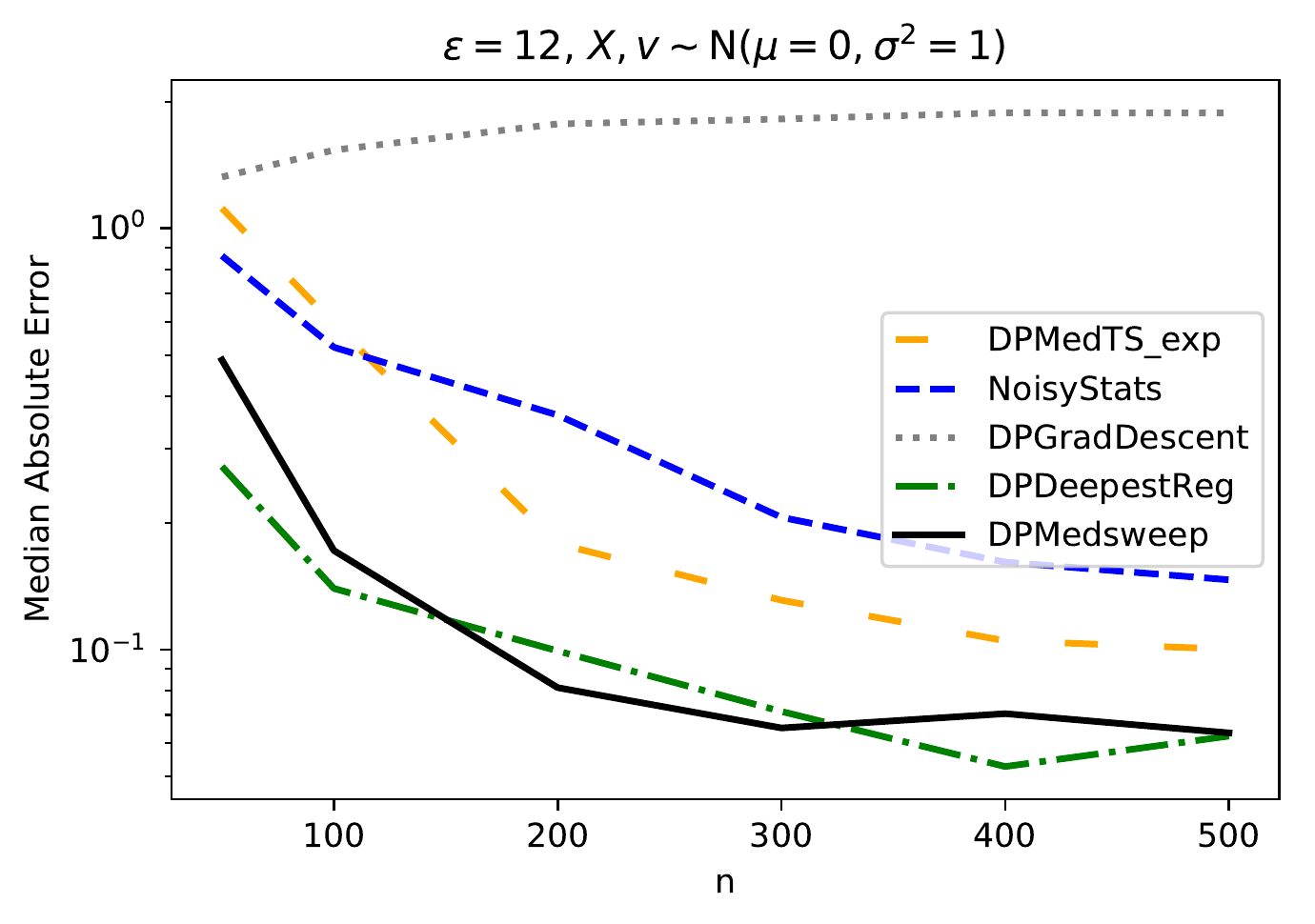} \\
    \includegraphics[scale=.58]{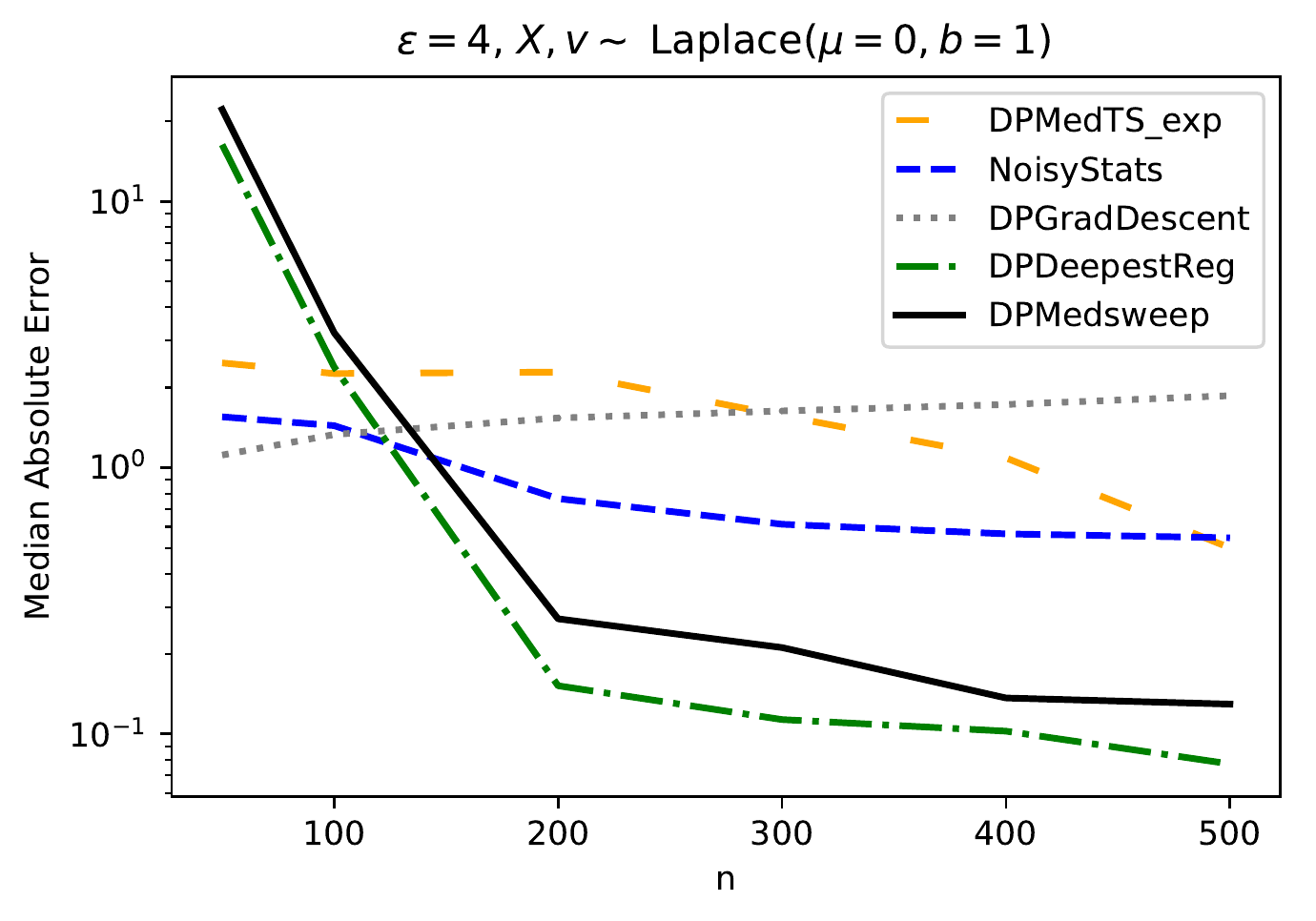} \includegraphics[scale=.58]{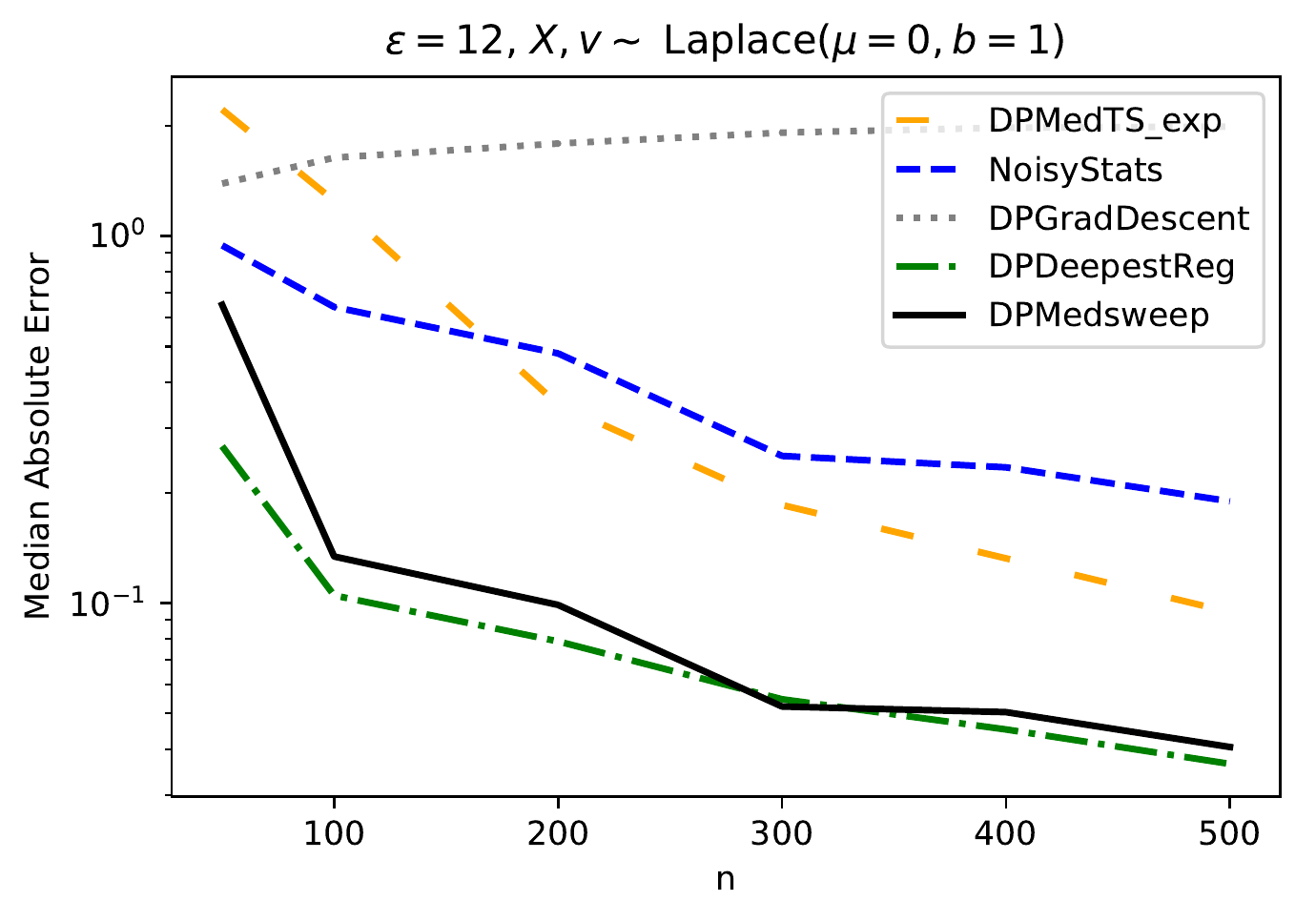} \\
    \includegraphics[scale=.58]{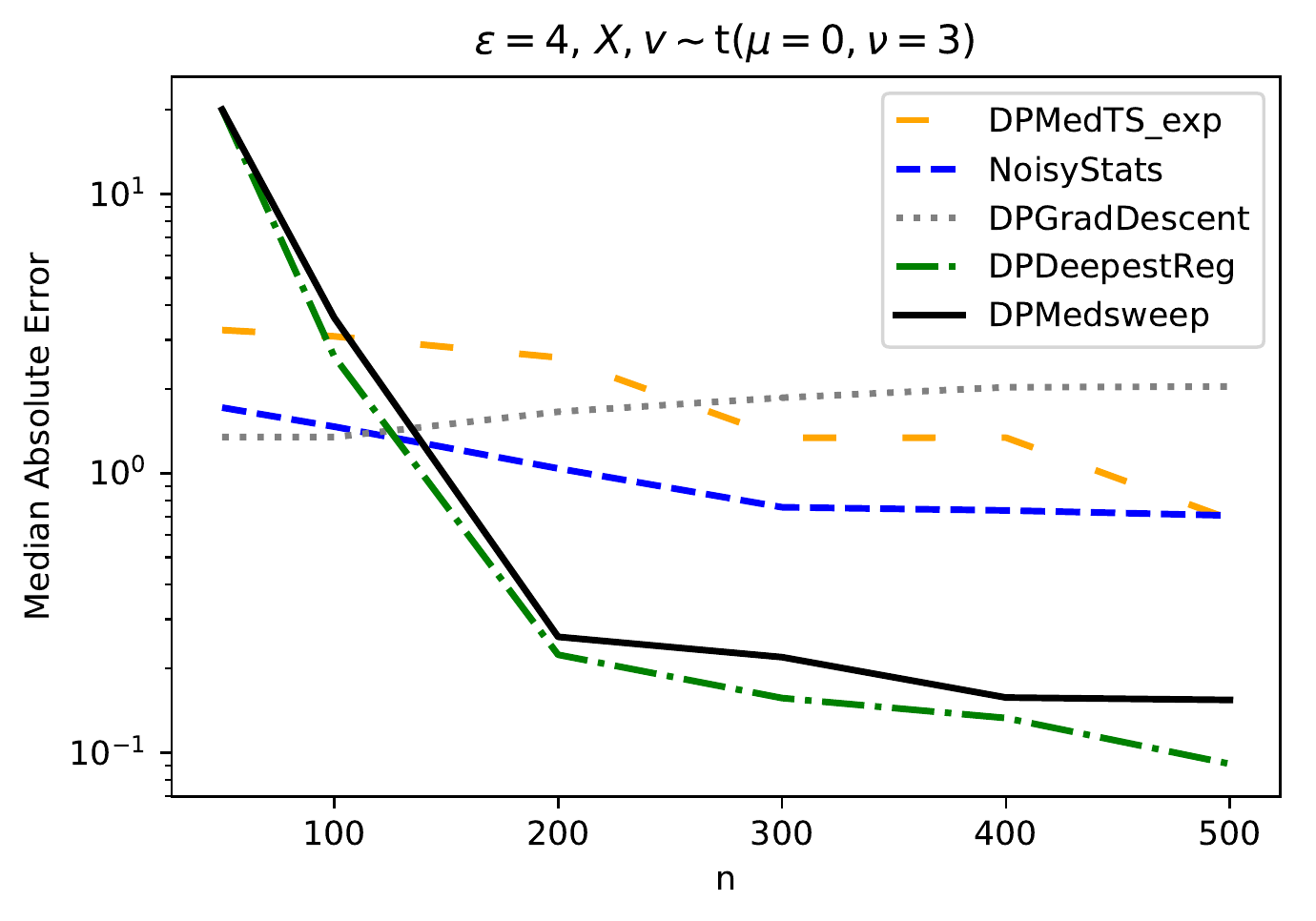} \includegraphics[scale=.58]{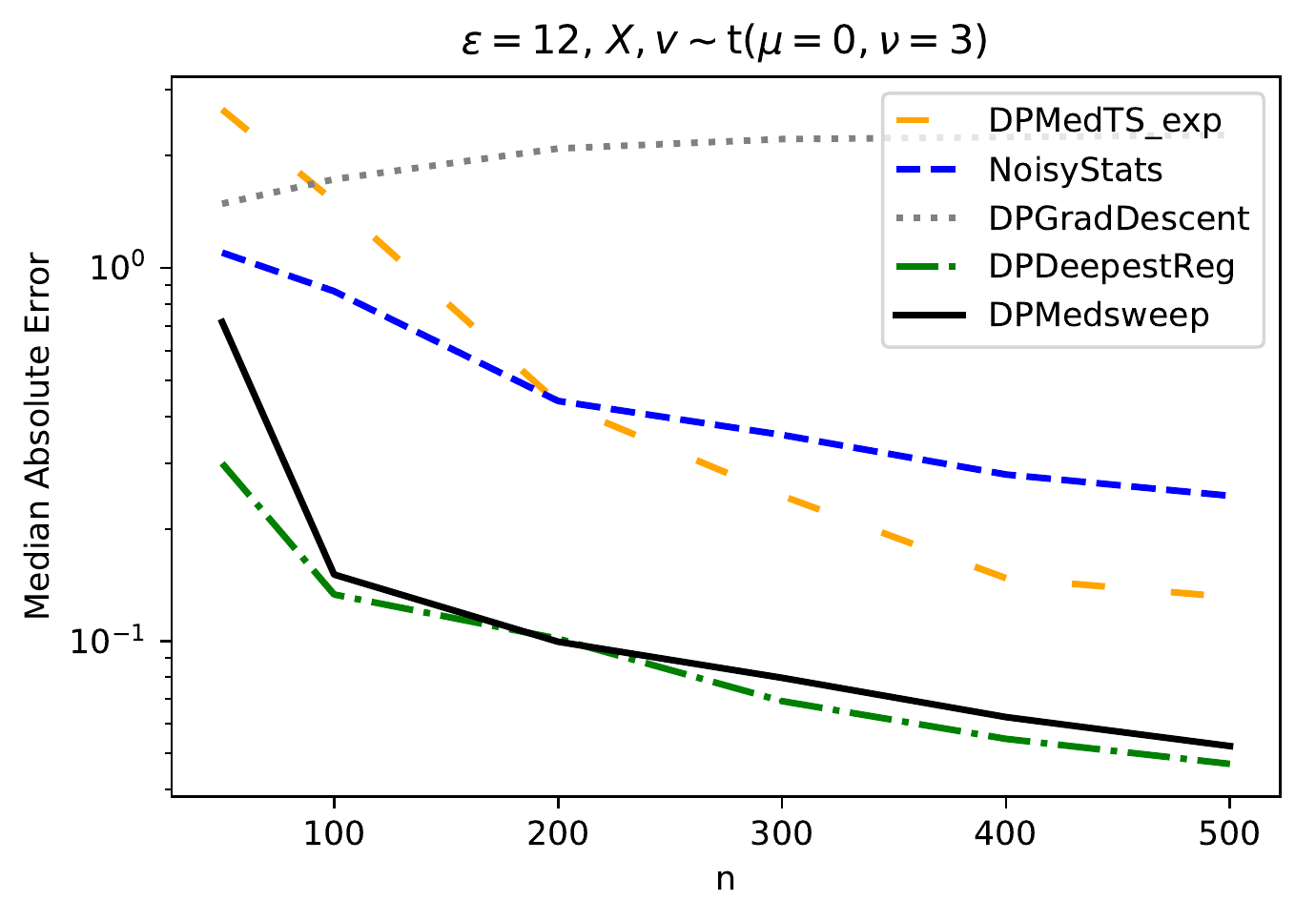}
    \caption{The median absolute error at the 25\% and 75\% percentile of $x_i$ of each DP estimator, for each DGP $x_i, v_i \sim N(\mu=0, \sigma^2=1)$ (top), $x_i, v_i \sim \textup{Laplace}(\mu=0, b=1)$ (middle), $x_i, v_i \sim t(\mu=0, \nu=3)$ (bottom), and for $\epsilon=4$ (left) and $\epsilon=12$ (right). For NoisyStats$(\cdot)$ and the DPMedTS\_exp$(\cdot)$ mechanisms, $\psi$ was set to 0.95. Both NoisyStats$(\cdot)$ and DPMedTS\_exp$(\cdot)$ satisfy $\epsilon-$DP; DPGradDesc$(\cdot),$ DPMedsweep$(\cdot),$ and DPDeepestReg$(\cdot)$ satisfy $(\epsilon, 10^{-6})-$DP.}
    \label{fig:sims_MAE_05}
\end{figure}

\begin{lemma}\label{RcontourBound}
    Suppose $\ndat$ is composed of observations that were sampled independently from the population distribution $\Ndat.$ Then, for $\kappa > 0,$ the following bound holds

    \begin{align}
        P\left(\sup_{\btheta\in\Theta, \; \bu\neq \boldsymbol{0}} \lvert G(\btheta, \bu) - G_n( \btheta, \bu) \rvert \geq \kappa\right) \leq 64 \left((n^2-1)^{d-1} + 1\right)^4 \exp\left(2 \kappa (2+(2-n) \kappa)\right).
    \end{align}
\end{lemma}
\begin{proof}
    This a direct consequence of Lemma \ref{Rvcdim} and the main result of \citep{devroye1982bounds}; see Theorem \ref{LcontourBound} for more detail.
\end{proof}

\begin{remark}
    The remark following Theorem \ref{LcontourBound}, regarding the rate of convergence of the probability bounds provided above being suboptimal, are also applicable in this case.
\end{remark}

\section{Simulation Results For Different Choices of \texorpdfstring{$\psi$}{psi}} \label{appendixC}

As described previously, we set the bounds on each $x_i$ and $y_i$ in our simulations for both the NoisyStats$(\cdot)$ and the DPMedTS\_exp$(\cdot)$ mechanisms by finding the smallest bounding box of the form $[-c, c]^2$ such that a proportion $\psi\in (0,1)$ of the data points are in $[-c, c]^2.$ The simulations provided above used $\psi=0.98,$ and, since this choice did have a significant impact on the simulation results, this  Appendix \ref{appendixC} provides the simulation results for the alternative choices of $\psi=0.95$ and $ \psi=1.$ Specifically, Figures \ref{fig:sims_MSE_05} and \ref{fig:sims_MAE_05} provide the mean squared error and median absolute error when $\psi=0.95,$ and Figures \ref{fig:sims_MSE_0} and \ref{fig:sims_MAE_0} provide these error metrics for the case in which $\psi=1.$ 

For the case in which $\psi=0.95,$ all of the comments regarding the relative performance of the DP estimators from Section \ref{sec:simulations} appear to still be the case; however, in almost all cases the performance appears to improve when $\psi=0.95,$ relative to when $\psi=0.98,$ by between 50 and 200\%, with the largest improvements for cases in which $\epsilon=12$ and for the DGPs with thicker tails. In contrast, for the cases in which $\psi=1,$ the statistical efficiency appears to change more noticeably; unsurprisingly, this is most apparent when $\epsilon$ is small and for both of the DGPs with thicker tails.

\begin{figure}[ht]
    \centering
    \includegraphics[scale=.58]{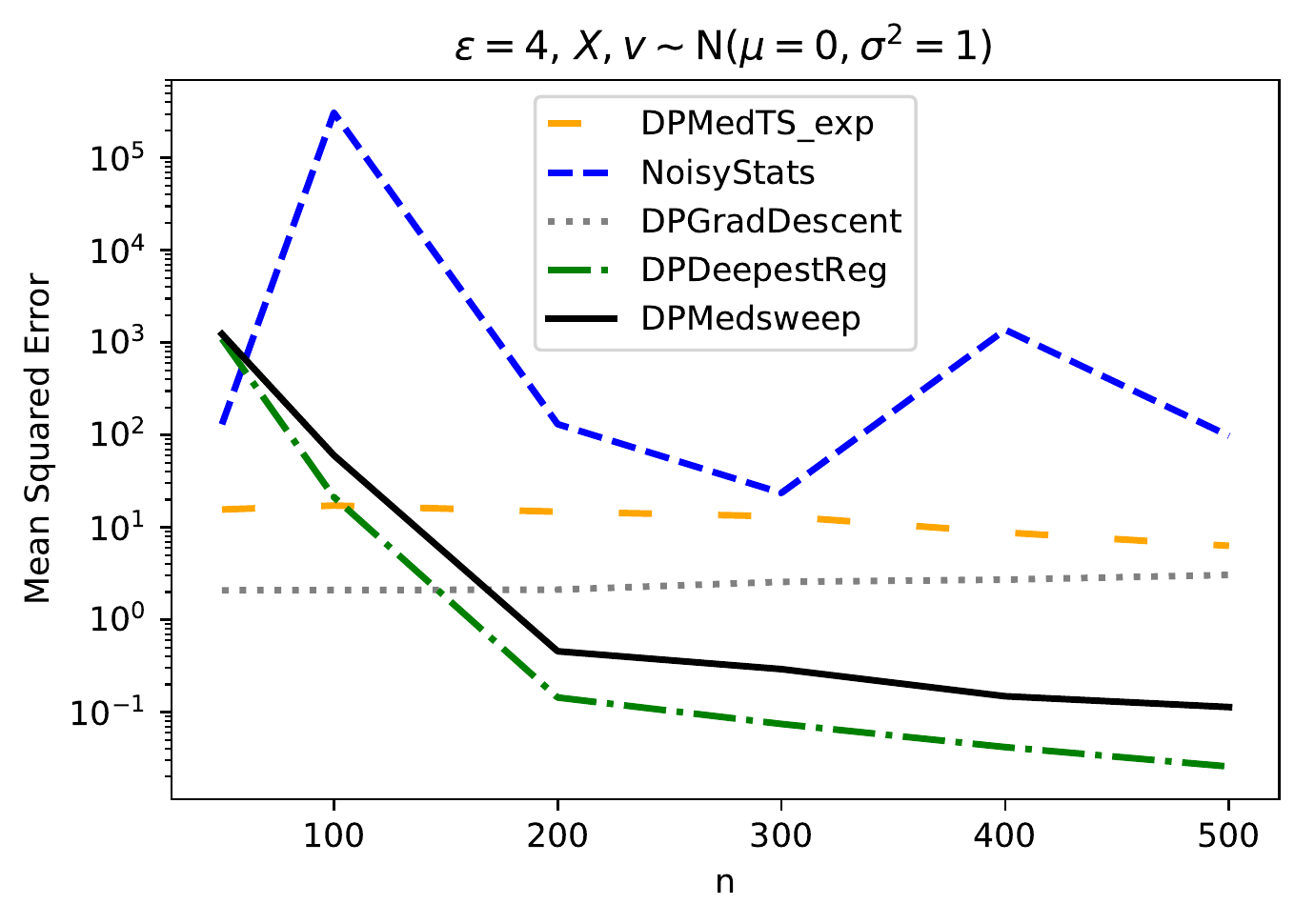} \includegraphics[scale=.58]{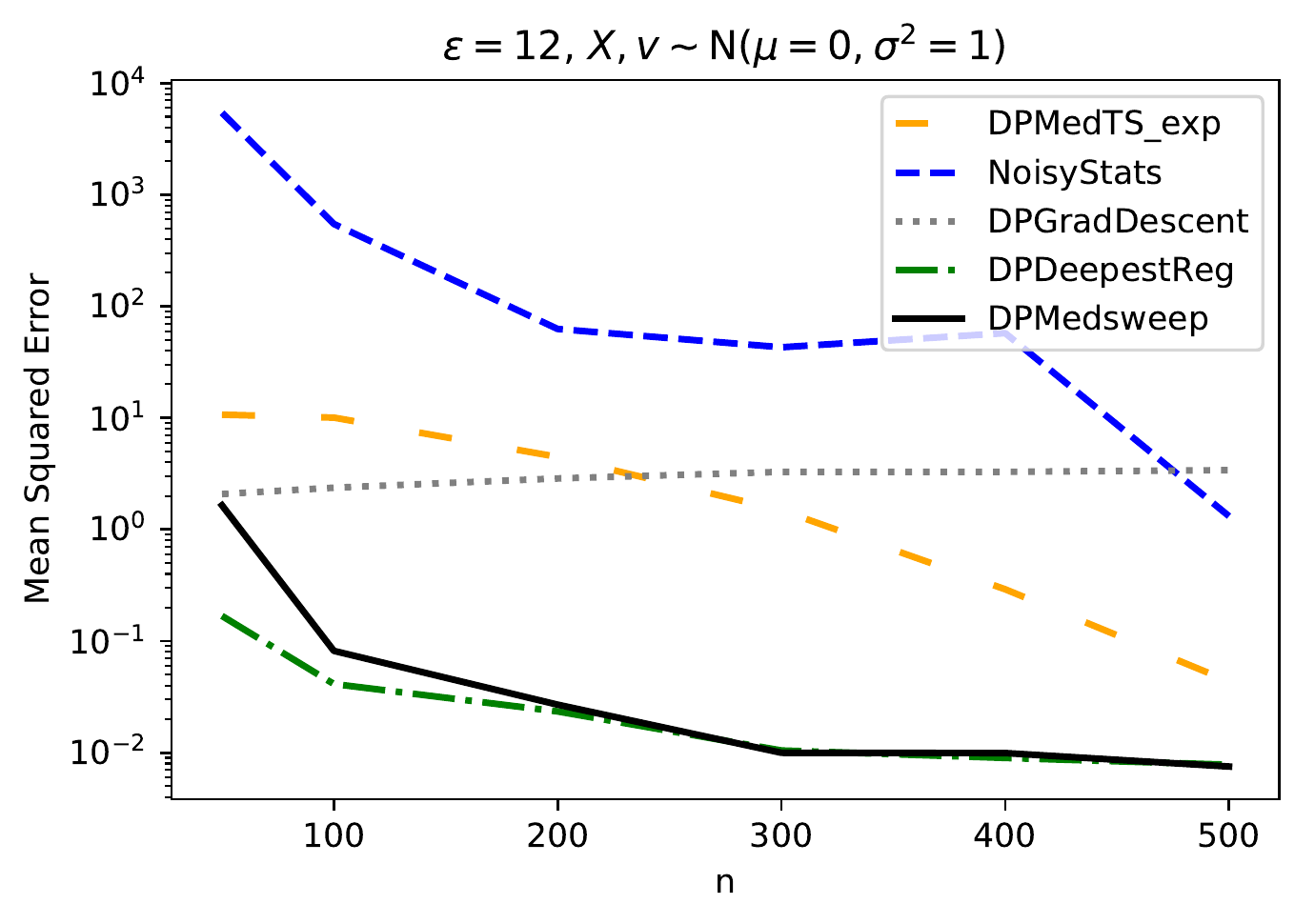} \\
    \includegraphics[scale=.58]{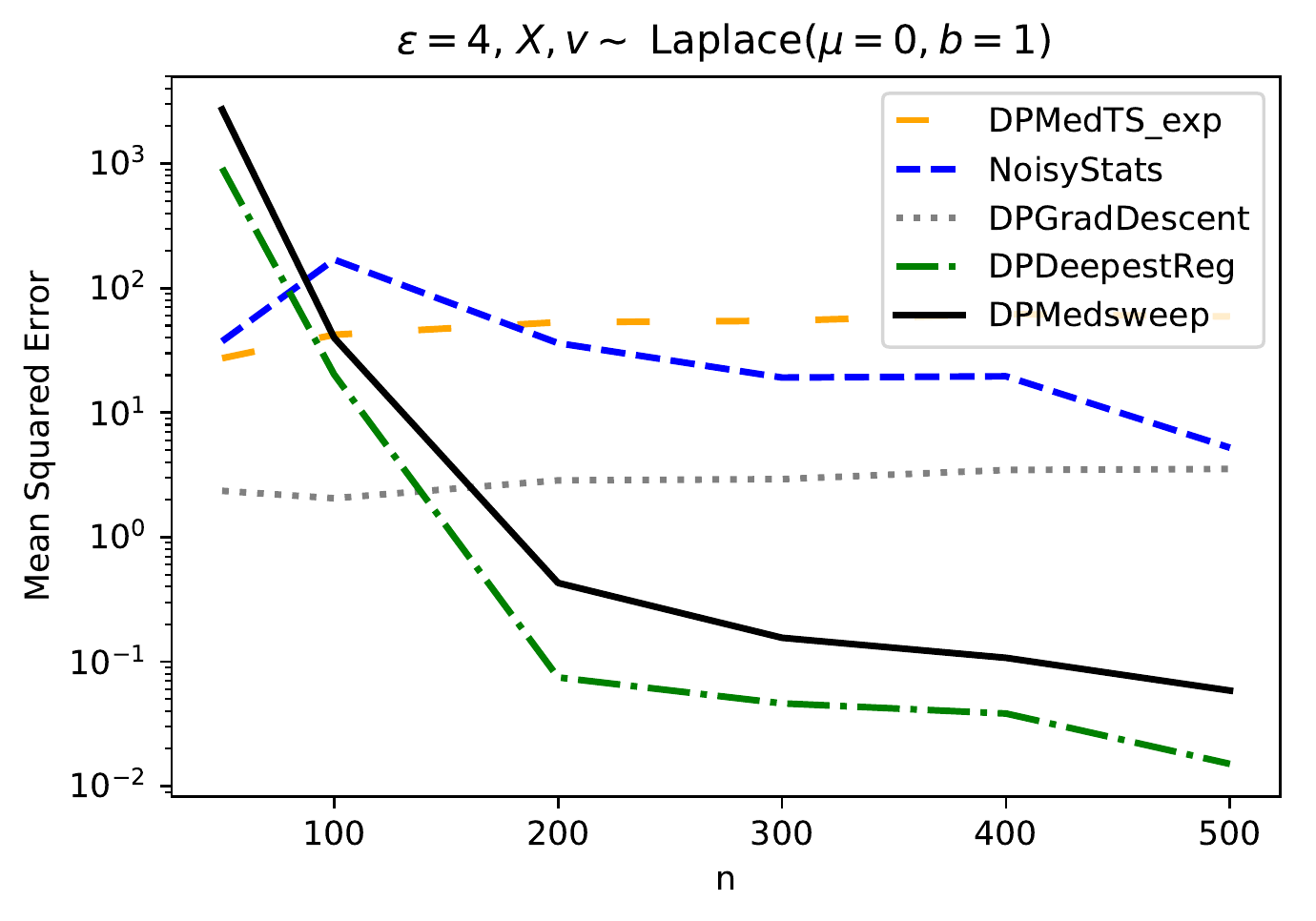} \includegraphics[scale=.58]{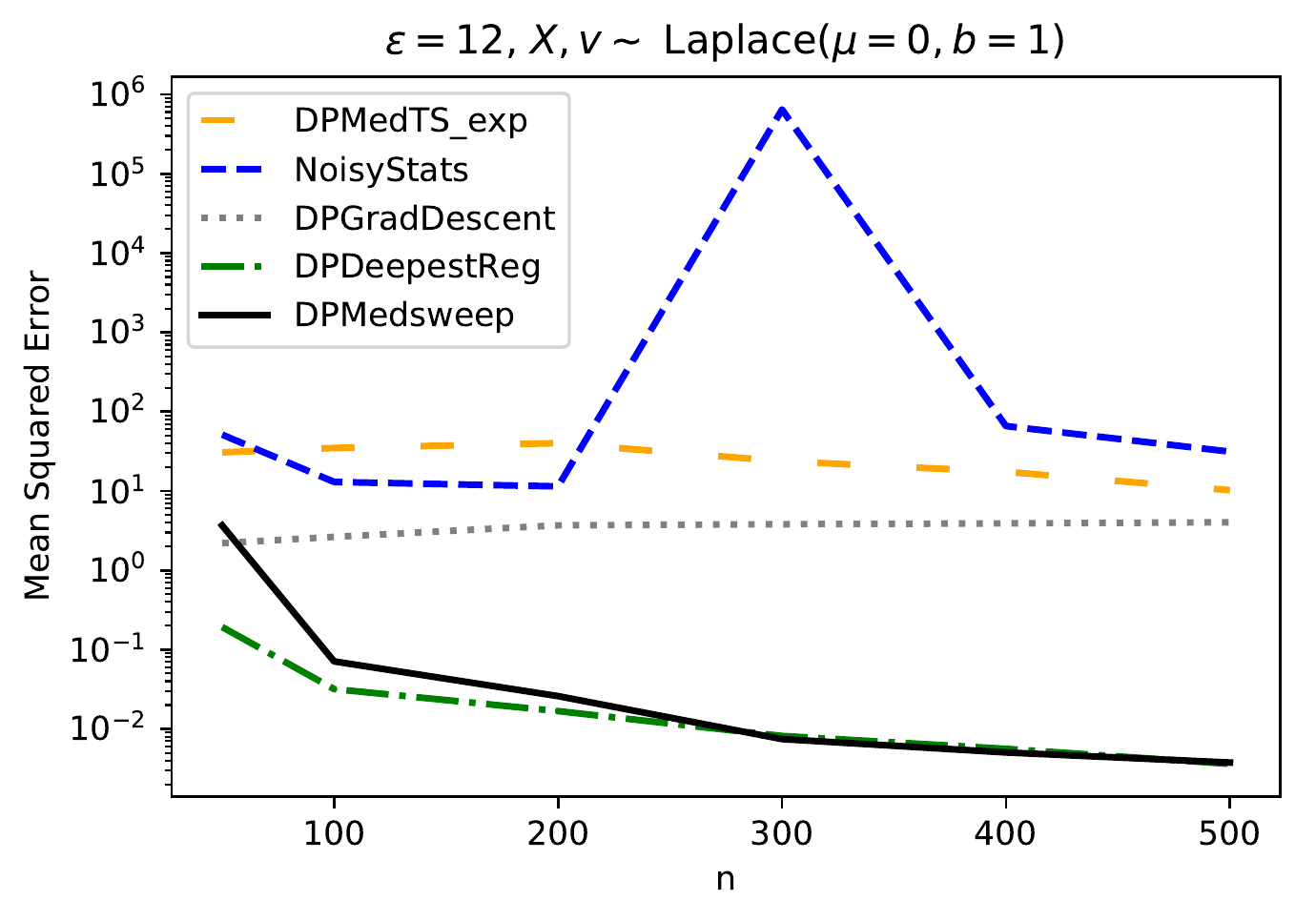} \\
    \includegraphics[scale=.58]{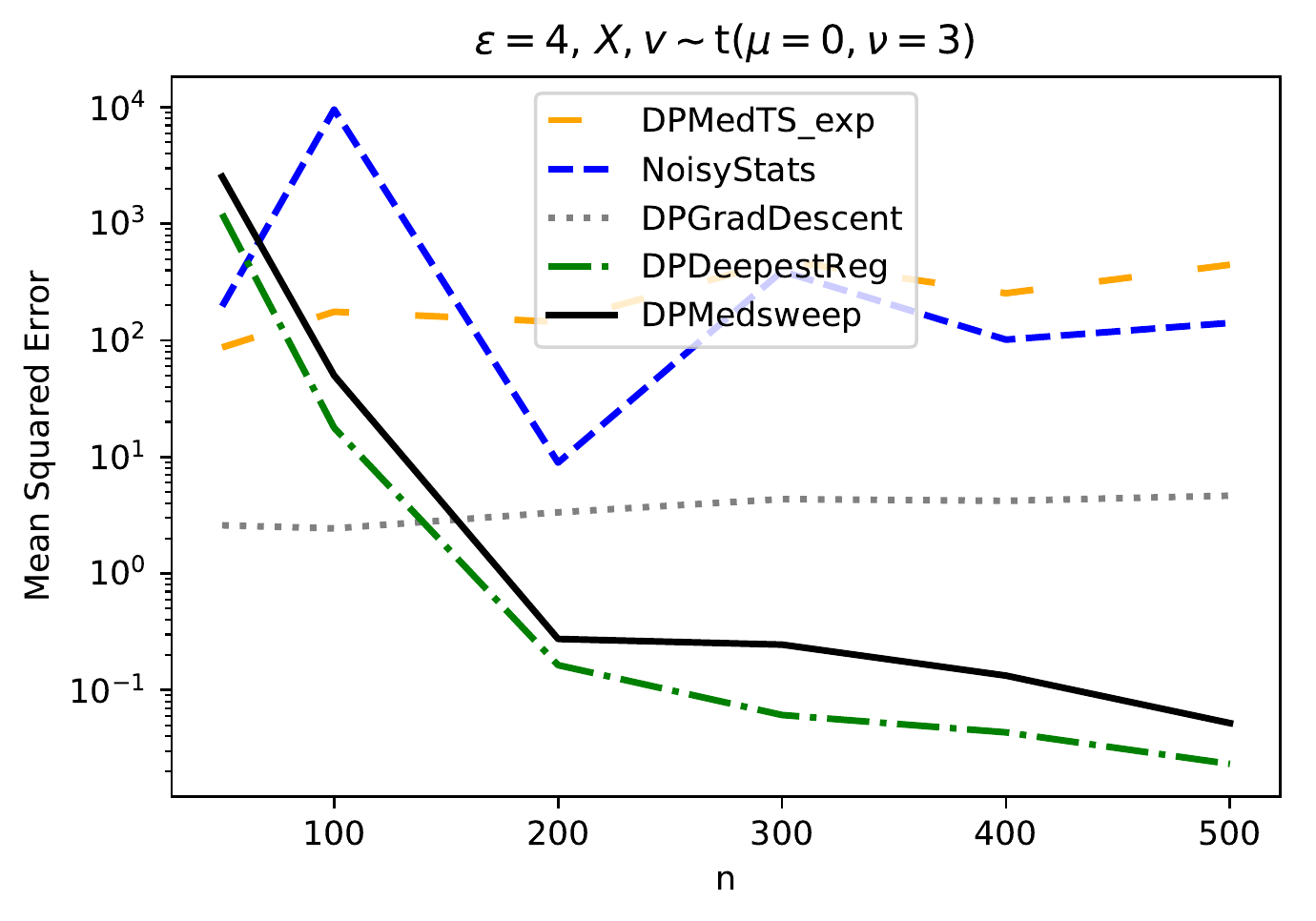} \includegraphics[scale=.58]{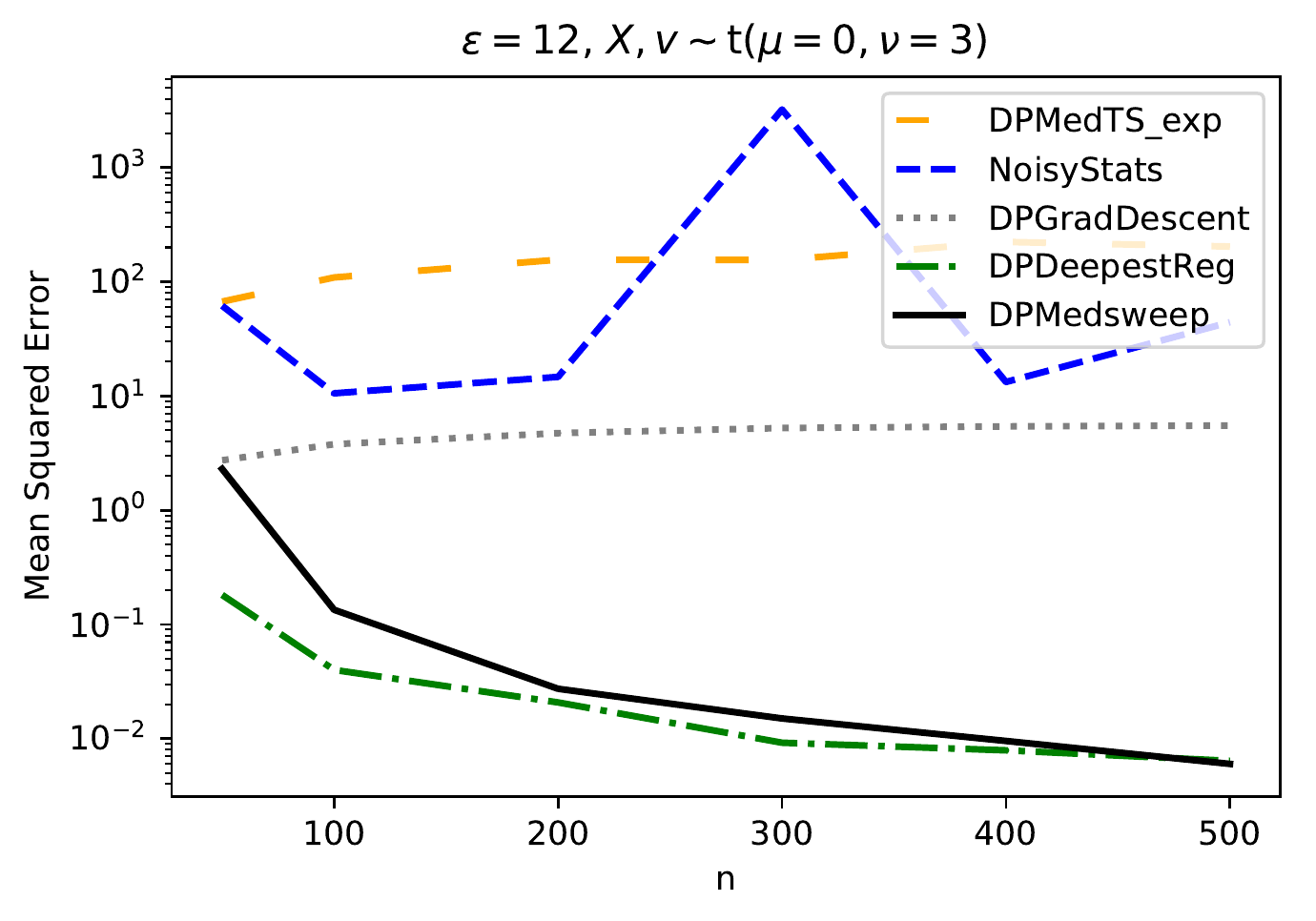}
    \caption{The mean squared error at the 25\% and 75\% percentile of $x_i$ of each DP estimator, for each DGP $x_i, v_i \sim N(\mu=0, \sigma^2=1)$ (top), $x_i, v_i \sim \textup{Laplace}(\mu=0, b=1)$ (middle), $x_i, v_i \sim t(\mu=0, \nu=3)$ (bottom), and for $\epsilon=4$ (left) and $\epsilon=12$ (right). For NoisyStats$(\cdot)$ and the DPMedTS\_exp$(\cdot)$ mechanisms, $\psi$ was set to 1. Both NoisyStats$(\cdot)$ and DPMedTS\_exp$(\cdot)$ satisfy $\epsilon-$DP; DPGradDesc$(\cdot),$ DPMedsweep$(\cdot),$ and DPDeepestReg$(\cdot)$ satisfy $(\epsilon, 10^{-6})-$DP.}
    \label{fig:sims_MSE_0}
\end{figure}

\begin{figure}[ht]
    \centering
    \includegraphics[scale=.58]{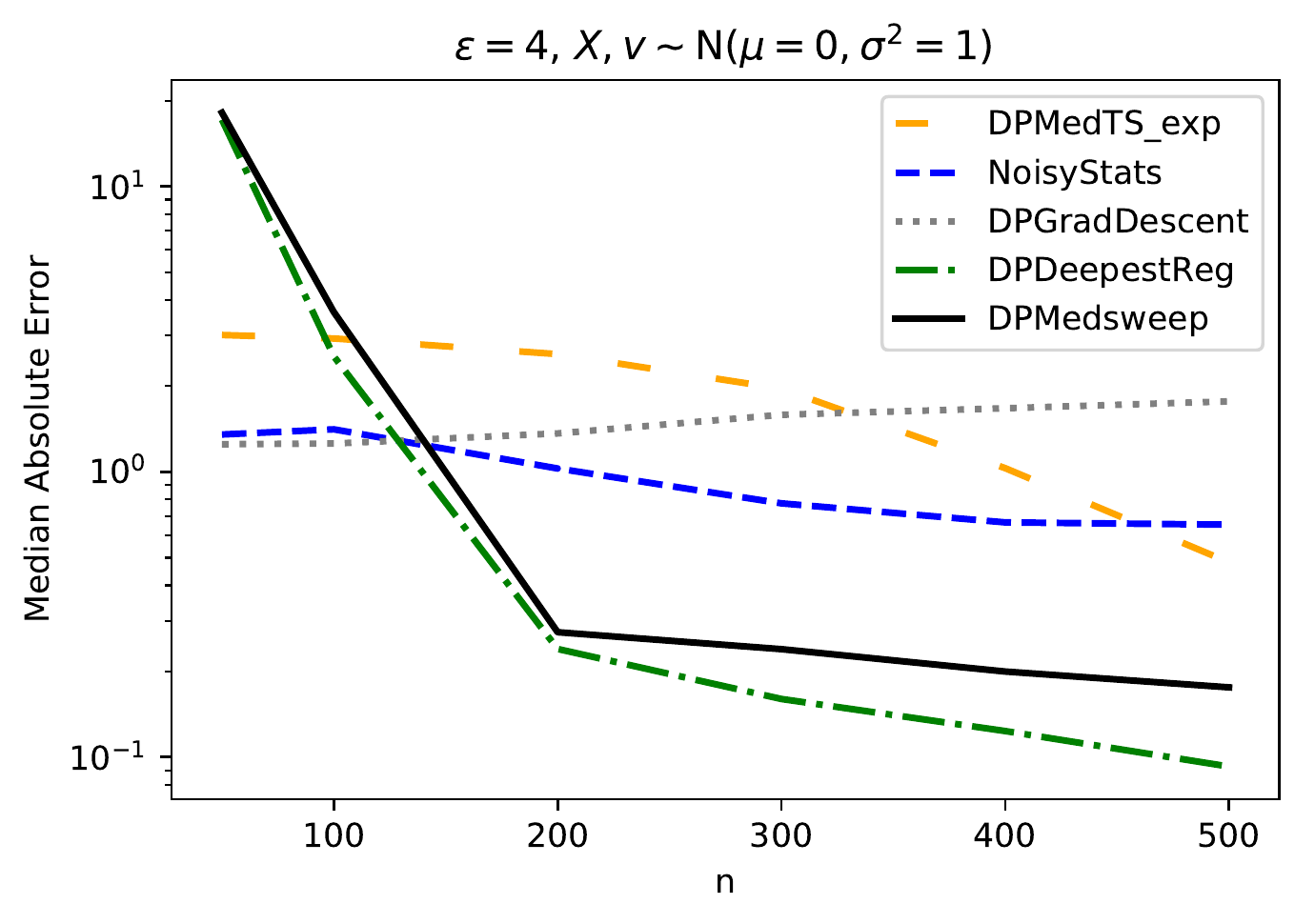} \includegraphics[scale=.58]{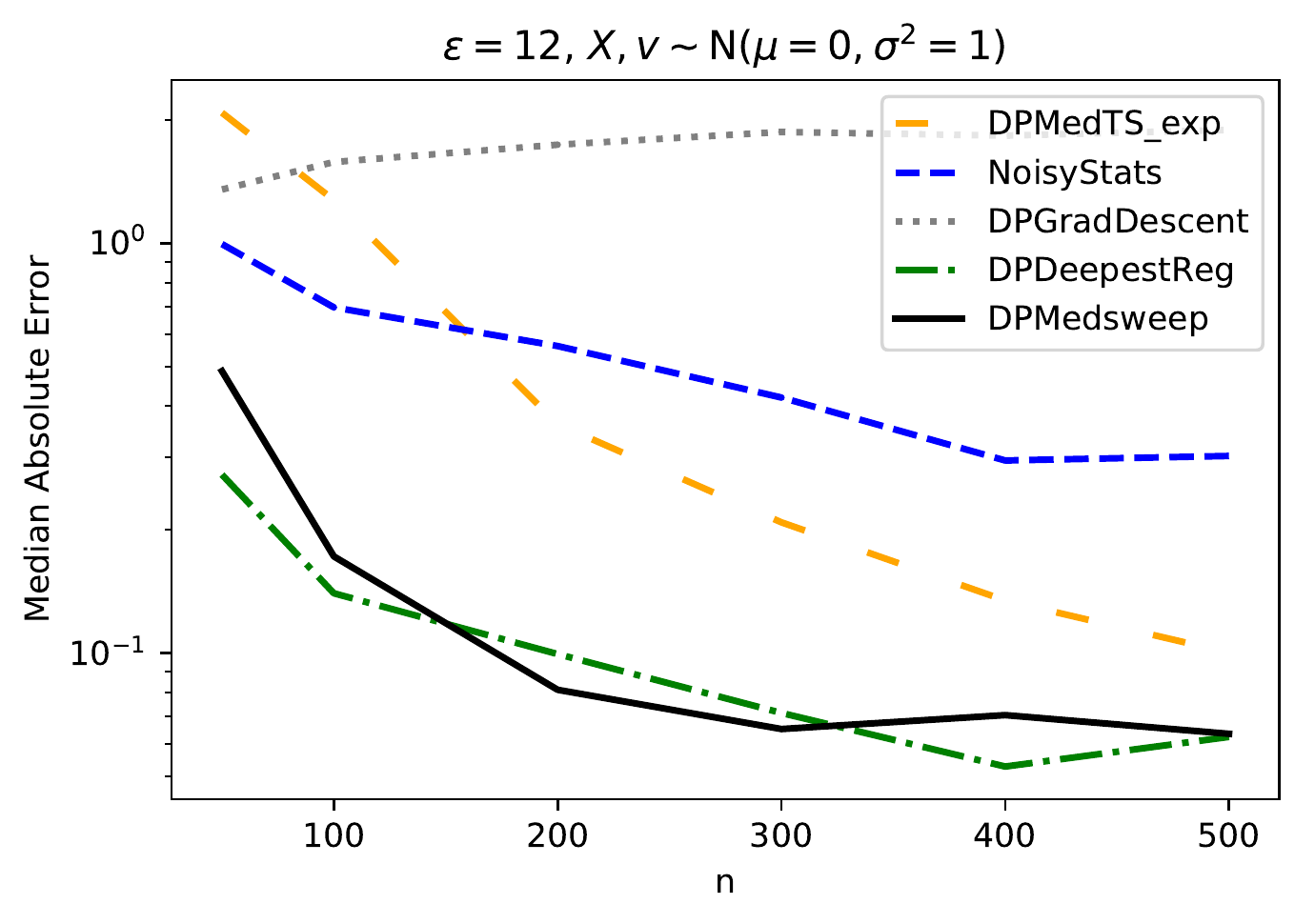} \\
    \includegraphics[scale=.58]{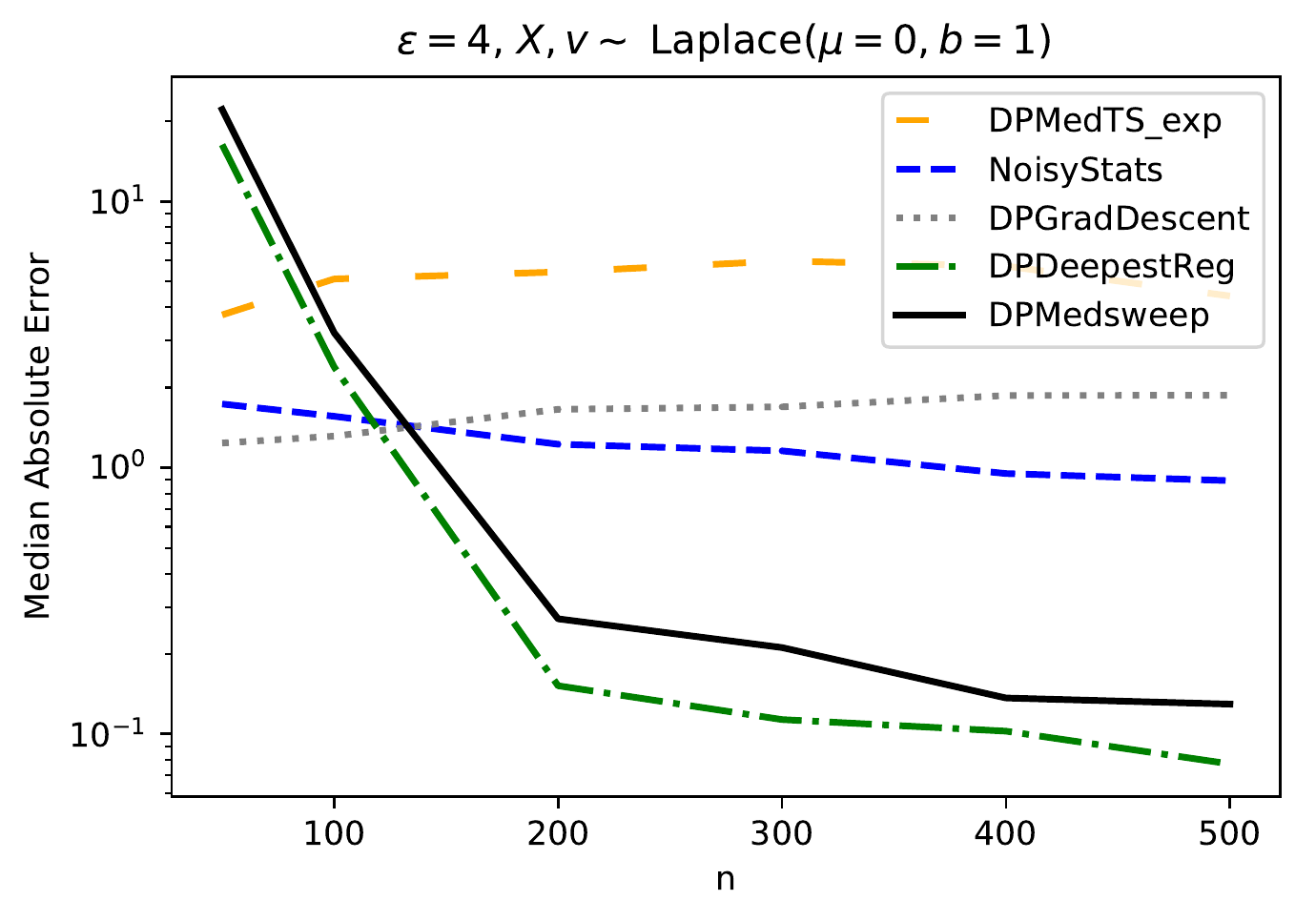} \includegraphics[scale=.58]{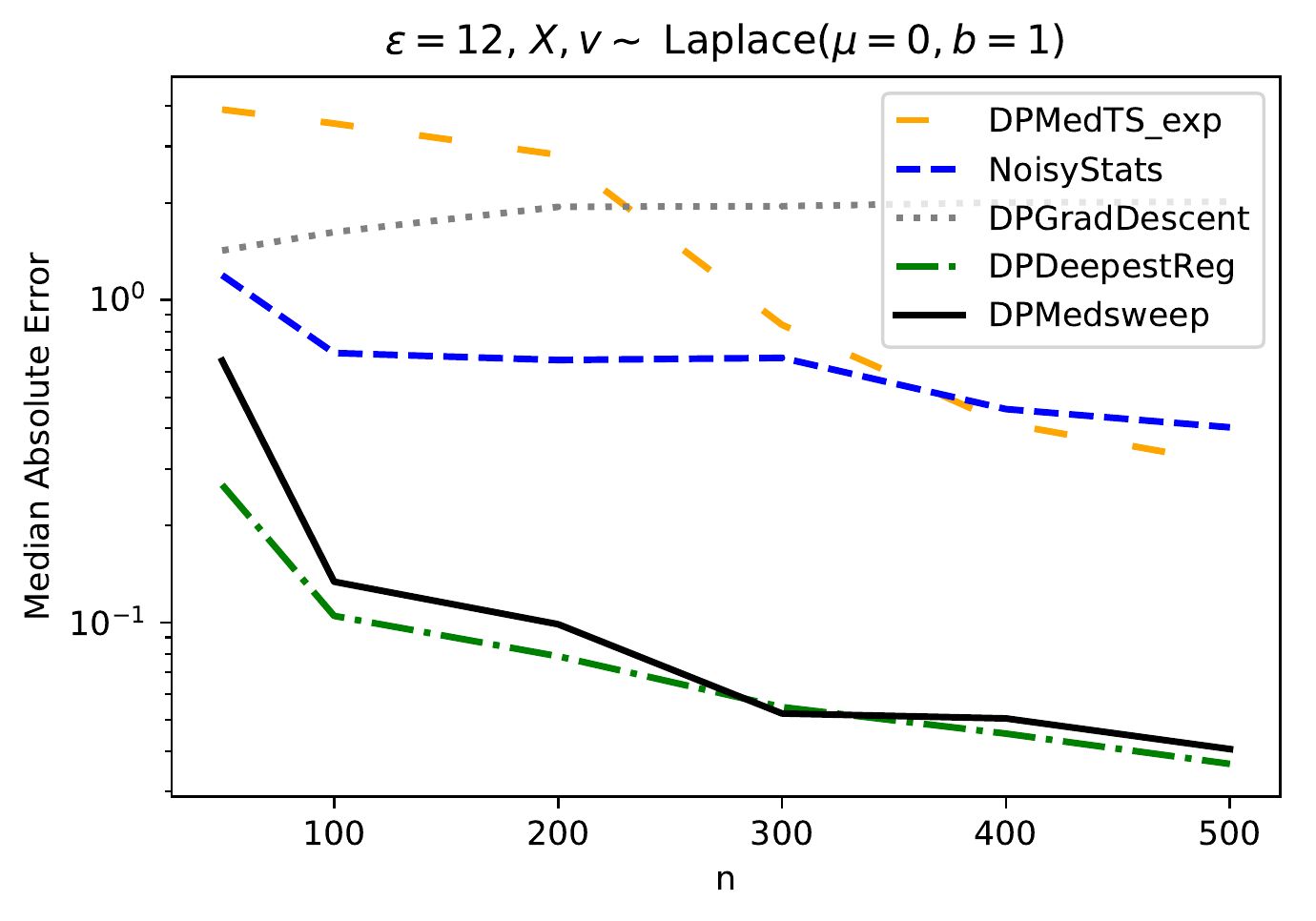} \\
    \includegraphics[scale=.58]{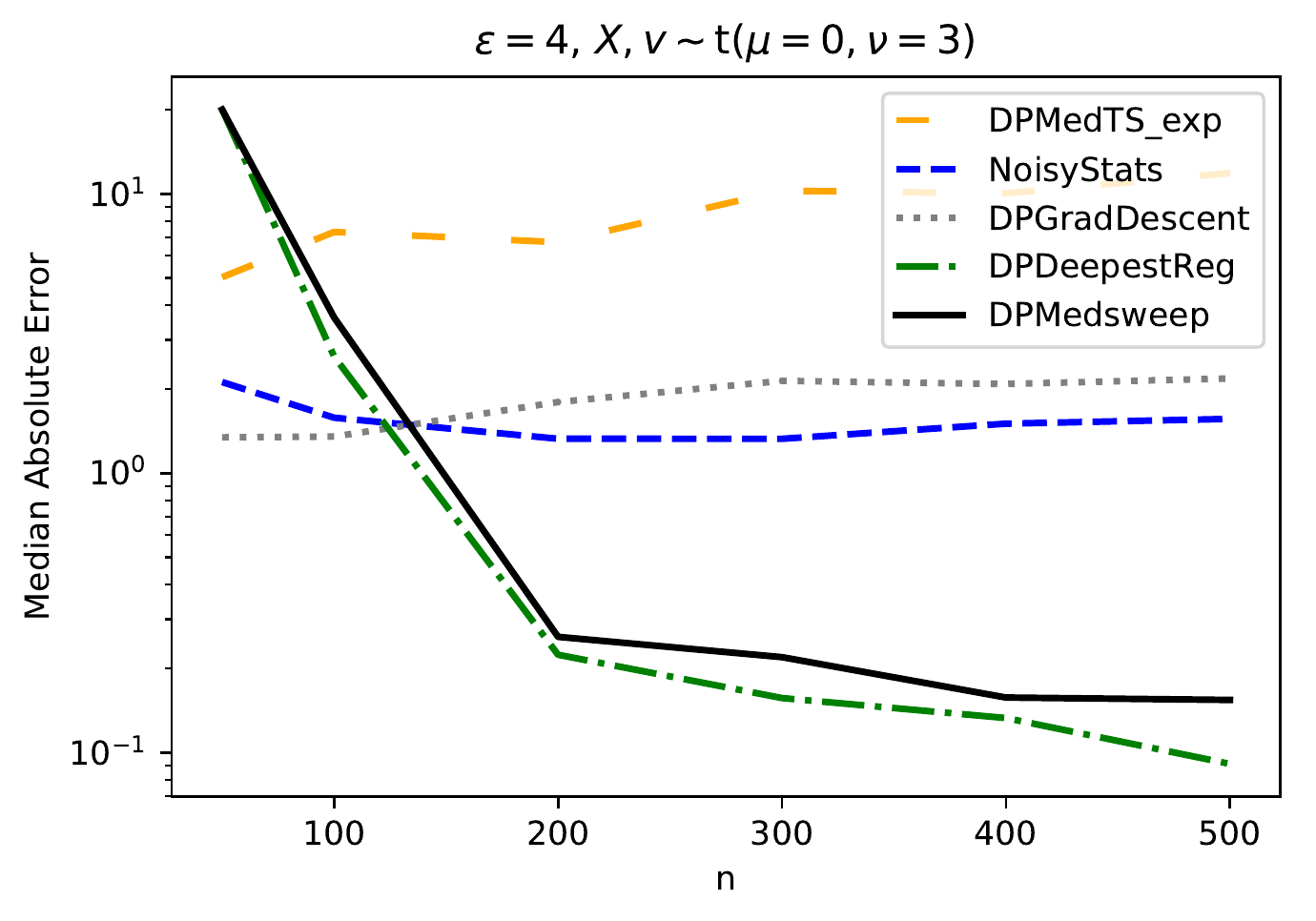} \includegraphics[scale=.58]{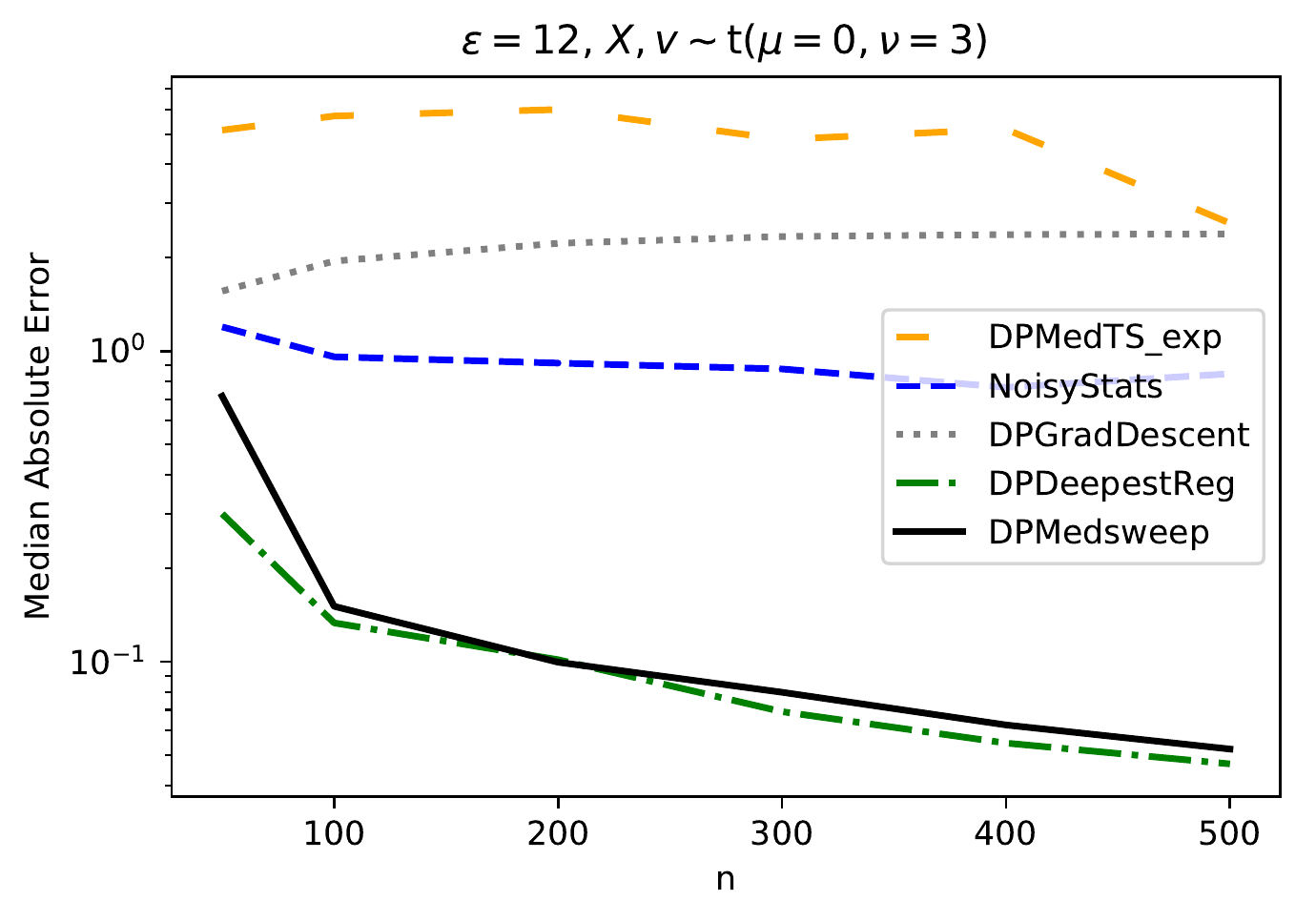}
    \caption{The median absolute error at the 25\% and 75\% percentile of $x_i$ of each DP estimator, for each DGP $x_i, v_i \sim N(\mu=0, \sigma^2=1)$ (top), $x_i, v_i \sim \textup{Laplace}(\mu=0, b=1)$ (middle), $x_i, v_i \sim t(\mu=0, \nu=3)$ (bottom), and for $\epsilon=4$ (left) and $\epsilon=12$ (right). For NoisyStats$(\cdot)$ and the DPMedTS\_exp$(\cdot)$ mechanisms, $\psi$ was set to 1. Both NoisyStats$(\cdot)$ and DPMedTS\_exp$(\cdot)$ satisfy $\epsilon-$DP; DPGradDesc$(\cdot),$ DPMedsweep$(\cdot),$ and DPDeepestReg$(\cdot)$ satisfy $(\epsilon, 10^{-6})-$DP.}
    \label{fig:sims_MAE_0}
\end{figure}

\clearpage

\bibliographystyle{apalike}
\bibliography{references}

\end{document}